\documentclass[10pt,journal,compsoc]{IEEEtran}

\ifCLASSOPTIONcompsoc
  \usepackage[nocompress]{cite}
\else
  \usepackage{cite}
\fi

\ifCLASSINFOpdf
\else
\fi

\usepackage{amsmath}
\usepackage{amsthm}
\usepackage{mathtools}
\usepackage{xcolor}
\usepackage{mypackage}
\usepackage{float}
\usepackage{booktabs}
\usepackage{tikz}
\tikzset{node/.style = {circle, draw=black,
                        minimum width=0.25cm, minimum height=0.25cm,
                        text centered, fill=white,
                        line width=0.25mm},		 
		 rect/.style = {rectangle, rounded corners, draw=black,
                        minimum width=1cm, minimum height=1cm,
                        text centered, fill=white,
                        line width=0.25mm},
         note/.style = {rectangle,
                        minimum width=1cm, minimum height=1cm,
                        text centered}}

\DeclarePairedDelimiterX{\infdivx}[2]{(}{)}{#1\;\delimsize\|\;#2}
\DeclarePairedDelimiterX{\condexp}[2]{[}{]}{#1\;\delimsize\vert\;#2}

\newtheorem{theorem}{Theorem}[section]
\newtheorem{lemma}[theorem]{Lemma}

\newtheorem{experimentmessage}{Message}

\theoremstyle{definition}
\newtheorem{definition}{Definition}[section]
\newtheorem{example}[definition]{Example}

\newcommand{\qedclub}{\hfill$\clubsuit$}

\begin{document}
\title{PAC-Bayes Bounds for Bandit Problems: A Survey and Experimental Comparison}

\author{Hamish Flynn,
David Reeb,
Melih Kandemir,
Jan Peters% <-this % stops a space
\IEEEcompsocitemizethanks{\IEEEcompsocthanksitem H. Flynn and D. Reeb are with the Bosch Center for Artificial Intelligence, Renningen, Germany.\protect\\
E-mail: hamish.flynn@de.bosch.com
\IEEEcompsocthanksitem M. Kandemir is with the Department of Mathematics and Computer Science (IMADA) at the University of Southern Denmark, Odense, Denmark.
\IEEEcompsocthanksitem J. Peters is with the Intelligent Autonomous Systems group at Technische Universit\"at Darmstadt, Germany.}}
\IEEEtitleabstractindextext{%
\begin{abstract}
PAC-Bayes has recently re-emerged as an effective theory with which one can derive principled learning algorithms with tight performance guarantees. However, applications of PAC-Bayes to bandit problems are relatively rare, which is a great misfortune. Many decision-making problems in healthcare, finance and natural sciences can be modelled as bandit problems. In many of these applications, principled algorithms with strong performance guarantees would be very much appreciated. This survey provides an overview of PAC-Bayes bounds for bandit problems and an experimental comparison of these bounds. On the one hand, we found that PAC-Bayes bounds are a useful tool for designing offline bandit algorithms with performance guarantees. In our experiments, a PAC-Bayesian offline contextual bandit algorithm was able to learn randomised neural network polices with competitive expected reward and non-vacuous performance guarantees. On the other hand, the PAC-Bayesian online bandit algorithms that we tested had loose cumulative regret bounds. We conclude by discussing some topics for future work on PAC-Bayesian bandit algorithms.
\end{abstract}

%\begin{IEEEkeywords}
%None.
%\end{IEEEkeywords}
}

\maketitle

\IEEEdisplaynontitleabstractindextext

% \fi
\IEEEpeerreviewmaketitle

\IEEEraisesectionheading{\section{Introduction}\label{sec:introduction}}

\IEEEPARstart{I}{s} it possible to know that a machine learning system will perform well before it is tested on new data? Within the field of statistical learning theory, there are several frameworks that can provide high probability bounds on the performance of a machine learning algorithm in a number of different learning problems. A relatively rare combination of framework and learning problem is the application of the PAC-Bayes framework to bandit problems.

The PAC-Bayes framework has recently grown in popularity for several possible reasons. First, it has emerged as one of the few ways to provide tight error bounds for deep neural networks \cite{dziugaite2017computing}, \cite{dziugaite2018data}, \cite{letarte2019binary}, \cite{rivasplata2019pac}, \cite{dziugaite2021data}, \cite{perez2021tighter}, \cite{perez2021learning}, \cite{perez2021progress}. Second, learning algorithms derived from PAC-Bayes bounds have performed competitively with traditional algorithms \cite{ambroladze2007tighter}, \cite{germain2009pac}, \cite{thiemann2017strongly}, \cite{reeb2018learning}. Third, PAC-Bayes bounds can motivate principled learning strategies, such as large margin classification \cite{boser1992margin}, \cite{herbrich2000margin}, \cite{langford2002margin}, \cite{mcallester2003margin} and preference for flat minima \cite{hochreiter1997flat}, \cite{dziugaite2017computing}, \cite{yang2019fast}, \cite{tsuzuku2020flat}. However, most PAC-Bayes bounds and algorithms are designed for supervised learning problems. Applications of PAC-Bayes to bandit problems are relatively under-explored. This survey provides an overview and an experimental comparison of PAC-Bayes bounds for bandit problems.

\textbf{PAC-Bayes} bounds \cite{shawe1997pac}, \cite{mcallester1998some} are Probably Approximately Correct (PAC) \cite{valiant1984pac} performance bounds for Bayesian learning algorithms. A PAC bound states that, with high probability (\textit{probably}), the error-rate of the hypothesis returned by a learning algorithm is upper bounded. If this upper bound on the error-rate is small, then the learning algorithm is \textit{approximately correct}. When PAC bounds are applied to Bayesian learning algorithms, the result is called a PAC-Bayes bound. In fact, PAC-Bayes bounds apply to any learning algorithm that returns a probability distribution over a hypothesis class.

\textbf{Bandits.} Bandit problems, first introduced by Thompson \cite{thompson1933bandits} and later formalised by Robbins \cite{robbins1952bandits}, are models of decision-making with uncertainty. There is a set of actions, and each action is associated with a reward distribution. A bandit algorithm must learn to choose the actions with the highest expected reward. The uncertainty comes from the fact that the reward distribution for each action is unknown and must be estimated based on previously observed actions and rewards. Bandit problems are frequently encountered in real-world problems, including clinical trials \cite{durand2018contextual}, \cite{bastani2020online}, dynamic pricing \cite{misra2019dynamic}, \cite{mueller2019dynamic} and recommendation systems \cite{mary2015recommender}, to name just a few.

\textbf{Motivation.}  At the time of writing, there is neither a detailed overview of PAC-Bayes bounds for bandit problems nor an experimental comparison of these bounds. It is therefore difficult to know which PAC-Bayes bandit bounds give the best guarantees or how tight the best bounds are. There are two main reasons why we believe that now is the right time to review PAC-Bayesian approaches to bandits. First, PAC-Bayes bounds have recently been used to design effective offline bandit algorithms with performance guarantees \cite{london2019bayesian}. Second, as we have mentioned, PAC-Bayes has been growing in popularity due to numerous successful applications to deep learning. In parallel, there has been growing interest in bandit algorithms that use deep neural network function approximation. We believe that it is worth investigating whether PAC-Bayes would be a useful tool for studying these deep bandit algorithms.

\begin{figure*}
\centering

\begin{tikzpicture}[rect/.style={rectangle, draw=black, fill=black!5, very thick, minimum height=7mm, minimum width=5mm}, scale=1.0]
% Specification of nodes (position, etc.)
\node (bdt)[rect] at (0.5,1.4){PAC-Bayes Bandit Bounds};

\node (rew)[rect] at (-4,0){Reward Bounds};
\node (reg)[rect] at (4.9,0){Cumulative Regret Bounds};

\node (rew_is)[rect] at (-8.3,-1.4){$r^{\mathrm{IS}}$};
\node (rew_cis)[rect] at (-2.9,-1.4){$r^{\mathrm{CIS}}$};
\node (rew_wis)[rect] at (1,-1.4){$r^{\mathrm{WIS}}$};
\node (reg_is)[rect] at (4.9,-1.4){$\Delta^{\mathrm{IS}}$};

\node (rew_is_ha)[rect] at (-10,-2.8){Hoeffding};
\node (rew_is_be)[rect] at (-8.05,-2.8){Bernstein};
\node (rew_is_kl)[rect] at (-6.5,-2.8){$kl$};
\node (rew_cis_ha)[rect] at (-4.5,-2.8){Hoeffding};
\node (rew_cis_be)[rect] at (-2.55,-2.8){Bernstein};
\node (rew_cis_kl)[rect] at (-1,-2.8){$kl$};
\node (rew_wis_es)[rect] at (1,-2.8){Efron-Stein};
\node (reg_is_ha)[rect] at (3.25,-2.8){Hoeffding};
\node (reg_is_be)[rect] at (5.2,-2.8){Bernstein};
\node (reg_is_kl)[rect] at (6.75,-2.8){$kl$};
  
% Specification of lines between nodes specified above
\draw[->, line width=0.5mm](bdt.south) to (rew.north);
\draw[->, line width=0.5mm](bdt.south) to (reg.north);

\draw[->, line width=0.5mm](rew.south) to (rew_is.north);
\draw[->, line width=0.5mm](rew.south) to (rew_cis.north);
\draw[->, line width=0.5mm](rew.south) to (rew_wis.north);
\draw[->, line width=0.5mm](reg.south) to (reg_is.north);

\draw[->, line width=0.5mm](rew_is.south) to (rew_is_ha.north);
\draw[->, line width=0.5mm](rew_is.south) to (rew_is_be.north);
\draw[->, line width=0.5mm](rew_is.south) to (rew_is_kl.north);
\draw[->, line width=0.5mm](rew_cis.south) to (rew_cis_ha.north);
\draw[->, line width=0.5mm](rew_cis.south) to (rew_cis_be.north);
\draw[->, line width=0.5mm](rew_cis.south) to (rew_cis_kl.north);
\draw[->, line width=0.5mm](rew_wis.south) to (rew_wis_es.north);
\draw[->, line width=0.5mm](reg_is.south) to (reg_is_ha.north);
\draw[->, line width=0.5mm](reg_is.south) to (reg_is_be.north);
\draw[->, line width=0.5mm](reg_is.south) to (reg_is_kl.north);

\end{tikzpicture}

\caption{A taxonomy of existing PAC-Bayes bandit bounds. The bounds are first separated into lower bounds on reward and upper bounds on cumulative regret. At the next level, the bounds are categorised by the empirical reward/regret estimate that they use. The reward estimates $r^{\mathrm{IS}}$, $r^{\mathrm{CIS}}$, and $r^{\mathrm{WIS}}$, and the regret estimate $\Delta^{\mathrm{IS}}$, are defined in Section \ref{sec:rew_bounds}, Appendix \ref{sec:app_wis_estimate} and Section \ref{sec:reg_bounds}. Finally, the bounds are divided according to the concentration inequality that they use in their proofs. $kl$ is the Binary KL divergence, defined in Section \ref{sec:rew_bounds}.}
\label{fig:taxonomy}
\end{figure*}
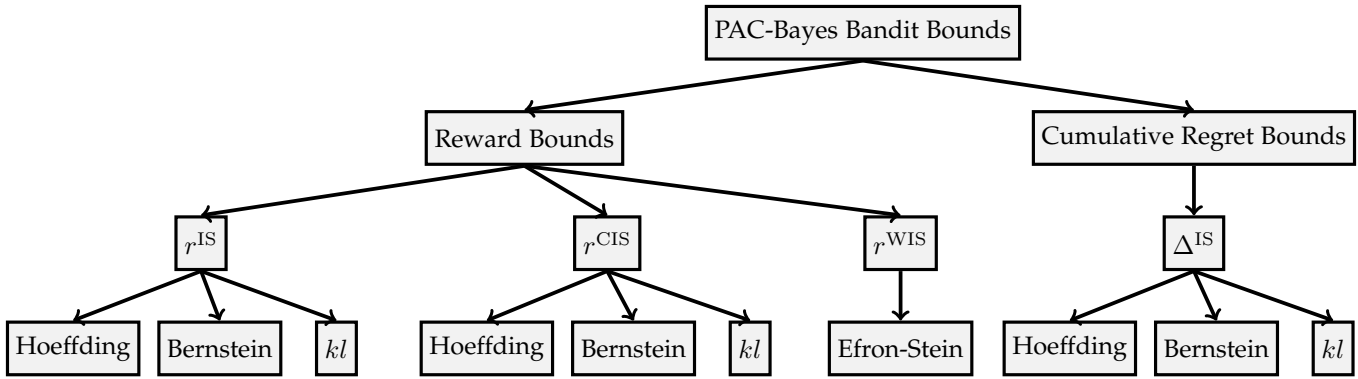

\textbf{Scope.} The scope of this survey is determined by the selection of PAC-Bayesian approaches to bandits that can be found in the literature. Consequently, we focus on policy search algorithms that directly learn a policy from data using reward estimates based on importance sampling. We found that there were no model-based PAC-Bayesian bandit algorithms, which first model the reward function and then use this model to learn a policy, so we do not cover these approaches. However, we discuss the compatibility of PAC-Bayes with other approaches to bandits in Sec. \ref{sec:beyond_policy_search}.

We cover offline and online variants of both multi-armed and contextual bandit problems. We consider two types of PAC-Bayes bounds: one for offline bandits and one for online bandits. For offline bandits, we consider lower bounds on the expected reward of a policy learned from historical data. For online bandits, we consider upper bounds on the cumulative regret suffered by playing a sequence of policies. The bounds considered in this survey are categorised further in Fig. \ref{fig:taxonomy}.

We only consider stationary, stochastic bandit problems, where the rewards are sampled from fixed distributions. We do not cover extensions such as restless bandits \cite{whittle1988restless} or adversarial bandits \cite{auer2002exp3}. We also do not cover bandit problems with additional structural assumptions, such as linear bandits \cite{auer2002using}.

\textbf{Findings.} We compared the values of the bounds, as well as the performance of bandit algorithms motivated by the bounds. On the one hand, we found that some of the PAC-Bayes lower bounds on the expected reward are surprisingly tight, particularly when data-dependent priors are used. Moreover, we found that directly optimising PAC-Bayes reward bounds can yield effective offline bandit algorithms. PAC-Bayes appears to be a useful tool for designing offline bandit algorithms with performance guarantees. On the other hand, we found that the few existing PAC-Bayes cumulative regret bounds are all loose, and that the algorithms motivated by these bounds are noticeably worse than state-of-the-art methods. The reason for this is that both the bounds and algorithms rely on loose upper bounds on the variance of importance sampling-based reward estimates.

\textbf{Related work.} PAC-Bayes bounds have been the subject of several tutorials \cite{mcallester2013pac}, \cite{vanerven2014mini}, \cite{laviolette2017tutorial}, \cite{alquier2021user}, surveys \cite{guedj2019primer} and monographs \cite{catoni2007pac}. McAllester \cite{mcallester2013pac} describes 3 different types of PAC-Bayes bounds and presents a new application of PAC-Bayes bounds to dropout. Van Erven \cite{vanerven2014mini} describes the relationship between PAC-Bayes bounds and some classical concentration inequalities. Laviolette \cite{laviolette2017tutorial} describes the history of PAC-Bayes bounds as well as some recent developments. Alquier \cite{alquier2021user} gives an overview of PAC-Bayes bounds for supervised learning and an introduction to localised bounds, fast-rate bounds and bounds for non i.i.d. data and unbounded losses. Guedj \cite{guedj2019primer} surveys the PAC-Bayes framework, its links to Bayesian methods, and some theoretical and algorithmic developments. Catoni \cite{catoni2007pac} provides a rich analysis of supervised classification using PAC-Bayes bounds. There have been a few experimental comparisons of some PAC-Bayes bounds \cite{foong2021tight}, \cite{perez2021tighter} in supervised learning problems. There are several books \cite{bubeck2012bandits}, \cite{slivkins2019bandits}, \cite{lattimore2020bandits} about bandit algorithms and their performance guarantees. However, none of these resources on bandits cover PAC-Bayes.

\textbf{Paper Overview.} First, we formally describe the online and offline variants of multi-armed and contextual bandit problems in Sec. \ref{sec:prob_state}. In Sec. \ref{sec:pacb_ex}, we describe the PAC-Bayesian approach to the bandit problems introduced in Sec. \ref{sec:prob_state}. We then provide a structured overview of PAC-Bayes bounds for bandit problems and some techniques for achieving the tightest bound values. Sec. \ref{sec:rew_bounds} reviews PAC-Bayes lower bounds on the expected reward, Sec. \ref{sec:reg_bounds} reviews PAC-Bayes upper bounds on the cumulative regret, and Sec. \ref{sec:optimising_bounds} reviews techniques for optimising PAC-Bayes bandit bounds with respect to the prior and other parameters. In Sec. \ref{sec:comparison}, we compare the PAC-Bayes bandit bounds in several experiments. Finally, in Sec. \ref{sec:conclusion}, we discuss our findings and comment on some open problems.

\textbf{Contributions.} Our first contribution is a comprehensive overview of existing PAC-Bayes bounds for bandit problems. Our second contribution is an experimental comparison of PAC-Bayes bounds and algorithms for bandit problems. We also provide a slightly tighter version of the Efron-Stein PAC-Bayes bound by Kuzborskij and Szepesv{\'a}ri \cite{kuzborskij2019efron}, which holds under slightly weaker conditions.

\section{Problem Formulation}
\label{sec:prob_state}

The goal of all the bandit problems we consider is to select the best policy $\pi$ from a set of policies $\Pi$, which we call the policy class. In this paper, we are interested in bandit algorithms that return a probability distribution over the policy class rather than a single policy. $\mathcal{P}(\Pi)$ denotes the set of all probability distributions over the policy class.

The choice of policy is informed by data. We use $\mathcal{Z}$ to denote the observation space. A bandit algorithm observes or collects a data set of observations $D_n = \{z_i\}_{i=1}^{n}$. Each $z_i$ is drawn from a distribution $P_i$ over $\mathcal{Z}$. In bandit problems, we may have non-identically distributed data, where $P_i \neq P_j$ for $i \neq j$. We may also have dependent data, where $z_i$ is drawn from $P_i = P(\cdot|z_1, \dots, z_{i-1})$. Usually, we will make $\mathcal{Z}$ more explicit. In the simplest case, we observe pairs of actions and rewards, so $\mathcal{Z} = \mathcal{A} \times \mathcal{R}$ where $\mathcal{A}$ is a set of actions and $\mathcal{R}$ is a set values that the rewards can take.

\subsection{Policy Search for Multi-Armed Bandits}
\label{sec:multi_bandit_prob}

A multi-armed bandit (MAB) problem is a tuple $\langle \mathcal{A}, \mathcal{R}, P_{R}\rangle$. $\mathcal{A}$ is a set of actions (or arms), $\mathcal{R}$ is a set of values that the rewards can take and $P_{R}(\cdot|a)$ is a distribution over rewards conditioned on the action $a$. $\mathcal{A}$ and $\mathcal{R}$ are known, but $P_{R}$ is unknown. Throughout this paper, we assume that the rewards are bounded between 0 and 1, so $\mathcal{R} \subseteq [0, 1]$.

A bandit algorithm selects actions through a policy $\pi$. In a MAB problem, a policy is a (possibly degenerate) probability distribution over the set of actions $\mathcal{A}$. $\pi(a)$ denotes the probability of selecting action $a$ under the policy $\pi$. In the offline MAB problem, an algorithm is given a data set $D_n = \{(a_i, r_i)\}_{i=1}^{n}$. We let $D_{i-1} = \{(a_j, r_j)\}_{j=1}^{i-1}$ denote the first $i-1$ elements of $D_n$. Each action $a_i$ is sampled from a behaviour policy $b(\cdot|D_{i-1})$ and each reward $r_i$ is sampled from the reward distribution, given $a_i$. Some of the PAC-Bayes bounds we will encounter hold only when the data set $D_n$ consists of i.i.d. samples. For these bounds to hold, we require that the entire data set is drawn using a fixed behaviour policy $b(\cdot)$. We always assume that the behaviour policies are known. The expected reward for a policy $\pi$ is defined as:
\begin{equation}
R(\pi) = \mathop{\mathbb{E}}_{a \sim \pi(\cdot), r \sim P_R(\cdot|a)}\left[r\right].\label{eqn:exp_rew}
\end{equation}

For a probability distribution $\rho \in \mathcal{P}(\Pi)$ over the policy class, the expected reward is $R(\rho) = \mathbb{E}_{\pi \sim \rho}\left[R(\pi)\right]$. Given a policy class $\Pi$ and a data set $D_n$, the goal of policy search in the offline MAB problem is to return a distribution $\rho^* \in \mathcal{P}(\Pi)$ that maximises the expected reward:
\begin{equation*}
\rho^{*} \in \argmax_{\rho \in \mathcal{P}(\Pi)}\left\{R(\rho)\right\}.
\end{equation*}

In the online MAB problem, an algorithm must learn and act simultaneously. Policy search in the online MAB problem proceeds in rounds. At round $i$, the algorithm selects a distribution $\rho_i \in \mathcal{P}(\Pi)$ to be played. A policy $\pi_i$ is drawn from $\rho_i$ and then an action $a_i$ is drawn from the policy $\pi_i$. The algorithm observes a reward $r_i$ drawn from the reward distribution $P_R(\cdot|a_i)$. To guide the selection of $\rho_i$ at each round, the algorithm can use the action-reward pairs gathered from previous rounds. In other words, the choice of $\rho_i$ can depend on the data $D_{i-1}$.

The goal of policy search in the online MAB problem is to select a sequence of distributions (over the policy class) $\rho_1, \dots, \rho_n$ that minimises the cumulative regret. For a sequence of policies $\pi_1, \dots, \pi_n$, the regret for round $i$ and the cumulative regret are defined as:
\begin{equation*}
\Delta(\pi_{i}) = R(\pi^*) - R(\pi_i), \qquad \Delta(\pi_{1:n}) = \sum_{i=1}^{n}\Delta(\pi_i),
\end{equation*}

where $\pi^* \in \argmax_{\pi \in \Pi}\left\{R(\pi)\right\}$ is an optimal policy. The per-round regret and cumulative regret for a sequence of distributions $\rho_1, \dots, \rho_n$ are defined as:
\begin{equation}
\Delta(\rho_{i}) = R(\pi^*) - R(\rho_i), \qquad \Delta(\rho_{1:n}) = \sum_{i=1}^{n}\Delta(\rho_i).
\label{eqn:mab_regret}
\end{equation}

The goal of minimising cumulative regret brings about a dilemma known as the exploration-exploitation trade-off. To achieve low cumulative regret, an algorithm must try out lots of policies to identify which ones have the highest expected reward. However, while it identifies which policies are the best, it must also limit the number of times it selects a sub-optimal policy.

\begin{example}[Clinical trial]
There are two flu treatments and we are given the results of a clinical trial where 100 patients have been randomly given either treatment A or treatment B. We want to decide which treatment is better. This can be modelled as an offline multi-armed bandit problem, where the actions are the treatment types and the rewards are the outcomes of the treatments. A PAC-Bayes reward bound could give a lower bound on the success rate of each treatment.

If we wanted to assign treatments to patients sequentially, with the goal of handing out the better treatment as often as possible, this could be modelled as an online bandit problem. A PAC-Bayes cumulative regret bound could tell us (before handing out any treatments) an upper bound on the gap between the optimal expected number of successful treatments and the expected number of successful treatments of our allocation strategy.
\end{example}

\subsection{Policy Search for Contextual Bandits}
\label{sec:contextual_bandit_prob}

A contextual bandit (CB) problem is a tuple $\langle \mathcal{S}, \mathcal{A}, \mathcal{R}, P_{S}, P_{R}\rangle$. $\mathcal{S}$ is a set of states (or contexts), $\mathcal{A}$ is a set of actions, $\mathcal{R}$ is a set of values that the rewards can take, $P_S(\cdot)$ is a distribution over the set of states and $P_{R}(\cdot|s, a)$ is a distribution over rewards conditioned on the state $s$ and the action $a$. $\mathcal{S}$, $\mathcal{A}$ and $\mathcal{R}$ are known, but $P_{S}$ and $P_{R}$ are unknown. As in the MAB problem, we assume that $\mathcal{R} \subseteq [0, 1]$ throughout this paper.

In a CB problem, a policy is a function that maps states to probability distributions over the set of actions $\mathcal{A}$. $\pi(a|s)$ denotes the probability of selecting action $a$, given the state $s$, under the policy $\pi$. The expected reward for a policy $\pi$ is defined as:
\begin{equation*}
R(\pi) = \mathop{\mathbb{E}}_{s \sim P_S(\cdot), a \sim \pi(\cdot|s), r \sim P_R(\cdot|s, a)}\left[r\right].
\end{equation*}

As before, the expected reward for a distribution $\rho \in \mathcal{P}$ over the policy class is $R(\rho) = \mathbb{E}_{\pi \sim \rho}\left[R(\pi)\right]$. The distinction between offline and online CB problems is very similar to the MAB case. In the offline CB problem, a data set $D_n = \{(s_i, a_i, r_i)\}_{i=1}^{n}$ of state-action-reward triples is available. The states $s_i$ are sampled from the state distribution $P_S$, the actions $a_i$ are sampled from behaviour policies $b(\cdot|s_i, D_{i-1})$ and the rewards $r_i$ are sampled from the reward distribution $P_{R}(\cdot|s_i, a_i)$. Whenever we require an i.i.d. data set, we will assume each action $a_i$ is drawn from the same, fixed behaviour policy $b(\cdot|s_i)$. Given a policy class $\Pi$ and a data set $D_n$, the goal of policy search in the offline CB problem is to return a distribution $\rho^* \in \mathcal{P}(\Pi)$ that maximises the expected reward.

In round $i$ of the online CB problem, an algorithm selects a distribution $\rho_i$. A state $s_i$ is drawn from $P_{S}$, a policy $\pi_i$ is drawn from $\rho_i$, an action $a_i$ is drawn from $\pi_i(\cdot|s_i)$, and then a reward $r_i$ is drawn from $P_R(\cdot|s_i, a_i)$. The choice of $\rho_i$ can depend on the data $D_{i-1}$. The goal is to select a sequence of distributions $\rho_1, \dots, \rho_n$ that minimises the cumulative regret. Per-round regret and cumulative regret are defined in the same way as in (\ref{eqn:mab_regret}).

\section{PAC-Bayes Bounds for Bandits}
\label{sec:pacb_ex}

A PAC-Bayesian approach to the policy search problems described in Sec. \ref{sec:prob_state} proceeds as follows. First, we fix a reference distribution or prior $\mu \in \mathcal{P}(\Pi)$ over the policy class $\Pi$. Then we observe data $D_n$, either a batch of historical data or the data collected in previous rounds, which helps us to learn another distribution $\rho \in \mathcal{P}(\Pi)$, which we will call a posterior distribution.

In the context of these policy search problems, a PAC-Bayes bound is an upper bound on either the difference between $R(\rho)$ and an empirical estimate of the reward of $\rho$ or the difference between $\Delta(\rho)$ and an empirical estimate of the regret of $\rho$, which holds uniformly over all posteriors $\rho$. One of the empirical reward estimates we consider is the importance sampling (IS) estimate
\begin{equation}
r^{\mathrm{IS}}(\pi, D_n) = \frac{1}{n}\sum_{i=1}^{n}\frac{\pi(a_i)}{b(a_i|D_{i-1})}r_i.\label{eqn:is_est}
\end{equation}

The IS estimate is an average of the observed rewards weighted by the importance weights $\pi(a_i)/b(a_i|D_{i-1})$. When upper bounding the difference between $R(\rho)$ and $r^{\mathrm{IS}}(\rho, D_n) = \mathop{\mathbb{E}}_{\pi \sim \rho}\left[r^{\mathrm{IS}}(\pi, D_n)\right]$, we face challenges that are not present in typical PAC-Bayesian learning settings. For example, the data $D_n$ are often not independent or identically distributed. This challenge can be dealt with using martingale techniques.

\subsection{PAC-Bayes and Martingales}

Martingales are a fundamental tool for modelling sequences of (possibly dependent) random variables. Martingales are sequences of random variables, for which at any point in the sequence, the conditional expectation of the present value is equal to the previous value. We will use the following basic definition of a martingale.

\begin{definition}[Martingale]
A sequence of random variables $(M_n|n \in \mathbb{N})$ is a martingale with respect to another sequence of random variables $(X_n|n \in \mathbb{N})$ if, for every $n \in \mathbb{N}$, $M_n$ is fully determined by $X_1, \dots, X_n$ (i.e. $M_n$ conditioned on $X_1, \dots, X_n$ is non-random) and
\begin{equation*}
\mathbb{E}[|M_n|] < \infty, \quad \mathbb{E}[M_n|X_1, \dots, X_{n-1}] = M_{n-1}.
\end{equation*}
\label{def:mart}
\end{definition}
\vspace{-0.4cm}

We call the property involving the conditional expectation the martingale property. If, instead of the martingale property, a sequence $(Z_n|n \in \mathbb{N})$ satisfies $\mathbb{E}[Z_n|X_1, \dots, X_{n-1}] = 0$ for all $n \in \mathbb{N}$, then we call $(Z_n|n \in \mathbb{N})$ a martingale difference sequence. Martingales often arise naturally in bandit problems, particularly when using importance sampling-based reward estimates. For example, for any $\pi \in \Pi$, the sequence of random variables
\begin{equation}
M_n^{\mathrm{IS}}(\pi) := n\left(r^{\mathrm{IS}}(\pi, D_n) - R(\pi)\right),\label{eqn:is_mart_ex}
\end{equation}
is a martingale with respect to $\{(a_i, r_i)\}_{i=1}^{n}$ (see Lem. \ref{lem:xis_mart}). Several authors \cite{seldin2012mart, balsubramani2015pac, wang2015pac, haddouche2022supermartingales, chugg2023unified} have proposed general purpose PAC-Bayes bounds for martingales, which provide upper (or lower) bounds on martingale mixtures $\mathbb{E}_{\pi \sim \rho}[M_n(\pi)]$ that hold uniformly over all posteriors $\rho \in \mathcal{P}(\Pi)$, where $(M_n(\pi)| n \in \mathbb{N})$ is a martingale for every $\pi \in \Pi$. These general purpose bounds can be applied to the martingale in (\ref{eqn:is_mart_ex}) to obtain PAC-Bayes bounds on the difference between $r^{\mathrm{IS}}(\rho, D_n)$ and $R(\rho)$. For the interested reader, we list these general purpose PAC-Bayes bounds for martingales in App. \ref{app:general_martingale_bounds}.

We will briefly mention that martingale techniques also allow one to derive time-uniform PAC-Bayes bounds, which are PAC-Bayes bounds that hold with high probability for all rounds $n \geq 1$ simultaneously. In bandit problems, time-uniform bounds are useful for constructing cumultive regret bounds, since we can simply add a time-uniform bound on the regret for each round. Recently, Haddouche and Guedj \cite{haddouche2022supermartingales} proved several time-uniform PAC-Bayes bounds and Chugg et al. \cite{chugg2023unified} proposed a unified framework for deriving time-uniform PAC-Bayes bounds.

\subsection{A Unified PAC-Bayes Bound}
\label{sec:unified_pb_bound}

Using one the canonical assumptions introduced in Section 10.2 of \cite{pena2009self}, we state and prove a unified PAC-Bayes bound, from which all the PAC-Bayes bounds in this survey can be derived as special cases. This unified bound allows us to describe some common features of the PAC-Bayes bounds in this survey. We say that a collection $((A(\pi), B(\pi)), \pi \in \Pi)$ of pairs of random variables indexed by a set $\Pi$, and an interval $\Lambda \subseteq \mathbb{R}$ satisfies the canonical assumption if, for all $\pi \in \Pi$ and all $\lambda \in \Lambda$, $A(\pi) \in \mathbb{R}$, $B(\pi) > 0$, and
\begin{equation}
\mathop{\mathbb{E}}\left[\exp\left(\lambda A(\pi) - \frac{\lambda^2}{2}B(\pi)^2\right)\right] \leq 1.\label{eqn:canon_pair}
\end{equation}

In the general purpose PAC-Bayes bounds for martingales in App. \ref{app:general_martingale_bounds}, $A(\pi)$ is typically a martingale $M_n(\pi)$ at a fixed step $n$. In the PAC-Bayes reward bounds in Sec. \ref{sec:rew_bounds} and regret bounds in Sec. \ref{sec:reg_bounds}, $A(\pi)$ is some form of difference between a reward (or regret) estimate for a policy $\pi$ and the expected reward (or regret) of a policy $\pi$. $B(\pi)$ is typically either a constant or a variance term, which measures some form of variance of $A(\pi)$. Under this canonical assumption, one can obtain an upper bound on $\mathbb{E}_{\pi \sim \rho}\left[A(\pi)\right]$, which holds uniformly over all posteriors $\rho \in \mathcal{P}(\Pi)$. We give a proof of Thm. \ref{thm:canon} shortly after the statement.

\begin{theorem}[Unified PAC-Bayes Bound]
For any $\delta \in (0, 1]$, any $\lambda \in \Lambda \cap (0, \infty)$ and any reference distribution $\mu \in \mathcal{P}(\Pi)$, with probability at least $1-\delta$ (over the randomness of $A(\pi)$ and $B(\pi)$ for each $\pi \in \Pi$), for all $\rho \in \mathcal{P}(\Pi)$ simultaneously:
\begin{equation*}
\mathop{\mathbb{E}}_{\pi \sim \rho}\left[A(\pi)\right] \leq \frac{\lambda}{2}\mathop{\mathbb{E}}_{\pi \sim \rho}\left[B(\pi)^2\right] + \frac{D_{\mathrm{KL}}(\rho||\mu) + \mathrm{ln}(1/\delta)}{\lambda}.
\end{equation*}
\label{thm:canon}
\end{theorem}
\vspace{-0.4cm}

This bound states that, with a specified probability (at least $1 - \delta$), $\mathbb{E}_{\pi \sim \rho}\left[A(\pi)\right]$ is upper bounded by $\mathbb{E}_{\pi \sim \rho}\left[B(\pi)^2\right]$ plus the KL divergence between $\rho$ and $\mu$, and a confidence penalty $\mathrm{ln}(1/\delta)$. The bound is only valid if the parameter $\lambda$ and the prior $\mu$ are both chosen before observing any data. Though we call $\mu$ the prior and $\rho$ the posterior, $\rho$ and $\mu$ are not a Bayesian prior and posterior, i.e. they are not necessarily related to each other via Bayes' rule. In a PAC-Bayes bound, the KL divergence $D_{\mathrm{KL}}(\rho||\mu)$ measures the complexity of the distribution $\rho$ against the distribution $\mu$.

The proof of Thm. \ref{thm:canon} requires the following tool, which is now standard in the PAC-Bayesian literature.

\begin{lemma}[Change of Measure Inequality \cite{donsker1975asymptotic, catoni2004statistical}]
For any measurable function $h: \Pi \to \mathbb{R}$ and any probability distribution $\mu \in \mathcal{P}(\Pi)$, such that $\mathbb{E}_{\pi \sim \mu}[e^{h(\pi)}] < \infty$, we have
\begin{equation*}
\sup_{\rho \in \mathcal{P}(\Pi)}\left\{\mathop{\mathbb{E}}_{\pi \sim \rho}\left[h(\pi)\right] - D_{\mathrm{KL}}(\rho||\mu)\right\} = \mathrm{ln}\left(\mathop{\mathbb{E}}_{\pi \sim \mu}\left[e^{h(\pi)}\right]\right).
\end{equation*}
Furthermore, when $\rho$ has the density function
\begin{equation*}
\rho(\pi) = \frac{\mu(\pi)e^{h(\pi)}}{\mathop{\mathbb{E}}_{\pi^{\prime} \sim \mu}\left[e^{h(\pi^{\prime})}\right]},
\end{equation*}
we have
\begin{equation*}
\mathop{\mathbb{E}}_{\pi \sim \rho}\left[h(\pi)\right] = D_{\mathrm{KL}}(\rho||\mu) + \mathrm{ln}\left(\mathop{\mathbb{E}}_{\pi \sim \mu}\left[e^{h(\pi)}\right]\right).
\end{equation*}
\label{lem:donsk}
\end{lemma}
\vspace{-0.4cm}

This inequality allows for bounds that hold with high probability for all $\rho \in \mathcal{P}(\pi)$ simultaneously, and gives rise to the KL divergence penalty in the bound in Thm. \ref{thm:canon}.

\begin{proof}[Proof of Theorem \ref{thm:canon}]
Before observing the random draw of $A(\pi)$ and $B(\pi)$, we fix our choices of $\lambda \in \Lambda \cap (0, \infty)$, $\mu \in \mathcal{P}(\Pi)$ and $\delta \in (0, 1]$. Then, we start from Equation \ref{eqn:canon_pair} and integrate both sides with respect to $\pi$, which gives
\begin{equation*}
\mathop{\mathbb{E}}_{\pi \sim \mu}\mathop{\mathbb{E}}\left[e^{\lambda A(\pi) - \frac{\lambda^2}{2}B(\pi)^2}\right] \leq 1.
\end{equation*}
Since $\mu$ does not depend on $A(\pi)$ and $B(\pi)$, we can use Tonelli's theorem to swap the order of the expectations. We then have
\begin{equation*}
\mathop{\mathbb{E}}\mathop{\mathbb{E}}_{\pi \sim \mu}\left[e^{\lambda A(\pi) - \frac{\lambda^2}{2}B(\pi)^2}\right] \leq 1.
\end{equation*}
Then, we use the change of measure inequality with $h(\pi) =  \lambda A(\pi) - \frac{\lambda^2}{2}B(\pi)^2$, which gives
\begin{equation*}
\mathop{\mathbb{E}}\left[e^{\sup_{\rho \in \mathcal{P}(\Pi)}\left\{\mathop{\mathbb{E}}_{\pi \sim \rho}\left[\lambda A(\pi) - \frac{\lambda^2}{2}B(\pi)^2\right] - D_{\mathrm{KL}}(\rho||\mu)\right\}}\right] \leq 1.
\end{equation*}
Using Markov's inequality, we have that for our fixed choices of $\lambda$, $\mu$ and $\delta$, with probability at least $1 - \delta$
\begin{equation*}
e^{\sup_{\rho \in \mathcal{P}(\Pi)}\left\{\mathop{\mathbb{E}}_{\pi \sim \rho}\left[\lambda A(\pi) - \frac{\lambda^2}{2}B(\pi)^2\right] - D_{\mathrm{KL}}(\rho||\mu)\right\}} \leq \frac{1}{\delta}.
\end{equation*}
By taking the logarithm of both sides and then rearranging terms, we recover the statement of the theorem.
\end{proof}

From here, one can continue by optimising the bound with respect to the prior $\mu$, and its paramter $\lambda$, as well as any parameters of the estimator. This is the subject of Sec. \ref{sec:optimising_bounds}.

\subsection{Offline Bandit Example}
\label{sec:pac_bayes_bandit_ex}

We continue our introduction to PAC-Bayes bounds for bandits with an example. We present a PAC-Bayes bound for the expected reward $R(\rho)$ in the MAB setting, which was originally proposed by Seldin et al. \cite{seldin2011mab}. Then, we present an offline bandit algorithm that is motivated by this bound. In Sec. \ref{sec:multi_bandit_prob}, we stated that the goal of the offline policy search problem is to choose a distribution $\rho^* \in \mathcal{P}(\Pi)$ that maximises the expected reward, i.e.
\begin{equation*}
\rho^* \in \argmax_{\rho \in \mathcal{P}(\Pi)}\left\{\mathop{\mathbb{E}}_{\pi \sim \rho}\left[R(\pi)\right]\right\}.
\end{equation*}

Since the reward distribution $P_R$ is unknown, we cannot directly maximise $R(\pi)$. However, $R(\pi)$ can be estimated from historical data $D_n$ by the importance sampling estimate $r^{\mathrm{IS}}(\pi, D_n)$. We will assume that for all the behaviour policies $b(\cdot|D_{0}), \dots, b(\cdot|D_{n-1})$, the importance weights $\pi(a)/b(a|D_{i-1})$ are uniformly (over $\pi$, $a$ and $D_{n}$) bounded above by $1/\epsilon_n$. We can maximise $r^{\mathrm{IS}}(\rho, D_n) = \mathop{\mathbb{E}}_{\pi \sim \rho}\left[r^{\mathrm{IS}}(\pi, D_n)\right]$ with respect to $\rho$. However, if the estimate $r^{\mathrm{IS}}(\rho, D_n)$ greatly overestimates the expected reward $R(\rho)$ for even a single choice of $\rho$, simply maximising the reward estimate may result in overfitting. When can we guarantee that $r^{\mathrm{IS}}(\rho, D_n)$ does not greatly overestimate $R(\rho)$? PAC-Bayes bounds can provide an answer.

\begin{theorem}[PAC-Bayes Hoeffding-Azuma bound for $r^{\mathrm{IS}}$ \cite{seldin2011mab}]
For any $\lambda > 0$, any $\delta \in (0, 1)$ and any probability distribution $\mu \in \mathcal{P}(\Pi)$, with probability at least $1 - \delta$ (over the sampling of $D_n$), for all distributions $\rho \in \mathcal{P}(\Pi)$ simultaneously:
\begin{equation*}
R(\rho) \geq r^{\mathrm{IS}}(\rho, D_n) - \frac{\lambda}{8n\epsilon_n^2} - \frac{D_{\mathrm{KL}}(\rho||\mu) + \mathrm{ln}(1/\delta)}{\lambda}.
\end{equation*}
\label{thm:ex_bound}
\end{theorem}
\vspace{-0.4cm}

This bound can be derived from Thm. \ref{thm:canon} with $A(\pi) = r^{\mathrm{IS}}(\pi, D_n) - R(\pi)$ and $B(\pi)^2 = \frac{1}{4n\epsilon_n^2}$. The canonical assumption in (\ref{eqn:canon_pair}) can be verified for this $A(\pi)$ and $B(\pi)$ by applying the classical Hoeffding-Azuma inequality \cite{azuma1967weighted} to the martingale in (\ref{eqn:is_mart_ex}). Thm. \ref{thm:ex_bound} states that if $\rho$ is close to the prior $\mu$ (as measured by the KL divergence) and $r^{\mathrm{IS}}(\rho, D_n)$ is high, then with high probability it is guaranteed that $R(\rho)$ is also high. We can define an offline policy search algorithm that returns the distribution $\widehat{\rho}$ that maximises the lower bound in Thm. \ref{thm:ex_bound}, and hence has the best performance guarantee. The resulting optimisation problem is
\begin{equation}
\widehat{\rho} \in \argmax_{\rho \in \mathcal{P}(\Pi)}\left\{\mathop{\mathbb{E}}_{\pi \sim \rho}[r^{\mathrm{IS}}(\pi, D_n)] - \frac{D_{\mathrm{KL}}(\rho||\mu)}{\lambda}\right\}.\label{eqn:bound_opt}
\end{equation}

The change of measure inequality in Lem. \ref{lem:donsk} shows that the optimisation problem in (\ref{eqn:bound_opt}) has a closed-form solution: $\widehat{\rho}(\pi) \propto \mu(\pi)e^{\lambda r^{\mathrm{IS}}(\pi, D_n)}$. When the policy class $\Pi$ is finite, the normalisation constant of $\widehat{\rho}$ can be calculated by summing over all $\pi \in \Pi$. When $\Pi$ is infinite, one can design algorithms that approximate $\widehat{\rho}$ with variational inference \cite{wainwright2008vi}, \cite{blei2017vi} or algorithms that sample from $\widehat{\rho}$ using Monte Carlo methods \cite{andrieu2003mcmc}, \cite{bardenet2017mcmc}. Of course, if $\Pi$ is a complicated (e.g. high-dimensional) policy class, then approximating or sampling from $\widehat{\rho}$ may be challenging. However, these challenges are beyond the scope of this survey.

\subsection{Relation To Existing Methods}

The basic algorithm in (\ref{eqn:bound_opt}) can provide a new perspective on some well-known principles for policy search.

\begin{example}[Relative Entropy Regularisaton \cite{kakade2001natural}, \cite{peters2003reinforcement}, \cite{bagnell2003covariant}]
Let the policy class be the set of all deterministic policies. Then $\Pi = \mathcal{A}$ and both $\rho$ and $\mu$ are now individual stochastic policies. Suppose there is a single behaviour policy $b$, and set $\mu = b$. The optimisation problem in (\ref{eqn:bound_opt}) becomes:
\begin{equation*}
\widehat{\rho} \in \argmax_{\rho \in \mathcal{P}(\mathcal{A})}\left\{\mathop{\mathbb{E}}_{a \sim \rho}[r^{\mathrm{IS}}(a, D_n)] - \frac{D_{\mathrm{KL}}(\rho||b)}{\lambda}\right\}.
\end{equation*}

This motivates maximising the IS reward estimate subject to a penalty on the relative entropy between $\rho$ and the behaviour policy $b$. Relative Entropy Policy Search \cite{peters2010reps}, Trust Region Policy Optimization \cite{schulman2015trpo} and Proximal Policy Optimization \cite{schulman2017ppo} are all based upon this principle of relative entropy regularisation. \qedclub
\label{ex:relative_entropy}
\end{example}

\begin{example}[Maximum Entropy \cite{ziebart2010maximum}]
Let $\Pi = \mathcal{A}$. This time, choose the prior $\mu$ to be a uniform distribution over $\mathcal{A}$. The KL divergence between $\rho$ and a uniform distribution is equal a constant minus the Shannon entropy $H(\rho)$ of $\rho$. Therefore, the optimisation problem in (\ref{eqn:bound_opt}) becomes:
\begin{equation*}
\widehat{\rho} \in \argmax_{\rho \in \mathcal{P}(\mathcal{A})}\left\{\mathop{\mathbb{E}}_{a \sim \rho}[r^{\mathrm{IS}}(a, D_n)] + \frac{H(\rho)}{\lambda}\right\}.
\end{equation*}

This motivates maximisation of a weighted sum of the reward estimate and the entropy of $\rho$, or alternatively, choosing the policy $\rho$ with the highest entropy subject to a constraint that the reward estimate is sufficiently high. This is essentially equivalent to a classical strategy known as Boltzmann exploration \cite{kaelbling1996survey}. In addition, several modern deep reinforcement learning algorithms, such as Soft Q-learning \cite{haarnoja2017soft} and Soft Actor-Critic \cite{haarnoja2018soft}, follow the maximum entropy principle.\qedclub
\end{example}

\section{PAC-Bayes Reward Bounds}
\label{sec:rew_bounds}

In this section, we give an overview of PAC-Bayes bounds for the expected reward, organised by the reward estimate used in the bound.

\subsection{Importance Sampling}
\label{sec:rew_is_estimate}

We have already encountered the importance sampling (IS) estimate, which was defined in (\ref{eqn:is_est}). We remind the reader of the assumption that the importance weights $\pi(a)/b(a|D_{i-1})$ are uniformly bounded above by $1/\epsilon_n$ for every $i = 1, \dots, n$. This can be achieved by constraining the behaviour policies and/or the policy class $\Pi$.

One of the most well-known PAC-Bayes bounds is the PAC-Bayes $kl$ bound, which was proposed by Seeger \cite{seeger2002pac} and improved by Maurer \cite{maurer2004note}. The binary KL divergence is defined as:
\begin{equation*}
kl(p||q) := p \; \mathrm{ln}\left(\frac{p}{q}\right) + (1 - p)\mathrm{ln}\left(\frac{1 - p}{1 - q}\right).
\end{equation*}

This is the KL divergence between a Bernoulli distribution with parameter $p$ and a Bernoulli distribution with parameter $q$, and is defined for $p, q \in [0, 1]$ (although it is infinite if $q=0$ or $q=1$). Seldin et al. \cite{seldin2011mab} derived the PAC-Bayes $kl$ bound for the IS estimate:
\begin{theorem}[PAC-Bayes $kl$ bound for $r^{\mathrm{IS}}$\cite{seldin2011mab}]
For any $\delta \in (0, 1)$ and any probability distribution $\mu \in \mathcal{P}(\Pi)$, with probability at least $1 - \delta$, for all distributions $\rho \in \mathcal{P}(\Pi)$ simultaneously:
\begin{equation*}
kl\infdivx*{\epsilon_n r^{\mathrm{IS}}(\rho, D_n)}{\epsilon_n R(\rho)} \leq \frac{D_{\mathrm{KL}}(\rho||\mu) + \mathrm{ln}(2\sqrt{n}/\delta)}{n}.
\end{equation*}
\label{thm:pac_bayes_kl}
\end{theorem}
\vspace{-0.4cm}

The original PAC-Bayes $kl$ bound holds only for i.i.d. data, yet the bound in Thm. \ref{thm:pac_bayes_kl} holds even when the behaviour policies are dependent on previous observations. Seldin et al. \cite{seldin2011mab} achieve this extra generality by using a comparison inequality (Lem. 1 of \cite{seldin2012mart}) that bounds expectations of convex functions of certain martingale-like sequences by expectations of the same functions of independent Bernoulli random variables. In this form, the PAC-Bayes $kl$ bound is not so useful; we would prefer a lower bound on $R(\rho)$. Following Seeger \cite{seeger2002pac}, the lower inverse of the binary KL divergence can be defined as:
\begin{equation*}
kl^{-1}(p, b) := \min\{q \in [0, 1] : kl(p||q) \leq b\}.
\end{equation*}

With this definition, the PAC-Bayes $kl$ bound for the $r^{\mathrm{IS}}$ estimate can be rewritten as:
\begin{equation}
R(\rho) \geq \frac{1}{\epsilon_n}kl^{-1}\left(\epsilon_n r^{\mathrm{IS}}(\rho, D_n), \frac{D_{\mathrm{KL}}(\rho||\mu) + \mathrm{ln}(2\sqrt{n}/\delta)}{n}\right).\label{eqn:pac_bayes_kl_inv}
\end{equation}

We refer to this bound as the PAC-Bayes $kl^{-1}$ bound. This is the tightest possible lower bound on $R(\rho)$ that can be derived from the PAC-Bayes $kl$ bound. From the definition of $kl^{-1}$, it is apparent that this bound is never vacuous (less than 0). Unfortunately, $kl^{-1}$ has no closed-form solution. However, it can be calculated numerically to arbitrary accuracy using, for example, the bisection method. Instead of inverting the binary KL divergence, one can use Pinsker's inequality \cite{pinsker1964information}:
\begin{equation*}
|p - q| \leq \sqrt{kl(p||q)/2}.
\end{equation*}

We can then obtain a (looser) high probability lower bound on the expected reward:
\begin{equation}
R(\rho) \geq r^{\mathrm{IS}}(\rho, D_n) - \frac{1}{\epsilon_n}\sqrt{\frac{D_{\mathrm{KL}}(\rho||\mu) + \mathrm{ln}(2\sqrt{n}/\delta)}{2n}}.\label{eqn:pac_bayes_pinsker}
\end{equation}

We refer to this bound as the PAC-Bayes Pinsker bound. Several authors \cite{mcallester2003margin}, \cite{tolstikhin2013pac}, \cite{thiemann2017strongly}, \cite{rivasplata2019pac} have used tighter (than Pinsker's inequality) bounds on the binary KL divergence to obtain better, more explicit PAC-Bayes bounds from the PAC-Bayes $kl$ bound. Similar bounds for the IS reward estimate can be obtained by combining the same techniques with Thm. \ref{thm:pac_bayes_kl}.

Seldin et al. \cite{seldin2012bern} provide a PAC-Bayes bound for the IS estimate that depends on the variance of the reward estimate. The (conditional) average variance of the IS estimate for the policy $\pi$ is defined as:
\begin{equation*}
V^{\mathrm{IS}}(\pi, D_n) = \frac{1}{n}\sum_{i=1}^{n}\mathop{\mathbb{E}}_{\substack{a_i^{\prime} \sim b(\cdot|D_{i-1}) \\ r_i^{\prime} \sim p_R(\cdot|a_i^{\prime})}}\left[\left(\frac{\pi(a_i^{\prime})}{b(a_i^{\prime}|D_{i-1})}r_i^{\prime} - R(\pi)\right)^2\right].
\end{equation*}

This is the average variance of the IS estimate given the observed sequence of behaviour policies. We write $V^{\mathrm{IS}}(\rho, D_n) = \mathbb{E}_{\pi \sim \rho}\left[V^{\mathrm{IS}}(\pi, D_n)\right]$. The bound is derived by using Bernstein's inequality for martingales instead of the Hoeffding-Azuma inequality.

\begin{theorem}[PAC-Bayes Bernstein Bound for $r^{\mathrm{IS}}$ \cite{seldin2012bern}]
For any $\lambda \in [0, n\epsilon_n]$, any $\delta \in (0, 1)$ and any probability distribution $\mu \in \mathcal{P}(\Pi)$, with probability at least $1 - \delta$, for all distributions $\rho \in \mathcal{P}(\Pi)$ simultaneously:
\begin{align*}
R(\rho) &\geq r^{\mathrm{IS}}(\rho, D_n) - \frac{\lambda(e-2)V^{\mathrm{IS}}(\rho, D_n)}{n}\\
&- \frac{D_{\mathrm{KL}}(\rho||\mu) + \mathrm{ln}(1/\delta)}{\lambda}.
\end{align*}
\label{thm:is_bernstein}
\end{theorem}
\vspace{-0.4cm}

Seldin et al. \cite{seldin2012bern} show that the variance for any policy $\pi$ satisfies $V^{\mathrm{IS}}(\pi, D_n) \leq 1/\epsilon_n$. This bound on the variance leads to the following lower bound:
\begin{equation}
R(\rho) \geq r^{\mathrm{IS}}(\rho, D_n) - \frac{\lambda(e-2)}{n\epsilon_n} - \frac{D_{\mathrm{KL}}(\rho||\mu) + \mathrm{ln}(1/\delta)}{\lambda}.\label{eqn:pac_bayes_bern_eps}
\end{equation}

The PAC-Bayes bounds for the IS estimate can be compared by examining their rates in $n$ and $\epsilon_n$. The rate at which they degrade as $\epsilon_n$ approaches 0 becomes particularly important as the action set $\mathcal{A}$ grows. For example, if $\mathcal{A} = \{1, \dots, K\}$ and the behaviour policies are all uniform, then $\epsilon_n \leq 1/K$. Alternatively, if $\mathcal{A}$ is a bounded subset of $\mathbb{R}^d$ and the behaviour policies are all uniform, then $\epsilon_n \leq \mathcal{O}(1/\mathrm{vol}(\mathcal{A}))$. These examples suggest that if a PAC-Bayes bound degrades rapidly as $\epsilon_n$ decreases, then the bound may be loose when $\mathcal{A}$ is large.

The Pinsker bound in (\ref{eqn:pac_bayes_pinsker}) has a rate of $\mathcal{O}(\frac{1}{\epsilon_n\sqrt{n}})$, ignoring the $\mathrm{ln}(\sqrt{n})$ term. For the PAC-Bayes Hoeffding-Azuma bound in Thm. \ref{thm:ex_bound}, it can be shown that the optimal value of $\lambda$ is proportional to $\epsilon_n\sqrt{n}$. With this choice of $\lambda$, this bound also has a rate of $\mathcal{O}(\frac{1}{\epsilon_n\sqrt{n}})$. The PAC-Bayes Bernstein bound in (\ref{eqn:pac_bayes_bern_eps}) has an improved rate in $\epsilon_n$. For a suitable choice of $\lambda$, this bound has a rate of $\mathcal{O}(\frac{1}{\sqrt{\epsilon_n n}})$. A Taylor expansion reveals that the PAC-Bayes $kl^{-1}$ bound in (\ref{eqn:pac_bayes_kl_inv}) decays approximately exponentially in $\frac{1}{n\epsilon_n}$ as $\frac{r^{\mathrm{IS}}(\rho, D_n)}{e}\exp\frac{-D_{\mathrm{KL}}(\rho||\mu) - \mathrm{ln}(2\sqrt{n}/\delta)}{n\epsilon_n r^{\mathrm{IS}}(\rho, D_n)}$. Based on these rates, we can expect the PAC-Bayes Bernstein and PAC-Bayes $kl^{-1}$ bounds to scale better to large action sets.

Finally, we discuss PAC-Bayes bounds for the IS estimate in the contextual bandit setting. In the CB setting, the IS estimate is defined as:
\begin{equation}
r^{\mathrm{IS}}(\pi, D_n) = \frac{1}{n}\sum_{i=1}^{n}\frac{\pi(a_i|s_i)}{b(a_i|s_i, D_{i-1})}r_i.\label{eqn:cb_is}
\end{equation}

We still require that $1/\epsilon_n$ is a uniform bound on the importance weights, though now the importance weights are $\pi(s|a)/b(a|s, D_{i-1})$. As in the MAB setting, one can construct martingales containing the IS estimate that are compatible with the Hoeffding-Azuma, Bernstein and Seldin et al.'s comparison inequality \cite{seldin2012mart}. Therefore, the same PAC-Bayes Hoeffding-Azuma, PAC-Bayes $kl$ and PAC-Bayes Bernstein bounds as in Theorems \ref{thm:ex_bound}, \ref{thm:pac_bayes_kl} and \ref{thm:is_bernstein} can be derived. In the CB versions of these bounds, $\epsilon_n$ and the reward estimate $r^{\mathrm{IS}}(\pi, D_n)$ are just defined slightly differently. A PAC-Bayes Bernstein bound for the IS estimate in the CB setting was first derived by Seldin et al. \cite{seldin2011cb}.

\subsection{Clipped Importance Sampling}
\label{sec:rew_cis_estimate}

In Section \ref{sec:rew_is_estimate}, we saw that all the PAC-Bayes bounds for the IS reward estimate degrade as the uniform bound $1/\epsilon_n$ on the importance weights increases. One way to ensure that $1/\epsilon_n$ is never too large is to clip the importance weights. Clipped (or truncated) importance sampling was first proposed by Ionides \cite{ionides2008trunc}. The clipped importance sampling (CIS) reward estimate for MAB problems is defined as:
\begin{equation}
r^{\mathrm{CIS}}(\pi, D_n) = \frac{1}{n}\sum_{i=1}^{n}\mathrm{min}\left(\frac{\pi(a_i)}{b(a_i|D_{i-1})}, \frac{1}{\tau}\right)r_i.
\label{eqn:mab_cis_est}
\end{equation}

By definition, the clipped importance weights are bounded above by $1/\tau$. However, clipping the importance weights introduces bias. Let $R^{\mathrm{CIS}}(\pi) = \mathbb{E}_{D_n}[r^{\mathrm{CIS}}(\pi, D_n)]$ denote the expected value of the CIS estimate, and let $R^{\mathrm{CIS}}(\rho) = \mathbb{E}_{\pi \sim \rho}[R^{\mathrm{CIS}}(\pi)]$. It can be shown (see Lemma \ref{lem:cis_bias}) that the CIS estimate is biased to underestimate the expected reward, i.e. $R^{\mathrm{CIS}}(\rho) \leq R(\rho)$.

Therefore, any lower bound on $R^{\mathrm{CIS}}(\rho)$ is also a lower bound on $R(\rho)$, which means the we can essentially ignore the bias of the CIS estimate if we only require a lower bound on the expected reward. It is possible to derive PAC-Bayes Hoeffding-Azuma, Bernstein and $kl$ bounds for the CIS estimate that are almost the same as those for the IS estimate, except that $\epsilon_n$ is replaced by $\tau$. Under the assumption that there is a single, fixed behaviour policy, Wang et al. \cite{wang2019pac} have proved PAC-Bayes Pinsker, Hoeffding-Azuma and Bernstein bounds for the CIS risk (one minus reward) estimate.

In Appendix \ref{sec:cis_ha_proof}, we prove the following PAC-Bayes Hoeffding-Azuma bound for the CIS reward estimate, which holds in the most general setting where the data are drawn from an arbitrary sequence of behaviour policies.

\begin{theorem}[PAC-Bayes Hoeffding-Azuma bound for $r^{\mathrm{CIS}}$]
For any $\tau \in (0, 1]$, any $\lambda > 0$, any $\delta \in (0, 1)$ and any probability distribution $\mu \in \mathcal{P}(\Pi)$, with probability at least $1 - \delta$, for all distributions $\rho \in \mathcal{P}(\Pi)$ simultaneously:
\begin{equation}
R(\rho) \geq r^{\mathrm{CIS}}(\rho, D_n) - \frac{\lambda}{8n\tau^2} - \frac{D_{\mathrm{KL}}(\rho||\mu) + \mathrm{ln}(1/\delta)}{\lambda}.\label{eqn:cis_pac_bayes_ha}
\end{equation}
\label{thm:cis_pac_bayes_ha}
\end{theorem}
\vspace{-0.4cm}

Like $\lambda$, the clipping parameter $\tau$ must be independent of the data $D_n$. We don't claim that this bound is new, since it is essentially a corollary of the PAC-Bayes Hoeffding-Azuma bound for martingales by Seldin et al. \cite{seldin2012mart}. A PAC-Bayes Bernstein bound for the CIS estimate can also be drived in the general case where the data are drawn from an arbitrary sequence of behaviour policies. We define the average variance of the CIS estimate as:
\begin{equation}
V^{\mathrm{CIS}}(\pi, D_n) = \frac{1}{n}\sum_{i=1}^{n}\mathop{\mathbb{E}}_{\substack{a_i^{\prime} \sim b(\cdot|D_{i-1}) \\ r_i^{\prime} \sim p_R(\cdot|a_i^{\prime})}}\left[\left(X_i^{\mathrm{CIS}}(\pi)\right)^2\right],\label{eqn:cis_var_def}
\end{equation}

where $X_i^{\mathrm{CIS}}(\pi) = \min\left(\frac{\pi(a_i^{\prime})}{b(a_i^{\prime}|D_{i-1})}, \frac{1}{\tau}\right)r_i^{\prime} - \mathbb{E}_{a_i^{\prime} \sim b(\cdot|D_{i-1}), r_i^{\prime} \sim p_R(\cdot|a_i^{\prime})}\left[\min\left(\frac{\pi(a_i^{\prime})}{b(a_i^{\prime}|D_{i-1})}, \frac{1}{\tau}\right)r_i^{\prime}\right]$. One can show (see Lemma \ref{lem:cis_var_mab}) that the average variance of the CIS estimate satisfies $V^{\mathrm{CIS}}(\pi, D_n) \leq 1/\tau$. In Appendix \ref{sec:cis_bern_proof}, we prove the following PAC-Bayes Bernstein bound for the CIS estimate.

\begin{theorem}[PAC-Bayes Bernstein Bound for $r^{\mathrm{CIS}}$]
For any $\tau \in (0, 1]$, any $\lambda \in [0, n\tau]$, any $\delta \in (0, 1)$ and any probability distribution $\mu \in \mathcal{P}(\Pi)$, with probability at least $1 - \delta$, for all distributions $\rho \in \mathcal{P}(\Pi)$ simultaneously:
\begin{align*}
R(\rho) &\geq r^{\mathrm{CIS}}(\rho, D_n) - \frac{\lambda(e-2)V^{\mathrm{CIS}}(\rho, D_n)}{n}\\
&- \frac{D_{\mathrm{KL}}(\rho||\mu) + \mathrm{ln}(1/\delta)}{\lambda}.
\end{align*}
\label{thm:cis_pac_bayes_bernstein}
\end{theorem}
\vspace{-0.4cm}

Applying the variance bound $V^{\mathrm{CIS}}(\pi, D_n) \leq 1/\tau$ gives the following high probability lower bound:
\begin{equation}
R(\rho) \geq r^{\mathrm{CIS}}(\rho, D_n) - \frac{\lambda(e-2)}{n\tau} - \frac{D_{\mathrm{KL}}(\rho||\mu) + \mathrm{ln}(1/\delta)}{\lambda}.\label{eqn:cis_pac_bayes_bernstein}
\end{equation}

This bound is essentially a corollary of the PAC-Bayes Bernstein bound for martingales by Seldin et al. \cite{seldin2012mart}. A PAC-Bayes $kl$ bound for the CIS estimate has so far only been proven for the case where the data are all drawn from a fixed behaviour policy. In this case, the CIS estimate is a sum of independent random variables bounded in $[0, 1/\tau]$. Therefore, one can apply Seeger's original PAC-Bayes $kl$ bound \cite{seeger2002pac} to the CIS estimate, scaled by a factor of $\tau$.

\begin{theorem}[PAC-Bayes $kl$ bound for $r^{\mathrm{CIS}}$]
If the data set $D_n$ is drawn from a single, fixed behaviour policy, then for any $\tau \in (0, 1]$, any $\delta \in (0, 1)$ and any probability distribution $\mu \in \mathcal{P}(\Pi)$, with probability at least $1 - \delta$, for all distributions $\rho \in \mathcal{P}(\Pi)$ simultaneously:
\begin{equation*}
kl\infdivx*{\tau r^{\mathrm{CIS}}(\rho, D_n)}{\tau R^{\mathrm{CIS}}(\rho)} \leq \frac{D_{\mathrm{KL}}(\rho||\mu) + \mathrm{ln}(2\sqrt{n}/\delta)}{n}.
\end{equation*}
\label{thm:cis_pac_bayes_kl}
\end{theorem}
\vspace{-0.4cm}

Since $R(\rho) \geq R^{\mathrm{CIS}}(\rho)$, one can still use Pinsker's inequality or the inverse of $kl$ to obtain lower bounds on $R(\rho)$. If we invert the Binary KL divergence, we obtain the $kl^{-1}$ bound for the CIS estimate:
\begin{equation}
R(\rho) \geq \frac{1}{\tau}kl^{-1}\left(\tau r^{\mathrm{CIS}}(\rho, D_n), \frac{D_{\mathrm{KL}}(\rho||\mu) + \mathrm{ln}(2\sqrt{n}/\delta)}{n}\right).\label{eqn:cis_pac_bayes_kl_inv}
\end{equation}

If we use Pinsker's inequality, we obtain:
\begin{equation}
R(\rho) \geq r^{\mathrm{CIS}}(\rho, D_n) - \frac{1}{\tau}\sqrt{\frac{D_{\mathrm{KL}}(\rho||\mu) + \mathrm{ln}(2\sqrt{n}/\delta)}{2n}}.\label{eqn:cis_pac_bayes_pinsker}
\end{equation}

The discussion about rates in Section \ref{sec:rew_is_estimate} applies to the bounds for CIS estimate, although the rates in $\epsilon_n$ are now rates in $\tau$. The PAC-Bayes Bernstein and $kl^{-1}$ bounds both have improved rates in $\tau$ and should therefore be preferred when $\tau$ is close to $0$.

Next, we will discuss PAC-Bayes bounds for the CIS estimate in the contextual bandit setting. In the CB setting, the CIS estimate is defined as:
\begin{equation}
r^{\mathrm{CIS}}(\pi, D_n) = \frac{1}{n}\sum_{i=1}^{n}\mathrm{min}\left(\frac{\pi(a_i|s_i)}{b(a_i|s_i, D_{i-1})}, \frac{1}{\tau}\right)r_i.
\label{eqn:cb_cis_est}
\end{equation}

As in the MAB setting, the CIS estimate is biased to underestimate $R(\rho)$ (meaning $R^{\mathrm{CIS}}(\rho) \leq R(\rho)$), so we can still essentially ignore it in the one-sided PAC-Bayes reward bounds. The PAC-Bayes Hoeffding-Azuma and Bernstein bounds for the CIS estimate, in Theorem \ref{thm:cis_pac_bayes_ha} and Theorem \ref{thm:cis_pac_bayes_bernstein}, also hold in the CB setting. When there is a fixed behaviour policy, the CB CIS estimate is still an average of i.i.d. random variables bounded in $[0, 1/\tau]$, so the PAC-Bayes $kl$ bound in Theorem \ref{thm:cis_pac_bayes_kl} also holds in the CB setting.

Finally, we will briefly describe a method for offline contextual bandits that is motivated by a PAC-Bayes bound for the CIS estimate. London and Sandler \cite{london2019bayesian} use a PAC-Bayes upper bound on the expected risk (1 minus the expected reward).

\begin{theorem}[PAC-Bayes risk bound using $r^{\mathrm{CIS}}$ \cite{london2019bayesian}]
If the data set $D_n$ is drawn from a single, fixed behaviour policy, then for any $\tau \in (0, 1]$, any $\delta \in (0, 1)$ and any probability distribution $\mu \in \mathcal{P}(\Pi)$, with probability at least $1 - \delta$, for all distributions $\rho \in \mathcal{P}(\Pi)$ simultaneously:
\begin{align*}
1 &- R(\rho) \leq 1 - r^{\mathrm{CIS}}(\rho, D_n)\\
&+ \sqrt{\frac{2\left(1/\tau - r^{\mathrm{CIS}}(\rho, D_n)\right)\left(D_{\mathrm{KL}}(\rho||\mu) + \mathrm{ln}(2\sqrt{n}/\delta)\right)}{\tau n}}\\
&+ \frac{2\left(D_{\mathrm{KL}}(\rho||\mu) + \mathrm{ln}(2\sqrt{n}/\delta)\right)}{\tau n}.
\end{align*}
\label{thm:cis_pac_bayes_risk}
\end{theorem}
\vspace{-0.4cm}

This bound can be derived from the PAC-Bayes $kl$ bound in Theorem \ref{thm:cis_pac_bayes_kl} by using a tighter version of Pinsker's inequality, suggested by McAllester \cite{mcallester2003margin}. London and Sandler choose a Gaussian prior centred at the behaviour policy. This bound motivates Logging (behaviour) Policy Regularisation \cite{london2019bayesian}, which selects a policy that has a small risk estimate (high reward estimate) and is close to the behaviour policy. This method is reminiscent of relative entropy regularisation, seen in Example \ref{ex:relative_entropy}.

\subsection{Weighted Importance Sampling}
\label{sec:rew_wis_estimate}

In Appendix \ref{sec:app_wis_estimate}, we present PAC-Bayes bounds for another reward estimate called the weighted importance sampling (WIS) estimate. The WIS estimate is not a sum of i.i.d. random variables or even the sum of a martingale difference sequence, so it is not compatible with the bounds presented in previous sections. Kuzborskij and Szepesv{\'a}ri \cite{kuzborskij2019efron} derived a very general Efron-Stein PAC-Bayes bound and showed that it can be used to upper bound the difference between the WIS estimate and its expected value. In Theorem \ref{thm:efron_stein_pac_bayes}, we present a modified version of this bound, which has slightly better constants and holds under weaker assumptions.

However, the WIS estimate is a biased estimate of the expected reward. We are not aware of any empirical bounds on the bias of the WIS estimate in the literature that don't require further assumptions on the reward distribution $P_R$. Therefore, we cannot obtain bounds on the difference between the WIS estimate and the expected reward.

\section{PAC-Bayes Regret Bounds}
\label{sec:reg_bounds}

In this section, we give an overview of PAC-Bayes bounds for the cumulative regret $\Delta(\rho_{1:n})$ associated with a sequence of distributions $\rho_1, \dots, \rho_n$ over the policy class $\Pi$. We focus here on the MAB problem and present similar regret bounds for contextual bandits in App. \ref{sec:app_cb_regret}. First, we state some PAC-Bayes bounds on the expected regret for a single round. Then, we present some PAC-Bayes bounds on the cumulative regret.

In the MAB setting, we consider the case when the set of actions is finite: $\mathcal{A} = \{1, \dots, K\}$. We set the policy class to be the set of all deterministic policies, so $\Pi = \mathcal{A}$. In this case, any distribution $\rho$ over the policy class $\mathcal{A}$ is a single stochastic policy. The IS estimate of the reward for a deterministic policy (an action) $a$ can be defined as:
\begin{equation}
r^{\mathrm{IS}}(a, D_n) = \frac{1}{n}\sum_{i=1}^{n}\frac{\mathbb{I}\{a_i = a\}}{b(a_i|D_{i-1})}r_i.
\end{equation}

This coincides with the earlier definition for general policy classes. For this choice of policy class, a uniform upper bound on the importance weights $\mathbb{I}\{a_i = a\}/b_i(a_i) \leq 1/\epsilon_n$ can be achieved by a uniform lower bound on the behaviour policy probabilities $b(a_i|D_{i-1})$, i.e. $b(a_i|D_{i-1}) \geq \epsilon_n$. The regret for an action $a$ is defined as:
\begin{equation*}
\Delta(a) = R(a^*) - R(a),
\end{equation*}

where $a^*$ is an action that maximises the expected reward. The IS regret estimate for an action $a$ is defined as:
\begin{equation}
\Delta^{\mathrm{IS}}(a, D_n) = r^{\mathrm{IS}}(a^*, D_n) - r^{\mathrm{IS}}(a, D_n).
\label{eqn:mab_is_regret}
\end{equation}

Seldin et al. \cite{seldin2011mab} \cite{seldin2012bern} showed that a martingale, which is compatible with both the Hoeffding-Azuma inequality and Bernstein’s inequality, can be constructed from the IS regret estimate. Consequently, one can obtain a PAC-Bayes Hoeffding-Azuma bound and a PAC-Bayes Bernstein bound on the difference between the expected regret and the IS regret estimate.

\begin{theorem}[PAC-Bayes Hoeffding-Azuma bound for $\Delta^{\mathrm{IS}}$ \cite{seldin2011mab}]
For any $\delta \in (0, 1)$, with probability at least $1 - \delta$, for all distributions $\rho \in \mathcal{P}(\mathcal{A})$ and all $n \geq 1$ simultaneously:
\begin{equation*}
\Delta(\rho) - \Delta^{\mathrm{IS}}(\rho, D_n) \leq \sqrt{\frac{2(\mathrm{ln}(K) + 2\mathrm{ln}(n+1) + \mathrm{ln}(1/\delta))}{n\epsilon_n^2}}.
\end{equation*}
\label{thm:is_reg_pac_bayes_hoeffding}
\end{theorem}
\vspace{-0.4cm}

One can observe several differences between this bound and the PAC-Bayes Hoeffding-Azuma bound in Theorem \ref{thm:ex_bound}. This bound uses a uniform prior $\mu$, and since both $\rho$ and $\mu$ are distributions over a finite set with $K$ elements, $D_{\mathrm{KL}}(\rho||\mu) \leq \mathrm{ln}(K)$. Hence, the KL divergence has been replaced with $\mathrm{ln}(K)$. This bound holds with probability at least $1 - \delta$ for all $n \geq 1$ simultaneously, rather than for a single $n \geq 1$. This is achieved by a union bound argument, discussed in Section \ref{sec:data_dep_priors2}, and introduces the $2\mathrm{ln}(n+1)$ term. Finally, this bound does not contain $\lambda$. This is because, for each $n$, we have set $\lambda_n = \sqrt{2n\epsilon_n^2(\mathrm{ln}(K) + 2\mathrm{ln}(n+1) + \mathrm{ln}(1/\delta))}$.

The PAC-Bayes Bernstein bound for the IS regret estimate contains (an upper bound on) the average variance of the IS regret estimate $V^{\mathrm{IS}}(a, D_n)$, which is defined as:
\begin{equation*}
V^{\mathrm{IS}}(a, D_n) = \frac{1}{n}\sum_{i=1}^{n}\mathop{\mathbb{E}}_{\substack{a_i^{\prime} \sim b(\cdot|D_{i-1}) \\ r_i^{\prime} \sim p_R(\cdot|a_i^{\prime})}}\left[\left(X_i^{\mathrm{IS}}(a)\right)^2\right],
\end{equation*}

where $X_i^{\mathrm{IS}}(a) = \frac{\mathbb{I}\{a_i^{\prime} = a^*\}}{b(a_i^{\prime}|D_{i-1})}r_i^{\prime} - \frac{\mathbb{I}\{a_i^{\prime} = a\}}{b(a_i^{\prime}|D_{i-1})}r_i^{\prime} - \Delta(a)$. Seldin et al. show that in both the MAB \cite{seldin2012bern} and CB settings \cite{seldin2011cb}, this average variance can be bounded as $V^{\mathrm{IS}}(a, D_n) \leq 2/\epsilon_n$. Using this bound on the variance, the following PAC-Bayes bound can be derived.

\begin{theorem}[PAC-Bayes Bernstein bound for $\Delta^{\mathrm{IS}}$ \cite{seldin2012bern}]
For any $\delta \in (0, 1)$, with probability at least $1 - \delta$, for all distributions $\rho \in \mathcal{P}(\mathcal{A})$ and all $n \geq 1$ simultaneously, where
\begin{equation}
\frac{\mathrm{ln}(K) + 2\mathrm{ln}(n+1) + \mathrm{ln}(1/\delta)}{2(e-2)n} \leq \epsilon_n,\label{eqn:reg_bern_n}
\end{equation}

we have that:
\begin{align*}
\Delta(\rho) &\leq \Delta^{\mathrm{IS}}(\rho, D_n)\\
&+\sqrt{\frac{8(e-2)(\mathrm{ln}(K) + 2\mathrm{ln}(n+1) + \mathrm{ln}(1/\delta))}{\epsilon_n n}}.
\end{align*}
\label{thm:is_reg_pac_bayes_bernstein}
\end{theorem}
\vspace{-0.4cm}

In this bound we have set $\lambda_n = \sqrt{n \epsilon_n (\mathrm{ln}(K) + 2\mathrm{ln}(n+1) + \mathrm{ln}(1/\delta))/(2(e-2))}$. The requirement that $\lambda_n \in [0, n\epsilon_n]$ becomes the requirement on $n$ in Equation \ref{eqn:reg_bern_n}. Following Seldin et al. \cite{seldin2011cb}, \cite{seldin2012bern}, we present cumulative regret bounds for a family of MAB algorithms. Let $\tilde{\rho}_n^{\mathrm{exp}}$ denote the following smoothed Gibbs policy:
\begin{align}
\rho_n^{\mathrm{exp}}(a) &\propto \mu(a)e^{\gamma_nr^{\mathrm{IS}}(a, D_n)},\label{eqn:exp3}\\
\tilde{\rho}_n^{\mathrm{exp}}(a) &= (1 - K\epsilon_{n+1})\rho_n^{\mathrm{exp}}(a) + \epsilon_{n+1}.\nonumber
\end{align}

If $\epsilon_n = 1/\sqrt{nK}$ and $\gamma_n = \sqrt{n\mathrm{ln}(K)/K}$, this strategy is known as EXP3 \cite{auer2002exp3}. Alternatively, in the limit as $\gamma_n$ tends to infinity, we obtain the $\epsilon$-greedy algorithm \cite{auer2002finite}. Since $\epsilon_n$ cannot be greater than $1/K$, we truncate $\epsilon_n$ to always be no more than $1/K$. The first step is to re-write the regret for a single round as:
\begin{align}
\Delta(\tilde{\rho}_n^{\mathrm{exp}}) &= \Delta(\rho_n^{\mathrm{exp}}) - \Delta^{\mathrm{IS}}(\rho_n^{\mathrm{exp}}, D_n)\nonumber\\
&+ \Delta^{\mathrm{IS}}(\rho_n^{\mathrm{exp}}, D_n) + R(\rho_n^{\mathrm{exp}}) - R(\tilde{\rho}_n^{\mathrm{exp}}).\label{eqn:regret_decomposition}
\end{align}

Seldin et al. \cite{seldin2012bern} show that $\Delta^{\mathrm{IS}}(\rho_n^{\mathrm{exp}}, D_n) \leq \mathrm{ln}(K)/\gamma_n$ and that $R(\rho_n^{\mathrm{exp}}) - R(\tilde{\rho}_n^{\mathrm{exp}}) \leq K\epsilon_{n+1}$. If the PAC-Bayes Hoeffding-Azuma bound in Theorem \ref{thm:is_reg_pac_bayes_hoeffding} is used to bound $\Delta(\rho_n^{\mathrm{exp}}) - \Delta^{\mathrm{IS}}(\rho_n^{\mathrm{exp}}, D_n)$, then we obtain the following cumulative regret bound.

\begin{theorem}[PAC-Bayes Hoeffding-Azuma cumulative regret bound \cite{seldin2011mab}, \cite{seldin2012bern}]
Let $\epsilon_n = n^{-1/4}K^{-1/2}$ and take any $\gamma_n$ such that $\gamma_n \geq n^{1/4}K^{-1/2}\sqrt{\mathrm{ln}(K)}$. For any $\delta \in (0, 1]$, with probability at least $1 - \delta$, for all $n \geq 1$ simultaneously, the cumulative regret is bounded by:
\begin{align*}
\sum_{i=1}^{n}\Delta(\tilde{\rho}_i^{\mathrm{exp}}) &\leq n^{3/4}K^{1/2}\bigg(1 + \sqrt{\mathrm{ln}(K)}\\
&+ \sqrt{2(\mathrm{ln}(K) + 2\mathrm{ln}(n+1) + \mathrm{ln}(1/\delta))}\bigg).
\end{align*}
\label{thm:is_hoeffding_regret}
\end{theorem}
\vspace{-0.4cm}

Ignoring log terms, this cumulative regret bound is of order $\mathcal{O}(n^{3/4}K^{1/2})$. If the PAC-Bayes Bernstein bound in Theorem \ref{thm:is_reg_pac_bayes_bernstein} is used to bound $\Delta(\rho_n^{\mathrm{exp}}) - \Delta^{\mathrm{IS}}(\rho_n^{\mathrm{exp}}, D_n)$, then we obtain the following cumulative regret bound.

\begin{theorem}[PAC-Bayes Bernstein cumulative regret bound \cite{seldin2012bern}]
Let $\epsilon_n = n^{-1/3}K^{-2/3}$ and take any $\gamma_n$ such that $\gamma_n \geq n^{1/3}K^{-1/3}\sqrt{\mathrm{ln}(K)}$. For any $\delta \in (0, 1]$, with probability at least $1 - \delta$, for all $n \geq 1$ simultaneously, where $n$ satisfies:
\begin{equation*}
n \geq K\left(\frac{\mathrm{ln}(K) + 2\mathrm{ln}(n+1) + \mathrm{ln}(1/\delta)}{2(e-2)}\right)^{3/2},
\end{equation*}

the cumulative regret is bounded by:
\begin{align*}
\sum_{i=1}^{n}\Delta(\tilde{\rho}_i^{\mathrm{exp}}) &\leq n^{2/3}K^{1/3}\bigg(1 + \sqrt{\mathrm{ln}(K)}\\
&+ 2\sqrt{2(e-2)\left(\mathrm{ln}(K) + 2\mathrm{ln}(n+1) + \mathrm{ln}(1/\delta)\right)}\bigg).
\end{align*}
\label{thm:is_bernstein_regret}
\end{theorem}
\vspace{-0.4cm}

Ignoring log terms, this cumulative regret bound is of order $\mathcal{O}(n^{2/3}K^{1/3})$. The improved scaling with $K$ and $n$ is due to the PAC-Bayes Bernstein bound having improved dependence on $\epsilon_n$. Sadly, this regret bound has a sub-optimal growth rate in $n$. Audibert and Bubeck \cite{audibert2009minimax} have shown that the cumulative regret for EXP3 can be upper bounded by a term of order $\mathcal{O}(\sqrt{nK\mathrm{ln}(K)})$. Moreover, Audibert and Bubeck show that the best possible worst-case regret bound that any algorithm can achieve is $\mathcal{O}(\sqrt{nK})$.

Seldin et al. \cite{seldin2012bern} hypothesise that the PAC-Bayes cumulative regret bound in Theorem \ref{thm:is_reg_pac_bayes_bernstein} can be improved for the EXP3 algorithm with $\epsilon_n = 1/\sqrt{nK}$ and $\gamma_n = \sqrt{n\mathrm{ln}(K)/K}$. They suggest, and verify empirically, that for this choice of $\epsilon_n$ and $\gamma_n$, the bound on the average variance can be tightened to $V^{\mathrm{IS}}(\rho_n^{\mathrm{exp}}, D_n) \leq 2K$. Using this bound on the average variance, and ignoring log terms, the cumulative regret bound in Theorem \ref{thm:is_reg_pac_bayes_bernstein} would become $\mathcal{O}(\sqrt{nK})$.

\section{Optimising PAC-Bayes Bandit Bounds}
\label{sec:optimising_bounds}

\subsection{The Choice of Prior}
\label{sec:priors}

In this section, we give an overview of methods for choosing the prior. We first motivate the utility of "good" priors, using the PAC-Bayes Hoeffding-Azuma bound from Thm. \ref{thm:ex_bound} (shown below) as an example.
\begin{equation*}
R(\rho) \geq r^{\mathrm{IS}}(\rho, D_n) - \frac{\lambda}{8n\epsilon_n^2} - \frac{D_{\mathrm{KL}}(\rho||\mu) + \mathrm{ln}(1/\delta)}{\lambda}.
\end{equation*}

This lower bound is largest when $r^{\mathrm{IS}}(\rho, D_n)$ is large and $D_{\mathrm{KL}}(\rho||\mu)$ is close to 0. To achieve this, $\mu$ must assign high probability to policies where $r^{\mathrm{IS}}(\pi, D_n)$ is large. This motivates priors that either depend on the data set $D_n$ (data-dependent priors) or on the distribution of the data set (distribution-dependent priors). In fact, Dziugaite et al. \cite{dziugaite2021data} have shown that data-dependent priors are sometimes necessary for tight PAC-Bayes bounds. PAC-Bayes bounds with data/distribution-dependent priors are of practical interest because they can yield tighter performance guarantees. They are also of theoretical interest because they can yield bounds with improved rates.

We now detail various approaches for deriving PAC-Bayes bounds with data/distribution-dependent priors. Many of these techniques are compatible with essentially any PAC-Bayes bound. Where this is the case, we apply them to the PAC-Bayes $kl$ bound for the IS estimate as an example, since we will later compare the PAC-Bayes $kl^{-1}$ bound with various priors in our experiments.

\subsubsection{Data-Dependent Priors via Sample Splitting}
\label{sec:data_dep_priors1}

One way to use a data-dependent prior is to split the data set into disjoint subsets $D_n = D_{1:m} \cup D_{m+1:n}$, of size $m$ and $n - m$, for some $m < n$. The first subset is used to learn a prior $\mu_{D_{1:m}}$. A PAC-Bayes bound is then evaluated on the second subset with the learned prior. Since $\mu_{D_{1:m}}$ does not depend on $D_{m+1:n}$, this prior is a valid choice when the bound is evaluated on the second subset. The PAC-Bayes $kl$ bound with this data-dependent prior is:

\begin{theorem}[PAC-Bayes $kl$ Bound with Sample Splitting]
For any $\delta \in (0, 1)$ and any prior $\mu_{D_{1:m}} \in \mathcal{P}(\Pi)$ that may depend on the subset $D_{1:m}$, with probability at least $1 - \delta$, for all $\rho \in \mathcal{P}(\Pi)$ simultaneously:
\begin{align*}
kl\infdivx*{\epsilon_n r^{\mathrm{IS}}(\rho, D_{m+1:n})}{\epsilon_n R(\rho)} &\leq \frac{D_{\mathrm{KL}}(\rho||\mu_{D_{1:m}})}{n-m}\\
&+ \frac{\mathrm{ln}(2\sqrt{n-m}/\delta)}{n-m}.
\end{align*}
\label{thm:pac_bayes_kl_subset_bound}
\end{theorem}
\vspace{-0.4cm}

This approach is very flexible, since the data-dependent prior can be learned in any way. We believe that Seeger \cite{seeger2002pac} was the first to use this technique. Subsequently, it has been used by others, such as Catoni \cite{catoni2003pac}, Ambroladze et al. \cite{ambroladze2007tighter}, and Germain et al. \cite{germain2009pac}. Recently, this approach has been used to obtain non-vacuous generalisation bounds for deep neural networks \cite{rivasplata2019pac}, \cite{perez2021tighter}, \cite{perez2021learning}, \cite{perez2021progress}.

\subsubsection{Data-Dependent Priors Selected From a Restricted Set of Priors}
\label{sec:data_dep_priors2}

Another way to use data-dependent priors is to define a set of priors in advance and then derive a modified PAC-Bayes bound that holds with high probability simultaneously for all priors in this set. One can then evaluate the modified PAC-Bayes bound with any prior from this set. The modified bound will contain an extra penalty that we must pay in order for the bound to hold for more than one prior.

Suppose we have a countable set of priors $\{\mu_i\}_{i=1}^{\infty}$ and we want the PAC-Bayes $kl$ bound to hold with probability $1 - \delta$ for all $\mu_i$ simultaneously. We have that for each $i$, with probability at least $1 - \delta_i$:
\begin{equation*}
kl\infdivx*{\epsilon_n r^{\mathrm{IS}}(\rho, D_n)}{\epsilon_n R(\rho)} \leq \frac{D_{\mathrm{KL}}(\rho||\mu_i) + \mathrm{ln}(2\sqrt{n}/\delta_i)}{n}.
\end{equation*}

By the union bound, this bound holds with probability at least $1 - \sum_{i=1}^{\infty}\delta_i$ for all $\rho$ and all $i \in \mathbb{N}$ simultaneously. We can freely choose $\{\delta_i\}_{i=1}^{\infty}$ such that $\sum_{i=1}^{\infty}\delta_i = \delta$. Therefore, at the cost of replacing $\delta$ with $\delta_i$, we can choose the prior in $\{\mu_i\}_{i=1}^{\infty}$ that results in the greatest lower bound.

This technique has previously been used to obtain parametric priors with data-dependent parameters, e.g. Gaussian priors with data-dependent variance \cite{langford2002not}, \cite{dziugaite2017computing}. It can also been derived by using a prior that is a mixture of several priors $\mu = \sum_{i=1}^{\infty}p_i\mu_i$ \cite{ambroladze2007tighter}, \cite{parrado2012pac}. The weights $p_i$ must satisfy $p_i > 0$ and $\sum_i^{\infty}p_i = 1$. This results in the same bound with $\delta_i = p_i\delta$.

A set of priors can also be defined by fixing a learning algorithm and then restricting the choice of prior to be one that is learned from the data using this learning algorithm. If the learning algorithm is stable, meaning the prior it selects is almost unaffected by small changes to the data, then we call the learned prior a stable prior. Dziugaite and Roy \cite{dziugaite2018data} and Rivasplata et al. \cite{rivasplata2020pac} obtain PAC-Bayes bounds with stable priors, where the stability of a prior is characterised by differential privacy.

Let $A: \mathcal{Z}^n \rightsquigarrow \mathcal{P}(\Pi)$ denote a randomised algorithm that maps a data set $D_n \in \mathcal{Z}^n$ to a prior $\mu \in \mathcal{P}(\Pi)$. Also, let the data set $D_n$ consist of $n$ i.i.d. samples.

\begin{definition}[Differential privacy]
A randomised algorithm $A: \mathcal{Z}^n \rightsquigarrow \mathcal{P}(\Pi)$ is $\eta$-\textit{differentially private} if for all pairs $D_n, D_n^{\prime} \in \mathcal{Z}^n$ that differ at only one coordinate, and all measurable subsets $B \subseteq \mathcal{P}(\Pi)$, we have:
\begin{equation*}
\mathbb{P}(A(D_n) \in B) \leq e^{\eta}\mathbb{P}(A(D_n^{\prime}) \in B).
\end{equation*}
\end{definition}

Dziugaite and Roy \cite{dziugaite2018data} show that any PAC-Bayes bound that holds for any data-independent prior $\mu$ with probability at least $1 - \delta^{\prime}$ can be turned into a PAC-Bayes bound that holds for any $\eta$-differentially private prior $\mu_{D_n}$ with probability at least $1 - \delta$ by replacing $D_{\mathrm{KL}}(\rho||\mu)$ with $D_{\mathrm{KL}}(\rho||\mu_{D_n})$ and replacing $\mathrm{ln}(1/\delta^{\prime})$ with $n\eta^2/2 + \eta\sqrt{n\mathrm{ln}(4/\delta)/2} + \mathrm{ln}(2/\delta)$. For example, the PAC-Bayes $kl$ bound becomes:

\begin{theorem}[PAC-Bayes $kl$ Bound with a Differentially Private Prior \cite{dziugaite2018data}]
If the data set $D_n$ is drawn from a single, fixed behaviour policy, then for any $\delta \in (0, 1)$ and any $\eta$-differentially private prior $\mu_{D_n} \in \mathcal{P}(\Pi)$, with probability at least $1 - \delta$, for all $\rho \in \mathcal{P}(\Pi)$ simultaneously:
\begin{align*}
kl\infdivx*{\epsilon_n r^{\mathrm{IS}}(\rho, D_n)}{\epsilon_n R(\rho)} &\leq \frac{D_{\mathrm{KL}}(\rho||\mu_{D_n}) + \mathrm{ln}(4\sqrt{n}/\delta)}{n}\\
&+ \frac{n\eta^2/2 + \eta\sqrt{n\mathrm{ln}(4/\delta)/2}}{n}.
\end{align*}
\label{thm:pac_bayes_kl_dp_bound}
\end{theorem}
\vspace{-0.4cm}

Since differential privacy is defined only for data sets consiting of i.i.d. samples, this bound only holds when the data are all drawn from a single, fixed behaviour policy (to ensure that the data are i.i.d.).

\subsubsection{Distribution-Dependent Priors}
\label{sec:dist_dep_priors}

The motivation for using a data-dependent prior was that it would assign high probability to policies where $r^{\mathrm{IS}}(\pi, D_n)$ is large. Assuming $r^{\mathrm{IS}}(\pi, D_n)$ is close to $R(\pi)$, we could instead use a prior that assigns high probability to policies where $R(\pi)$ is large, such as $\tilde{\mu}(\pi) \propto \mathrm{exp}(R(\pi))$. This prior is data-independent, but we cannot calculate the KL divergence between $\rho$ and this prior, since $R(\pi)$ is unknown. Lever et al. \cite{lever2010distribution}, \cite{lever2013tighter} provide a method for upper bounding the KL divergence between restricted sets of posteriors and distribution-dependent priors.

We restrict ourselves to empirical Gibbs posteriors $\rho_{\gamma}$ and Gibbs priors $\mu_{\gamma}$ that have density functions of the following form:
\begin{equation*}
\rho_{\gamma}(\pi) \propto \mu(\pi)e^{\gamma r^{\mathrm{IS}}(\pi, D_n)}, \qquad \mu_{\gamma}(\pi) \propto \mu(\pi)e^{\gamma R(\pi)}.
\end{equation*}

$\mu$ is a data-independent reference distribution. Lever et al. \cite{lever2010distribution} show that for this choice of $\rho_{\gamma}$ and $\mu_{\gamma}$:
\begin{align}
D_{\mathrm{KL}}(\rho_{\gamma}||\mu_{\gamma}) &\leq \gamma \left(r^{\mathrm{IS}}(\rho_{\gamma}, D_n) - R(\rho_{\gamma})\right)\label{eqn:lever_kl1}\\
&+ \gamma \left(R(\mu_{\gamma}) - r^{\mathrm{IS}}(\mu_{\gamma}, D_n)\right).\nonumber
\end{align}

Both expected values on the right-hand-side of Equation \ref{eqn:lever_kl1} can be upper bounded using any of the PAC-Bayes bounds for the IS reward estimate, with $\mu_{\gamma}$ as the prior. If the Pinsker bound is used, this results in a quadratic inequality for $D_{\mathrm{KL}}(\rho_{\gamma}||\mu_{\gamma})$, which holds with probability at least $1 - \delta$. The solution of this inequality tells us that with probability at least $1 - \delta$:
\begin{equation*}
D_{\mathrm{KL}}(\rho_{\gamma}||\mu_{\gamma}) \leq \frac{2\gamma}{\epsilon\sqrt{2n}}\sqrt{\mathrm{ln}(2\sqrt{n}/\delta)} + \frac{\gamma^2}{2n\epsilon^2}.
\end{equation*}

See \cite{lever2010distribution} or \cite{seldin2011mab} for a detailed derivation. This upper bound can be substituted into the PAC-Bayes $kl$ bound.

\begin{theorem}[PAC-Bayes $kl$ Lever Bound \cite{lever2010distribution}, \cite{seldin2011mab}]
For any $\gamma > 0$ and $\delta \in (0, 1)$, with probability at least $1 - \delta$:
\begin{align*}
kl\infdivx*{\epsilon_n r^{\mathrm{IS}}(\rho_{\gamma}, D_n)}{\epsilon_n R(\rho_{\gamma})} &\leq \frac{\mathrm{ln}(4\sqrt{n}/\delta)}{n} + \frac{\gamma^2}{2n^2\epsilon_n^2}\\
&+ \frac{\gamma\sqrt{2\mathrm{ln}(4\sqrt{n}/\delta)}}{n\sqrt{n}\epsilon_n}.
\end{align*}
\label{thm:pac_bayes_kl_lever_bound}
\end{theorem}
\vspace{-0.4cm}

Since the PAC-Bayes Pinsker bound, which was used to upper bound the right-hand side of Equation \ref{eqn:lever_kl1}, holds when the data are drawn from a sequence of dependent behaviour policies, so does the PAC-Bayes $kl$ Lever bound. The best value of $\gamma$ will be large enough for $r^{\mathrm{IS}}(\rho_{\gamma}, D_n)$ to be large, but not so large that the bound is dominated by the $\gamma$-dependent terms. When using the distribution-dependent prior $\mu_{\gamma}$, it is still helpful to have an informative reference distribution $\mu$, since then $\gamma$ can be close to 0 and $\rho_{\gamma}$ will still have high estimated reward. Lever et al.'s method of upper bounding the distribution-dependent KL divergence can be applied more generally to other kinds of Gibbs posteriors and priors. See \cite{lever2010distribution} or \cite{lever2013tighter} for more information.

In the case where the data $D_n$ are i.i.d., Oneto et al. \cite{oneto2016pac} proved a tighter upper bound on $D_{\mathrm{KL}}(\rho_{\gamma}||\mu_{\gamma})$. Oneto et al. \cite{oneto2016pac} also proved another PAC-Bayes bound for empirical Gibbs posteriors. This bound only holds when the data $D_n$ are i.i.d., so when there is a single, fixed behaviour policy. Let $\rho_{\gamma}^{\setminus i}$ denote the leave-one-out Gibbs posterior, which has the density function: $\rho_{\gamma}^{\setminus i}(\pi) \propto \mu(\pi)\exp({\frac{\gamma}{n}\sum_{j=1, j \neq i}^{n}\frac{\pi(a_j)}{b(a_j)}r_j})$. This is the empirical Gibbs posterior with the $i$th datum removed. Following Oneto et al., one can show that any posterior that is symmetric (meaning it does not depend on the order of the training data), and has a certain distribution stability property, satisfies the following bound.

\begin{theorem}[Distribution stability bound \cite{oneto2016pac}]
If the data set $D_n$ is drawn from a single, fixed behaviour policy and if the method for selecting the posterior $\rho$ and leave-one-out posteriors $\rho^{\setminus i}$ from the data set $D_n$ satisfies the distribution stability property:
\begin{equation*}
\max_{a^{\prime}, r^{\prime}}\left\{\left|\mathop{\mathbb{E}}_{\pi \sim \rho}\left[\frac{\pi(a^{\prime})}{b(a^{\prime})}r^{\prime}\right] - \mathop{\mathbb{E}}_{\pi \sim \rho^{\setminus i}}\left[\frac{\pi(a^{\prime})}{b(a^{\prime})}r^{\prime}\right]\right|\right\} \leq \beta,
\end{equation*}

then for all $D_n$, all $i \in \{1, \dots, n\}$, and with $\beta$ that goes to 0 as $\mathcal{O}(1/n)$, then with probability at least $1 - \delta$:
\begin{equation*}
\left|R(\rho) - r^{\mathrm{IS}}(\rho, D_n)\right| \leq 2\beta + \left(4n\beta + \frac{1}{\epsilon_n}\right)\sqrt{\frac{\mathrm{ln}(2/\delta)}{2n}}.
\end{equation*}
\label{thm:pac_bayes_dist_stab}
\end{theorem}
\vspace{-0.4cm}

Oneto et al. show that the Gibbs posterior satisfies the distribution stability property with $\beta \leq \frac{2\gamma}{n\epsilon_n}$. Therefore, using Theorem \ref{thm:pac_bayes_dist_stab}, we have that, if the data $D_n$ are drawn from a single behaviour policy, then the following bound holds for the Gibbs posterior with probability at least $1 - \delta$:
\begin{equation}
\left|R(\rho_{\gamma}) - r^{\mathrm{IS}}(\rho_{\gamma}, D_n)\right| \leq \frac{4\gamma}{n\epsilon_n} + \left(\frac{8\gamma}{\epsilon_n} + \frac{1}{\epsilon_n}\right)\sqrt{\frac{\mathrm{ln}(2/\delta)}{2n}}.\label{eqn:pac_bayes_gibbs_dist_stab}
\end{equation}

Once again, there is a trade-off between setting $\gamma$ large enough for the empirical reward to be high, but not so large that the $\gamma$-dependent penalty terms become too large.

Finally, we present another technique based on algorithmic stability for deriving PAC-Bayes bounds with certain distribution-dependent priors, which is due to Rivasplata et al. \cite{rivasplata2018pac}. Let the data set $D_n \in \mathcal{Z}^n$ consist of $n$ i.i.d. samples. Let $D_n^{(i)}$ be the data set $D_n$, except with it's $i$th element $z_i$ replaced with $z_i^{\prime}$. The hypothesis sensitivity coefficients are defined as:

\begin{definition}[Hypothesis sensitivity coefficients \cite{rivasplata2018pac}]
Consider a learning algorithm $A: \mathcal{Z}^n \to \mathcal{H}$ that maps a data set to hypothesis in a separable Hilbert space $\mathcal{H}$. The hypothesis sensitivity coefficients of $A$ are defined as:
\begin{equation*}
\beta_n = \sup_{i \in [n]}\sup_{z_i, z_i^{\prime}}\left\{\norm{A(D_n) - A(D_n^{(i)})}_{\mathcal{H}}\right\}.
\end{equation*}
\end{definition}

Rivasplata et al. use the posterior $\rho_{A}$ and the distribution-dependent prior $\mu_{A}$, which are defined as:
\begin{equation*}
\rho_{A} = \mathcal{N}(A(D_n), \sigma^2I), \qquad \mu_{A} = \mathcal{N}(\mathbb{E}_{D_n}[A(D_n)], \sigma^2I).
\end{equation*}

The KL divergence between $\rho_{A}$ and $\mu_{A}$ is equal to $\norm{A(D_n) - \mathbb{E}_{D_n}[A(D_n)]}_{\mathcal{H}}^2/(2\sigma^2)$. Rivasplata et al. show that if the algorithm $A$ has hypothesis sensitivity coefficients $\beta_n$, then the output of the algorithm $A(D_n)$ satisfies a concentration inequality, which implies an upper bound on $D_{\mathrm{KL}}(\rho_{A}||\mu_{A})$. With probability at least $1 - \delta$:
\begin{equation*}
\norm{A(D_n) - \mathbb{E}_{D_n}[A(D_n)]}_{\mathcal{H}} \leq \sqrt{n}\beta_n\left(1 + \sqrt{\frac{1}{2}\mathrm{ln}\left(\frac{1}{\delta}\right)}\right).
\end{equation*}

One can then use the union bound to combine any PAC-Bayes bound using the posterior $\rho_{A}$ and prior $\mu_{A}$ with the concentration inequality satisfied by the algorithm $A$. The PAC-Bayes $kl$ bound becomes:

\begin{theorem}[PAC-Bayes $kl$ Hypothesis Sensitivity Bound \cite{rivasplata2018pac}]
If the data set $D_n$ is drawn from a single, fixed behaviour policy, then for any $\delta \in (0, 1)$ and any algorithm $A$ with hypothesis sensitivity coefficients $\beta_n$, with probability at least $1 - \delta$:
\begin{align*}
kl\infdivx*{\epsilon_n r^{\mathrm{IS}}(\rho_A, D_n)}{\epsilon_n R(\rho_A)} &\leq \frac{\mathrm{ln}(4\sqrt{n}/\delta)}{n}\\
&+ \frac{n\beta_n^2\left(1 + \sqrt{\mathrm{ln}(2/\delta)/2}\right)^2}{2\sigma^2}.
\end{align*}
\label{thm:pac_bayes_kl_hypo_sens_bound}
\end{theorem}
\vspace{-0.4cm}

Unlike the previous techniques using distribution-dependent priors, this time there is no explicit dependence on a data-independent reference distribution.

\subsubsection{Data-Dependent Approximations of Distribution-Dependent Priors}
\label{sec:local_priors}

Distribution-dependent Gibbs priors were first used by Catoni \cite{catoni2004statistical}. Catoni proved that the KL divergence between an arbitrary posterior $\rho$ and a distribution-dependent Gibbs prior can be upper bounded by the KL divergence between $\rho$ and an empirical (data-dependent) Gibbs prior. We are not aware of any way to apply this technique to the PAC-Bayes $kl$ bound. Therefore, we apply the technique, described in Section 1.3.4. of \cite{catoni2007pac}, to the PAC-Bayes Bernstein bound.

First, we define the distribution-dependent Gibbs prior $\mu_{\beta R}$ and the data-dependent Gibbs prior $\mu_{\beta r^{\mathrm{IS}}}$ as:
\begin{equation*}
\mu_{\beta R}(\pi) = \frac{\mu(\pi)e^{\beta R(\pi)}}{\mathop{\mathbb{E}}_{\pi \sim \mu}\left[e^{\beta R(\pi)}\right]}, \quad \mu_{\beta r^{\mathrm{IS}}}(\pi) = \frac{\mu(\pi)e^{\beta r^{\mathrm{IS}}(\pi, D_n)}}{\mathop{\mathbb{E}}_{\pi \sim \mu}\left[e^{\beta r^{\mathrm{IS}}(\pi, D_n)}\right]}.
\end{equation*}

Catoni showed that $D_{\mathrm{KL}}(\rho||\mu_{\beta R})$ is related to $D_{\mathrm{KL}}(\rho||\mu_{\beta r^{\mathrm{IS}}})$ in the following way:
\begin{align*}
D_{\mathrm{KL}}(\rho||\mu_{\beta R}) &= D_{\mathrm{KL}}(\rho||\mu_{\beta r^{\mathrm{IS}}}) + \beta\mathop{\mathbb{E}}_{\pi \sim \rho}\left[r^{\mathrm{IS}}(\pi, D_n) - R(\pi)\right]\\
&+ \mathrm{ln}\left(\mathop{\mathbb{E}}_{\pi \sim \mu}\left[e^{\beta R(\pi)}\right]\right) - \mathrm{ln}\left(\mathop{\mathbb{E}}_{\pi \sim \mu}\left[e^{\beta r^{\mathrm{IS}}(\pi, D_n)}\right]\right).
\end{align*}

If we could upper bound $\mathrm{ln}\left(\mathop{\mathbb{E}}_{\pi \sim \mu}\left[\mathrm{exp}(\beta R(\pi))\right]\right) - \mathrm{ln}\left(\mathop{\mathbb{E}}_{\pi \sim \mu}\left[\mathrm{exp}(\beta r^{\mathrm{IS}}(\pi, D_n))\right]\right)$, then we could upper bound the difference between $D_{\mathrm{KL}}(\rho||\mu_{\beta R})$ and $D_{\mathrm{KL}}(\rho||\mu_{\beta r^{\mathrm{IS}}})$. We could then combine the PAC-Bayes Bernstein bound, using the prior $\mu_{\beta R}$, with the upper bound on $D_{\mathrm{KL}}(\rho||\mu_{\beta R})$, to obtain a ``localised" PAC-Bayes Bernstein bound for the IS estimate. In Appendix \ref{sec:local_bern_proof}, we show how this can be done, and that the result is the following bound.

\begin{theorem}[Localised PAC-Bayes Bernstein Bound for $r^{\mathrm{IS}}$ \cite{catoni2007pac}, \cite{seldin2012bern}]
For any $\lambda \in [0, n\epsilon_n]$, any $\beta$ satisfying $0 \leq \beta < \lambda$, any $\delta \in (0, 1)$ and any probability distribution $\mu \in \mathcal{P}(\Pi)$, with probability at least $1 - \delta$, for all distributions $\rho \in \mathcal{P}(\Pi)$ simultaneously:
\begin{align*}
R(\rho) &\geq r^{\mathrm{IS}}(\rho, D_n) - \frac{(\lambda^2 + \beta^2)(e-2)}{(\lambda - \beta)n\epsilon_n}\\
&- \frac{D_{\mathrm{KL}}(\rho||\mu_{\beta r^{\mathrm{IS}}}) + 2\mathrm{ln}(1/\delta)}{\lambda - \beta}.
\end{align*}
\label{thm:local_pac_bayes_bernstein}
\end{theorem}
\vspace{-0.4cm}

For more information about Catoni's localisation technique and its consequences, see \cite{catoni2007pac}.

Finally, we describe two more techniques for obtaining PAC-Bayes bounds with data-dependent priors that are similar to the localisation technique. The first, by London and Sandler \cite{london2019bayesian}, also uses a data-dependent approximation of a distribution-dependent prior. We require i.i.d. data $D_n = \{z_i\}_{i=1}^{n}$ and we restrict the posterior and prior to be $d$-dimensional Gaussian distributions:
\begin{equation*}
\rho_{\bs{\theta}} = \mathcal{N}(\bs{\theta}, \sigma^2I), \qquad \mu_{\widehat{\bs{\theta}}} = \mathcal{N}(\mathbb{E}_{D}[\widehat{\bs{\theta}}], \sigma^2I).
\end{equation*}

For example, $\rho$ and $\mu$ may be distributions over policy parameters and $\widehat{\bs{\theta}}$ could be an estimate of the parameters of the behaviour policy or the optimal policy. Define $\widehat{\bs{\theta}}$ as:
\begin{equation*}
\widehat{\bs{\theta}} = \argmin_{\bs{\theta}}\left\{\frac{1}{n}\sum_{i=1}^{n}L(\bs{\theta}, z_i) + \lambda\norm{\bs{\theta}}_2^2\right\}.
\end{equation*}

London and Sandler \cite{london2019bayesian} show that if $L(\cdot, z_i)$ is convex and $\beta$-Lipschitz for any $z_i$, then $\widehat{\bs{\theta}}$ satisfies a concentration inequality. With probability at least $1 - \delta$:
\begin{equation*}
||\widehat{\bs{\theta}} - \mathbb{E}_{D}[\widehat{\bs{\theta}}]||_2^2 \leq \frac{\beta}{\lambda}\sqrt{\frac{2\mathrm{ln}(2/\delta)}{n}}.
\end{equation*}

This can be used to upper bound the KL divergence between $\rho_{\bs{\theta}}$ and $\mu_{\widehat{\bs{\theta}}}$ with high probability. The PAC-Bayes $kl$ bound with $\rho_{\bs{\theta}}$ and $\mu_{\widehat{\bs{\theta}}}$ is:

\begin{theorem}[PAC-Bayes $kl$ London and Sandler Bound \cite{london2019bayesian}]
If the data set $D_n$ is drawn from a single, fixed behaviour policy, then for any $\delta \in (0, 1)$, with probability at least $1 - \delta$ and for all $\bs{\theta} \in \mathbb{R}^d$ simultaneously:
\begin{align*}
kl&\infdivx*{\epsilon_n r^{\mathrm{IS}}(\rho_{\bs{\theta}}, D_n)}{\epsilon_n R(\rho_{\bs{\theta}})} \leq \frac{\mathrm{ln}(4\sqrt{n}/\delta)}{n}\\
&+ \frac{\left(||\bs{\theta} - \widehat{\bs{\theta}}||_2 + (\beta/\lambda)\sqrt{2\mathrm{ln}(4/\delta)/n}\right)^2}{2n\sigma^2}.
\end{align*}
\label{thm:pac_bayes_kl_london_bound}
\end{theorem}
\vspace{-0.4cm}

This is similar to the PAC-Bayes $kl$ Hypothesis Sensitivity bound in Theorem \ref{thm:pac_bayes_kl_hypo_sens_bound}. Here though, the mean of the Gaussian posterior $\rho_{\bs{\theta}}$ is unrestricted and the bound contains a data-dependent penalty term $||\bs{\theta} - \widehat{\bs{\theta}}||_2$.

Rivasplata et al. \cite{rivasplata2020pac} propose another method for deriving PAC-Bayes bounds with data-dependent Gibbs priors. In the proof of the PAC-Bayes Hoeffding-Azuma bound (Theorem \ref{thm:ex_bound}) in Section \ref{sec:pacb_ex}, the first two steps, where Lemma \ref{lem:donsk} and Markov's inequality are used, do not require the prior to be data-independent. If we follow the proof of Theorem \ref{thm:ex_bound}, but stop after the first two steps, we have that for any $\mu$, with probability at least $1 - \delta$ and all $\rho$ simultaneously:
\begin{align}
R(\rho) &\geq r^{\mathrm{IS}}(\rho, D_n) - \frac{D_{\mathrm{KL}}(\rho||\mu) + \mathrm{ln}(1/\delta)}{\lambda}\label{eqn:ex_dgibbs}\\
&- \frac{\mathrm{ln}\left(\mathop{\mathbb{E}}_{D_n}\mathop{\mathbb{E}}_{\pi \sim \mu}\left[e^{\lambda\left(r^{\mathrm{IS}}(\pi, D_n) - R(\pi)\right)}\right]\right)}{\lambda}.\nonumber
\end{align}

If we could upper bound $\mathrm{ln}\left(\mathop{\mathbb{E}}_{D_n}\mathop{\mathbb{E}}_{\pi \sim \mu}\left[\mathrm{exp}\left(\lambda\left(r^{\mathrm{IS}}(\pi, D_n) - R(\pi)\right)\right)\right]\right)$ for a data-dependent $\mu$, then we would obtain a PAC-Bayes bound with a data-dependent prior. Following Rivasplata et al. \cite{rivasplata2020pac}, if the data $D_n$ are i.i.d. and the prior is $\mu_{\beta r^{\mathrm{IS}}}$ then:
\begin{align*}
\mathrm{ln}\left(\mathop{\mathbb{E}}_{D_n}\mathop{\mathbb{E}}_{\pi \sim \mu_{\beta r^{\mathrm{IS}}}}\left[e^{\lambda\left(r^{\mathrm{IS}}(\pi, D_n) - R(\pi)\right)}\right]\right) &\leq \frac{2}{\epsilon_n^2}\left(1 + \frac{2\lambda\beta}{n}\right)\\
&+ \mathrm{ln}\left(1 + e^{\frac{\lambda^2}{2n\epsilon_n^2}}\right).
\end{align*}

Combining this with Equation \ref{eqn:ex_dgibbs}, we obtain:

\begin{theorem}[PAC-Bayes Hoeffding-Azuma Empirical Gibbs Bound\cite{rivasplata2020pac}]
If the data set $D_n$ is drawn from a single, fixed behaviour policy, then for any $\lambda > 0$, any $0 \leq \beta \leq \lambda$, any $\delta \in (0, 1)$ and any probability distribution $\mu \in \mathcal{P}(\Pi)$, with probability at least $1 - \delta$, for all distributions $\rho \in \mathcal{P}(\Pi)$ simultaneously:
\begin{align*}
R(\rho) &\geq r^{\mathrm{IS}}(\rho, D_n) - \frac{2}{\lambda\epsilon_n^2} - \frac{4\beta}{n\epsilon_n^2}\\
&- \frac{D_{\mathrm{KL}}(\rho||\mu_{\beta r^{\mathrm{IS}}}) + \mathrm{ln}((1 + e^{\frac{\lambda^2}{2n\epsilon_n^2}})/\delta)}{\lambda}.
\end{align*}
\label{thm:pac_bayes_ha_emp_gibbs}
\end{theorem}
\vspace{-0.4cm}

This is similar to the localised PAC-Bayes Bernstein bound. However, this bound only holds for i.i.d. data and has worse dependence on $\epsilon_n$.

\subsubsection{Priors Learned From Other Data Sets}

PAC-Bayesian meta learning \cite{pentina2014pac}, \cite{amit2018meta}, \cite{rothfuss2021pacoh}, \cite{rothfuss2021fpacoh}, \cite{liu2021statistical}, \cite{theresa2021transfer}, \cite{meunier2021meta}, \cite{flynn2022pac} is another line of work in which priors are learned from data. These methods use data sets from previous learning tasks to learn a distribution over priors. Flynn et al. \cite{flynn2022pac} propose PAC-Bayes bounds for meta-learning priors over the policy class for multi-armed bandit problems.

\subsection{Optimising Bound Parameters}
\label{sec:bound_params}

Many PAC-Bayes bounds contain parameters that must be set before observing the data, such as $\lambda$ in the PAC-Bayes Bernstein bound in Theorem \ref{thm:is_bernstein}. We would like to be able to choose optimal values of these parameters, however the optimal values are often data-dependent. In Appendix \ref{sec:app_bound_params}, we present methods for approximately optimising PAC-Bayes bounds with respect to their parameters. These methods are very similar to the sample splitting and union bound priors, so we defer this to the Appendix.

\section{Experimental Comparison}
\label{sec:comparison}

In this section, we compare the values and properties of the presented PAC-Bayes bandit bounds. In Section \ref{sec:benchmarks}, we describe the benchmark tasks on which we evaluate the bounds. Then, we discuss insights gained from our experiments. In Section \ref{sec:regret_comparison}, we compare the cumulative regret bounds. In Section \ref{sec:reward_comparison}, we compare the reward bounds.

\subsection{Benchmarks}
\label{sec:benchmarks}

We used three benchmark tasks, one multi-armed bandit problem and two contextual bandit problems.

\subsubsection{MAB Binary Benchmark}

The first benchmark is a multi-armed bandit problem with a finite set of actions $\mathcal{A} = \{1, \dots, K\}$. The rewards are always either 0 or 1, and the reward distribution for action $a_i$ is a Bernoulli distribution with parameter $p_i$. The Bernoulli parameters $p_i$ are drawn uniformly from the interval $[0, 0.8]$ and one action always has $p_i = 0.8$. For the policy class $\Pi$, we use the set of all deterministic policies, so $\Pi = \mathcal{A}$. We report results averaged over several instances of this problem, with different randomly generated Bernoulli parameters.

\subsubsection{CB Binary Linear Benchmark}

The next benchmark is a contextual bandit problem where the optimal policy is a linear function of the state. The set of states is $\mathcal{S} = \mathbb{R}^d$ and the set of actions is $\mathcal{A} = \{1, \dots, K\}$. The state distribution $P_S$ is a standard Gaussian distribution. The rewards are either 0 or 1. When creating an instance of this problem, we sample a linear classifier:
\begin{equation*}
f(s; \theta^*) = \argmax_{a \in \{1, \dots, K\}}\left\{\inner{s}{\theta^*}_a\right\}.
\end{equation*}

The weight matrix $\theta^{*} \in \mathbb{R}^{d \times K}$ is drawn from a standard Gaussian distribution. $\inner{s}{\theta^*}_a$ is the $a$th element of $\inner{s}{\theta^*}$. For a given state $s$ and action $a$, if $a = f(s; \theta^*)$, then the reward is drawn from a Bernoulli distribution with parameter 0.8. Otherwise, the reward is drawn from a Bernoulli distribution with parameter 0.2. The policy class contains all linear softmax policies:
\begin{equation*}
\Pi = \left\{\pi_{\theta}(a|s) = \frac{\mathrm{exp}(\inner{s}{\theta}_a)}{\sum_{a^{\prime}}\mathrm{exp}(\inner{s}{\theta}_{a^{\prime}})}\bigg| \theta \in \mathbb{R}^{d \times K}\right\}.
\end{equation*}

\subsubsection{CB Classification Benchmark}

For the final benchmark task, we turned four classification data sets found on OpenML \cite{vanschoren2013openml} and the UCI Machine Learning Repository \cite{dua2019uci} into contextual bandit problems. The states are the covariates of the classification problem, the actions are predicted class labels and the rewards are 1 if the action matches the true class label and 0 otherwise. In the resulting contextual bandit problems, $\mathcal{S} \subseteq \mathbb{R}^d$, where $d$ is between $7$ and $64$, and $\mathcal{A} = \{1, \dots, K\}$, where $K$ is $10$ or $11$. See Appendix \ref{sec:class_data_sets} for more information about the data sets used. For the policy class, we use multi-layer perceptrons with two hidden layers of 200 units each. The final layer has a softmax activation function and the remaining layers have the Elu activation function \cite{clevert2015elu} with $\alpha = 1$.

\subsection{Regret Bounds}
\label{sec:regret_comparison}

In Section \ref{sec:reg_bounds}, we saw several PAC-Bayes cumulative regret bounds for certain (online) multi-armed bandit algorithms. We now evaluate these bounds and algorithms as well as the PAC-Bayes cumulative regret bound that would be possible if the improved bound on the variance of the IS estimate suggested by Seldin et al. \cite{seldin2012bern} was proven for EXP3.

In the MAB Binary benchmark, with $K = 10$, we compared the multi-armed bandit algorithms described in Equation \ref{eqn:exp3} with several settings of $\gamma_n$ and $\epsilon_n$. Motivated by the PAC-Bayes Hoeffding-Azuma cumulative regret bound, we tested $\epsilon$-greedy with $\epsilon_n = n^{-1/4}K^{-1/2}$ and an EXP3-like algorithm with $\gamma_n = n^{1/4}K^{-1/2}\sqrt{\mathrm{ln}(K)}$ and $\epsilon_n = n^{-1/4}K^{-1/2}$. We call these algorithms HA $\epsilon$-greedy and HA EXP3 repectively. Motivated by the PAC-Bayes Bernstein cumulative regret bound, we tested $\epsilon$-greedy with $\epsilon_n = n^{-1/3}K^{-2/3}$ and an EXP3-like algorithm with $\gamma_n = n^{1/3}K^{-1/3}\sqrt{\mathrm{ln}(K)}$ and $\epsilon_n = n^{-1/3}K^{-2/3}$. We call these algorithms Bern $\epsilon$-greedy and Bern EXP3 repectively. Finally, we run (standard) EXP3 and the UCB1 algorithm \cite{auer2002finite} for comparison.

We evaluate each cumulative regret bound with $\delta = 0.05$. For HA $\epsilon$-greedy and Bern $\epsilon$-greedy, the $\sqrt{\mathrm{ln}(K)}$ term in their cumulative regret bounds (Theorem \ref{thm:is_hoeffding_regret} and Theorem \ref{thm:is_bernstein_regret}) can be removed, so we report different bound values for the $\epsilon$-greedy and EXP3-like variants.

\begin{experimentmessage}
The multi-armed bandit algorithms motivated by the PAC-Bayes Hoeffding-Azuma and Bernstein cumulative regret bounds all performed poorly compared to EXP3 and UCB1.
\end{experimentmessage}

Figure \ref{fig:mab_regret} shows both the actual cumulative regret (left) and the PAC-Bayes bounds on the cumulative regret (right) over 10000 steps. Surprisingly, EXP3 had very similar cumulative regret to UCB1. HA EXP3, HA $\epsilon$-greedy, Bern EXP3 and Bern $\epsilon$-greedy all had much higher cumulative regret. The PAC-Bayes cumulative regret bounds for each of these algorithms were very loose, each being more than a factor of 10 above the actual cumulative regret and worse than the trivial bound at $n=10000$. Note that each of the PAC-Bayes cumulative regret bounds would eventually drop below the trivial bound for large enough $n$. While the hypothetical PAC-Bayes bound for EXP3 is much lower than the other bounds, it is still not really tight. The average cumulative regret for EXP3 at $n=10000$ was roughly 415, whereas the bound was roughly 4500.

\begin{figure}[H]
\centering
\includegraphics[width=1.0\columnwidth]{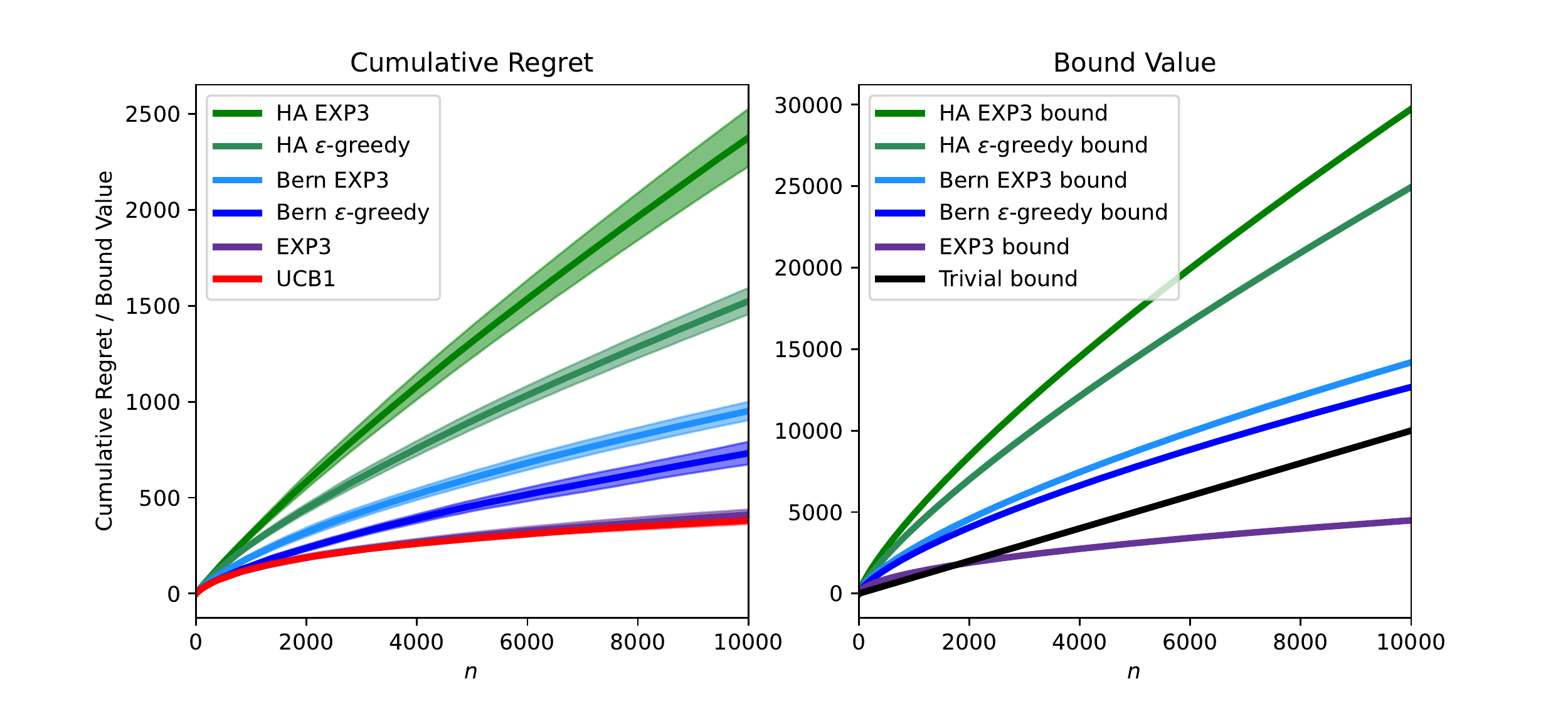}
\caption{Comparison of the MAB algorithms and bounds in the MAB Binary benchmark with $K = 10$. The left plot shows the average cumulative regret plus/minus 1 standard deviation for each algorithm. The right plot shows the cumulative regret bounds, each with $\delta = 0.05$. The EXP3 bound is the cumulative regret bound that would be possible if the improved bound on the variance, suggested by Seldin et al. \cite{seldin2012bern}, was proven. The trivial bound assumes maximum regret at each round.}
\label{fig:mab_regret}
\end{figure}

\begin{experimentmessage}
The PAC-Bayes Hoeffding-Azuma and Bernstein cumulative regret bounds are both loose. If the improved bound on the variance of the IS estimate suggested by Seldin et al. \cite{seldin2012bern} was proven for EXP3, then a much better (though still not really tight) PAC-Bayes cumulative regret bound would be possible.
\end{experimentmessage}

\subsection{Reward Bounds}
\label{sec:reward_comparison}

In this section, we present our observations about the PAC-Bayes reward bounds for the IS and CIS estimates. Since we are not aware of a bound on the bias term in the Efron-Stein WIS bound, we only evaluate it in Appendix \ref{sec:efron_stein_experiment}, assuming the bias is 0. We compare the bounds in the (offline) MAB Binary and CB Binary Linear benchmarks. In each experiment we optimise each bound with respect to the posterior $\rho$ and then report the value of the bound and the expected reward for this $\rho$. This allows us to compare the best possible value of each bound as well as which bound works the best as a learning objective. For details about how we optimise the various bounds with respect to $\rho$ and then evaluate them, see Appendix \ref{sec:opt_and_eval}.

We always use a data set of size $n=1000$ in the MAB Binary benchmark and $n=10000$ in the CB Binary Linear benchmark. Unless stated otherwise, we use $K = 10$ for the MAB benchmark, we use $d=10$ and $K=10$ for the CB benchmark, and the data set is generated using a uniform behaviour policy. In Section \ref{sec:offline_method}, motivated by our observations, we evaluate a new offline PAC-Bayes bandit algorithm in the CB Classification benchmark.

\begin{figure*}
\centering
\includegraphics[width=0.33\textwidth]{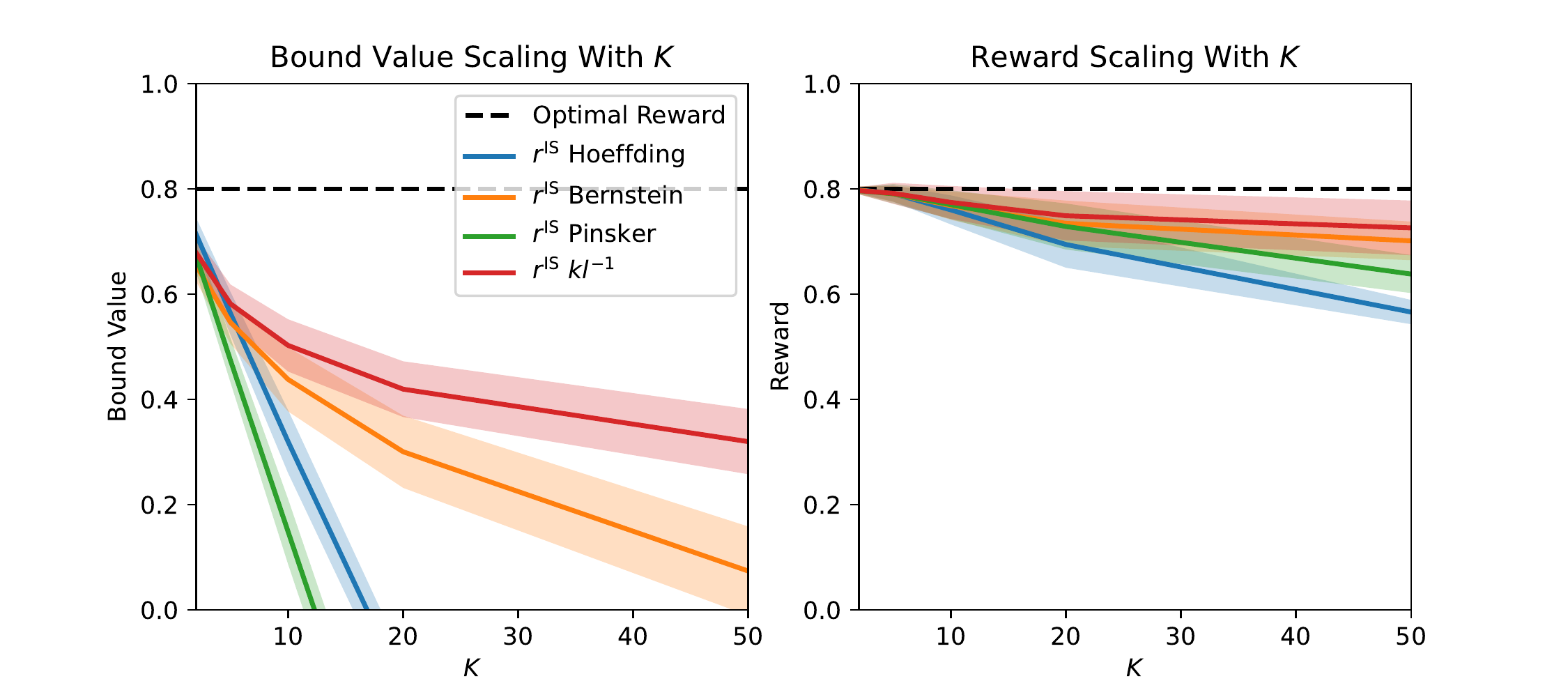}
\includegraphics[width=0.33\textwidth]{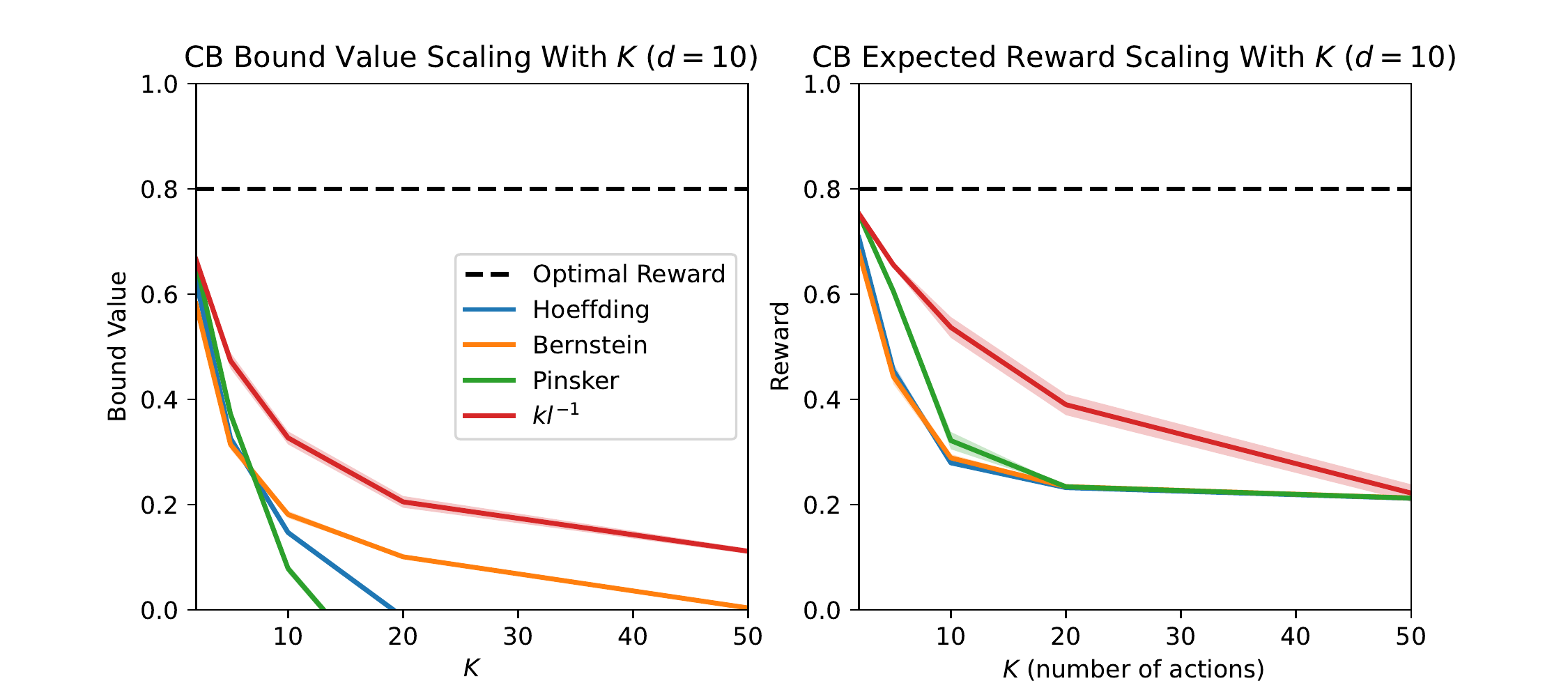}
\includegraphics[width=0.33\textwidth]{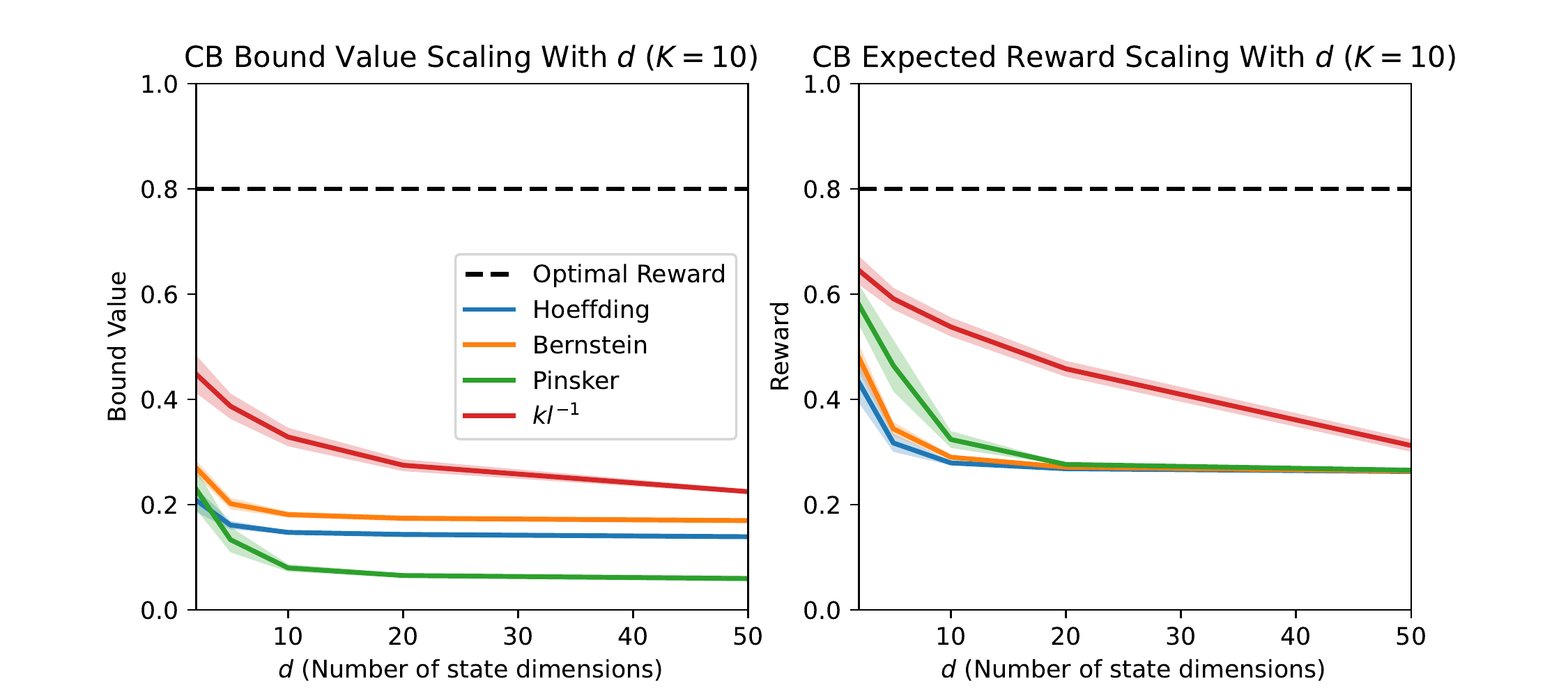}
\caption{Our comparison of the PAC-Bayes reward bounds. (Left) The bound value and expected reward for each bound in the MAB Binary benchmark. The number of actions $K$ varies from 2 to 50 along the $x$ axes. (Middle) The bound value and expected reward for each bound in the CB Binary Linear benchmark. The number of dimensions of the states $d$ is fixed at 10 and the number of actions $K$ varies from 2 to 50 along the $x$ axes. (Right) The bound value and expected reward for each bound in the CB Binary Linear benchmark. $d$ varies from 2 to 50 along the $x$ axes and $K$ is fixed at 10.}
\label{fig:bound_comparison}
\end{figure*}

\subsubsection{Insights About Different Bounds}
\label{sec:bound_comparison}

We first investigate which of the PAC-Bayes bounds available for the IS and CIS estimates is best. We varied the number of actions $K$ and the number of dimensions $d$ of the state vector to investigate how each of the bounds scales with $K$ and $d$. In the MAB benchmark, $K$ varied from $2$ to $50$. In the CB benchmark, we ran the experiment twice. First, $d$ was fixed at 10 and $K$ varied from $2$ to $50$. Then, $K$ was fixed at 10 and $d$ varied from $2$ to $50$.

\begin{experimentmessage}
The PAC-Bayes $kl^{-1}$ bound gives the greatest lower bound on the expected reward. The posterior learned by maximising the $kl^{-1}$ bound achieves the highest expected reward.
\end{experimentmessage}

In the left and middle pairs of plots in Figure \ref{fig:bound_comparison}, we observe that increasing the number of actions causes the bound values to decay rapidly. As one would expect, due to its improved dependence on $\epsilon_n$, the Bernstein bound decays at a much slower rate than the Hoeffding-Azuma and Pinsker bounds. The $kl^{-1}$ bound scales up the best to large numbers of actions. As seen in the rightmost pair of plots in Figure \ref{fig:bound_comparison}, increasing the number of dimensions of the states appeared to have less effect on the bound values. The PAC-Bayes $kl^{-1}$ bound consistently gave the greatest bound values and yielded posteriors with the greatest reward.

\subsubsection{Insights About Clipping}
\label{sec:estimator_comparison}

In this section, we compare the PAC-Bayes $kl^{-1}$ bound for the IS and CIS estimates. Since clipping affects the importance weights, which are determined by the behaviour policy, we test the bounds with several behaviour policies to try and identify if and when clipping is helpful. First, we use a uniform behaviour policy. Next, we use an informative behaviour policy. In the MAB benchmark the informative policy was an $\epsilon$-smoothed Gibbs policy:
\begin{equation*}
b^{\mathrm{inf}}(a) \propto e^{10 R(a)}, \quad \tilde{b}^{\mathrm{inf}}(a) = (1 - K\epsilon)b^{\mathrm{inf}}(a) + \epsilon.
\end{equation*}

In the CB benchmark, the informative behaviour policy was another $\epsilon$-smoothed policy:
\begin{equation*}
b^{\mathrm{inf}}(a|s) \propto e^{\inner{s}{\theta^{*}}_a}, \quad \tilde{b}^{\mathrm{inf}}(a|s) = (1 - K\epsilon)b^{\mathrm{inf}}(a|s) + \epsilon.
\end{equation*}

$\theta^{*}$ is the weight matrix of the unknown linear classifier that generates the rewards. Finally, we use a randomly generated behaviour policy. In the MAB benchmark, the random behaviour policy was an $\epsilon$-smoothed probability vector drawn randomly from a symmetric Dirichlet distribution with $\alpha = 1$. In the CB benchmark, the behaviour policy was an $\epsilon$-smoothed linear softmax policy with a weight matrix $\theta$ drawn randomly from a standard Gaussian distribution. For both the informative and random behaviour policies, we used $\epsilon = 0.01$.

\begin{experimentmessage}
Using the CIS estimate instead of the IS estimate can improve both the bound value and the expected reward of the learned posterior when the behaviour policy is non-uniform.
\end{experimentmessage}

In Figure \ref{fig:clipping} we see that the PAC-Bayes $kl^{-1}$ bound for the IS estimate yields a lower bound value and lower expected reward with the informative behaviour policy than with the uniform behaviour policy. When the behaviour policy is uniform, the bound for the CIS estimate is no better than the bound for the IS estimate. However, when the behaviour policy is non-uniform, and particularly when it is informative, the $kl^{-1}$ bound for the CIS estimate can yield a greater bound value and greater expected reward.

\subsubsection{Insights About Choosing the Prior}
\label{sec:prior_comparison}

In this section, we evaluate the presented methods for choosing the prior by using them to set the prior in a PAC-Bayes bound for the IS estimate. For the prior selection methods that work with any PAC-Bayes bound, we use them with the $kl^{-1}$ bound, since this appeared to be the best in our earlier experiments.

In the MAB benchmark, the bounds we compared are: the $kl^{-1}$ bound with a uniform prior (Theorem \ref{thm:pac_bayes_kl}), the $kl^{-1}$ bound with a prior learned using a subset of the data (Theorem \ref{thm:pac_bayes_kl_subset_bound}), the $kl^{-1}$ bound with a differentially-private prior (Theorem \ref{thm:pac_bayes_kl_dp_bound}), the $kl^{-1}$ Lever bound (Theorem \ref{thm:pac_bayes_kl_lever_bound}), Oneto et al.'s distribution stability bound (Theorem \ref{thm:pac_bayes_dist_stab}), the localised PAC-Bayes Bernstein bound (Theorem \ref{thm:local_pac_bayes_bernstein}) and the PAC-Bayes Hoeffding-Azuma Empirical Gibbs bound (Theorem \ref{thm:pac_bayes_ha_emp_gibbs}). We do not evaluate the $kl$ hypothesis sensitivity bound (Theorem \ref{thm:pac_bayes_kl_hypo_sens_bound}) or the $kl$ London and Sandler bound (Theorem \ref{thm:pac_bayes_kl_london_bound}) because we are not aware of a suitable learning algorithm with known hypothesis sensitivity coefficients for the first or a suitable convex and $\beta$-Lipschitz function $L$ for the second. We compare the same bounds in the CB benchmark, but without the localised PAC-Bayes Bernstein bound or the PAC-Bayes Hoeffding-Azuma Empirical Gibbs bound. This is because we cannot calculate $D_{\mathrm{KL}}(\rho||\mu_{\beta r^{\mathrm{IS}}})$ for the linear softmax policy class. In Appendix \ref{sec:prior_implementation}, we describe how each of these bounds implemented.

\begin{figure*}
\centering
\includegraphics[width=0.33\textwidth]{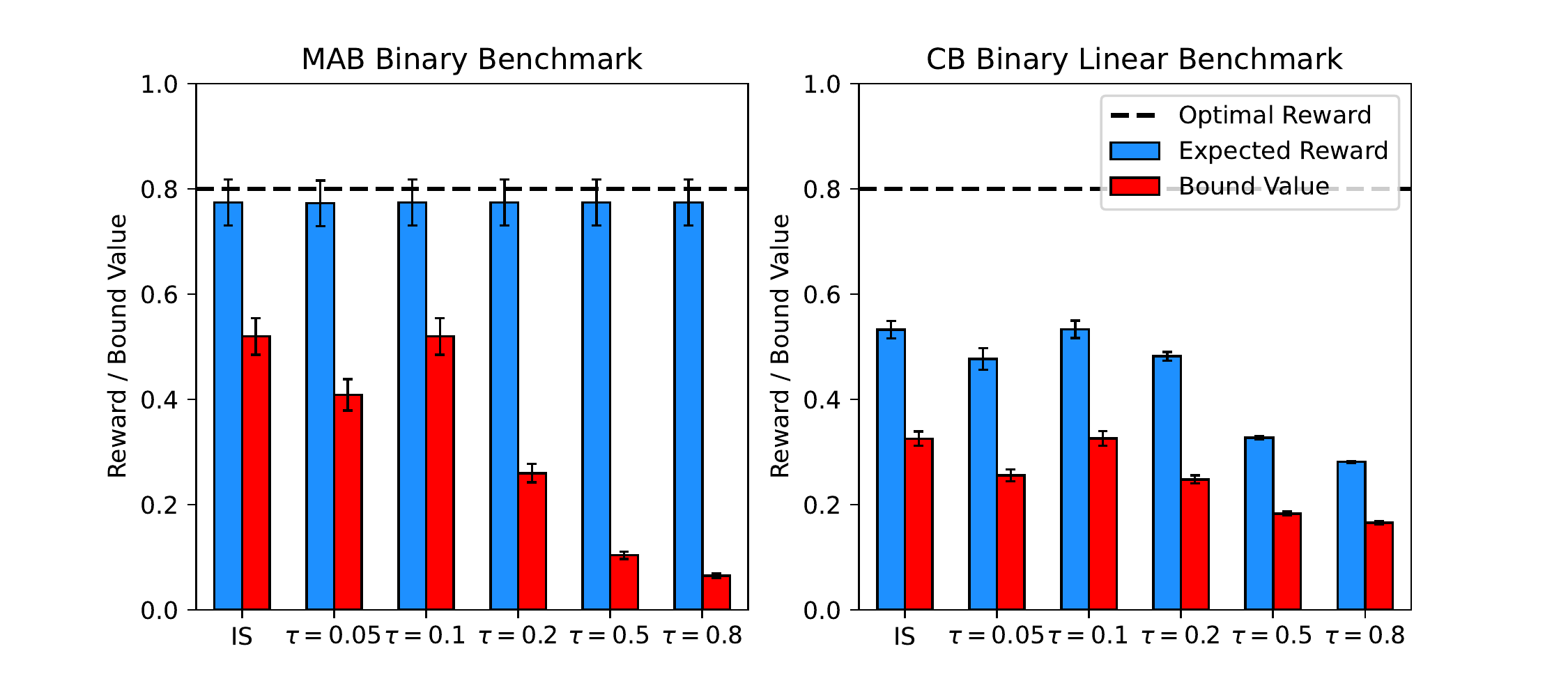}
\includegraphics[width=0.33\textwidth]{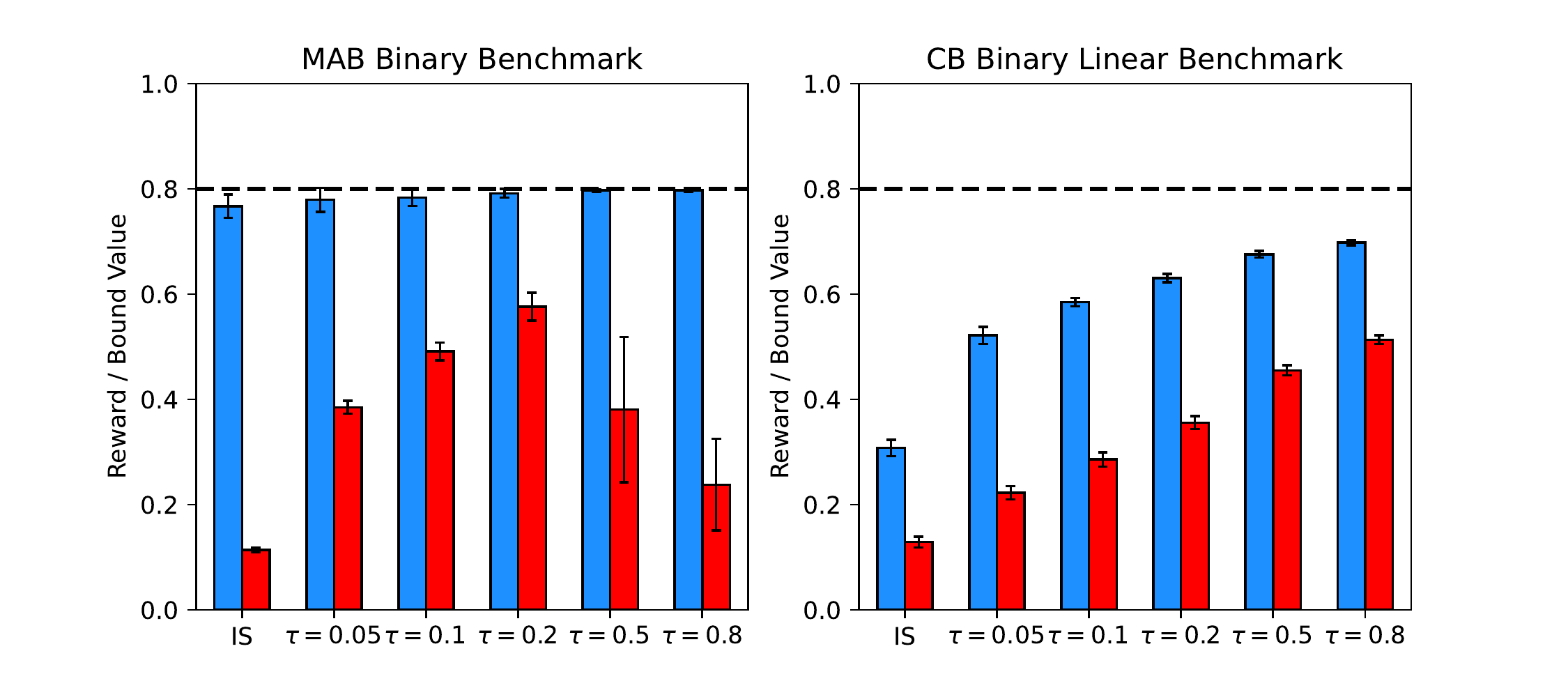}
\includegraphics[width=0.33\textwidth]{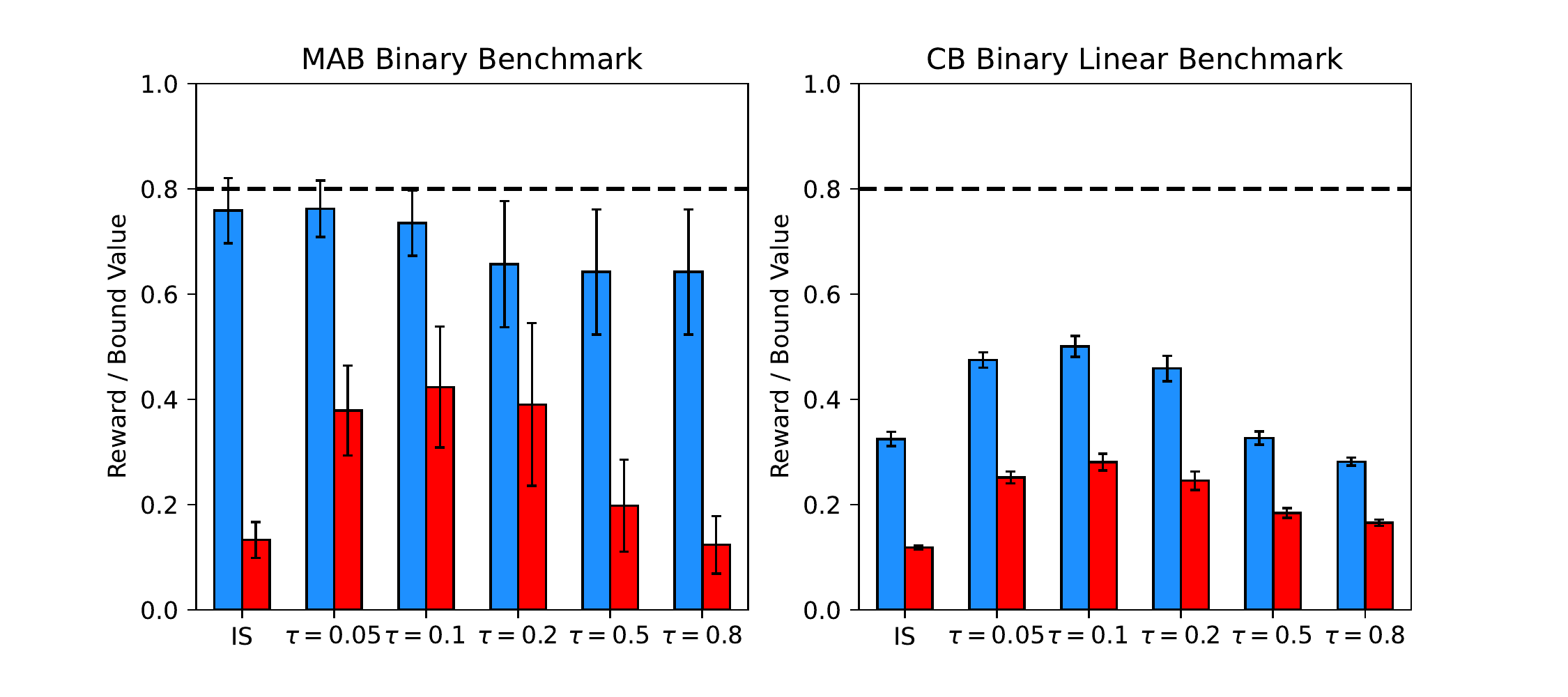}
\caption{Our comparison of the PAC-Bayes $kl^{-1}$ reward bound with the IS estimate and the CIS estimate with several $\tau$'s. (Left) The expected reward (blue) and bound value (red) for each estimate in the MAB and CB benchmarks with a uniform behaviour policy. (Middle) The expected reward and bound value for each estimate in the MAB and CB benchmarks with an informative behaviour policy. (Right) The expected reward and bound value for each estimate in the MAB and CB benchmarks with a random, non-uniform behaviour policy.}
\label{fig:clipping}
\end{figure*}

\begin{experimentmessage}
A data-dependent prior learned using a subset of the data appears to be the best way to set the prior.
\end{experimentmessage}

Figure \ref{fig:priors_and_classification} (left) shows the expected reward and bound values for the bounds we compared. In the MAB benchmark, none of the bounds with data-dependent or distribution-dependent priors achieved higher reward or higher bound values than the $kl^{-1}$ bound with a uniform prior. In this problem, and with a uniform prior, $D_{\mathrm{KL}}(\rho||\mu) \leq \mathrm{ln}(K)$. Since this is already small (relative to $\mathrm{ln}(1/\delta)$), it not so surprising that the more sophisticated priors did not help. The localised Bernstein bound and the Hoeffding-Azuma Empirical Gibbs bound were both greatest when $\beta = 0$. With this choice of $\beta$, the empirical Gibbs prior $\mu_{\beta r^{\mathrm{IS}}}$ is a uniform prior. The distribution stability bound and the Hoeffding-Azuma Empirical Gibbs bound were both vacuous, with average values of $-3.613$ and $-5.608$ respectively.

In the CB benchmark, the $kl^{-1}$ bound with a prior learned from a subset of the data had a greater expected reward and bound value compared to the $kl^{-1}$ bound with a standard Gaussian prior. With the $\eta$-differentially private prior, we found that as soon as $\eta$ is large enough that the prior is informative, the $\eta$-dependent penalty terms become large enough to offset this benefit. The bound value was greatest when $\eta$ was very close to 0, and we observe that the expected reward and bound value for the $\eta$-DP prior and the uninformative prior are almost the same. With the Lever and distribution stability bounds for the Gibbs posterior $\rho_{\gamma}$, we found that when $\gamma$ is large enough for $\rho_{\gamma}$ to have large empirical reward, the bounds on $D_{\mathrm{KL}}(\rho_{\gamma}||\mu_{\gamma})$ are large enough to offset this. Consequently, these two bounds were greatest when $\gamma$ was small, resulting in underfitting, low expected reward and low bound values. The average bound value for the distribution stability bound was -3.807. Our results suggest that using a subset of the data to learn a prior appears to be the best way to set the prior, at least for large enough policy classes.

\subsubsection{Insights About Choosing Bound Parameters}
\label{sec:parameter_comparison}

In Appendix \ref{sec:app_param_exp}, we compare the methods presented in Appendix \ref{sec:app_bound_params} for approximately optimising the parameters of PAC-Bayes bounds. We used each method to set the $\lambda$ parameter of the $r^{\mathrm{IS}}$ PAC-Bayes Bernstein bound. Surprisingly, we found that setting $\lambda$ to a fixed, data-independent value resulted in better bound values than using either sample splitting (Theorem \ref{thm:pac_bayes_bern_subset_bound}) or union bounds (Theorem \ref{thm:pac_bayes_bern_grid_bound}). We explore this result further in \ref{sec:app_param_exp} and find that whenever $n$ is large enough for the Bernstein bound to be non-vacuous for some value of $\lambda$, the minimum of the Bernstein bound with respect to $\lambda$ is flat, which means that a reasonable data-independent $\lambda$ is almost as good as the optimum value.

\subsubsection{A Method For Offline Bandits}
\label{sec:offline_method}

Using the insights gained from the previous experiments, we propose a method for offline contextual bandit problems and we test it in the Contextual Bandit Classification problem where the policy class is a set of neural networks.

For the first step of our method, we use the first half of the training data $D_{1:n/2}$ to learn a diagonal Gaussian prior over the neural network weights $\theta$ by maximising
\begin{equation}
\mathop{\mathbb{E}}_{\theta \sim \mu_{D}}\left[r^{\mathrm{IS}}(\pi_{\theta}, D_{1:n/2})\right] - \beta D_{\mathrm{KL}}(\mu_{D}||\mu),\label{eqn:prior_obj}
\end{equation}

with respect to $\mu_{D}$. $\pi_{\theta}$ is a neural network with weights $\theta$. $\mu$ is a standard Gaussian distribution. To choose $\beta$, we split $D_{1:n/2}$ into a training set $D_{\mathrm{tr}}$ and a validation set $D_{\mathrm{val}}$. We learn diagonal Gaussian priors by maximising Equation \ref{eqn:prior_obj} for $\beta \in \{10^{-k}|k \in \{1, \dots, 6\}\}$. We choose the value of $\beta$ where the resulting prior $\mu_{D}$ maximises $\mathbb{E}_{\theta \sim \mu_{D}}[r^{\mathrm{IS}}(\pi_{\theta}, D_{\mathrm{val}})]$. Next, we learn the clipping parameter $\tau$. With $\mu_{D}$ fixed, and using the first half of the training data, we optimise the following objective with respect to $\tau$:
\begin{equation*}
\frac{1}{\tau}kl^{-1}\left(\tau r^{\mathrm{CIS}}(\mu_{D}, D_{1:n/2}), \frac{\mathrm{ln}(\sqrt{2n}/\delta)}{n/2}\right).
\end{equation*}

This approximates the value of $\tau$ that would be optimal if we were to use the posterior $\rho = \mu_{D}$. Now that we have our data-dependent prior $\mu_{D}$ and data-dependent $\tau$, we learn the posterior by maximising the $kl^{-1}$ bound with respect to $\rho$ and using the second half of the training data.
\begin{equation}
\frac{1}{\tau}kl^{-1}\left(\tau r^{\mathrm{CIS}}(\rho, D_{n/2+1:n}), \frac{D_{\mathrm{KL}}(\rho||\mu_{D})) + \mathrm{ln}(\sqrt{2n}/\delta)}{n/2}\right).\label{eqn:proposed_obj}
\end{equation}

Finally, we evaluate the bound (Equation \ref{eqn:proposed_obj}) at the learned posterior, using the second half of the training data and the data-dependent $\mu_{D}$ and $\tau$.

\begin{figure*}
\centering
\includegraphics[width=0.49\textwidth]{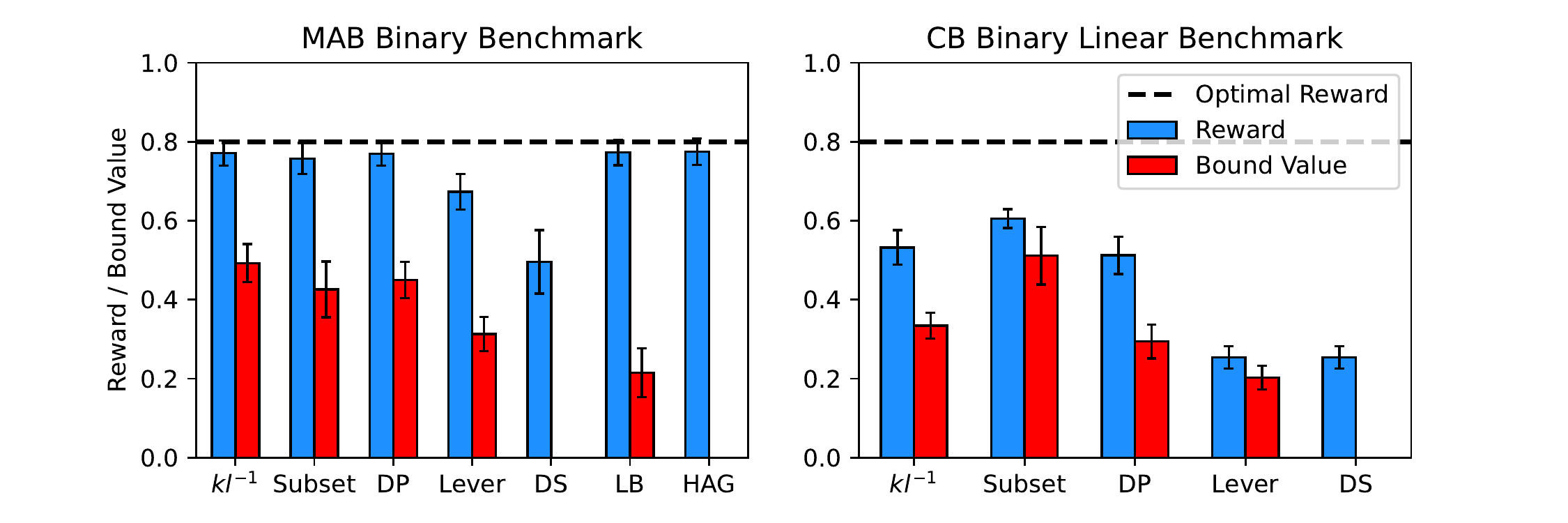}
\includegraphics[width=0.49\textwidth]{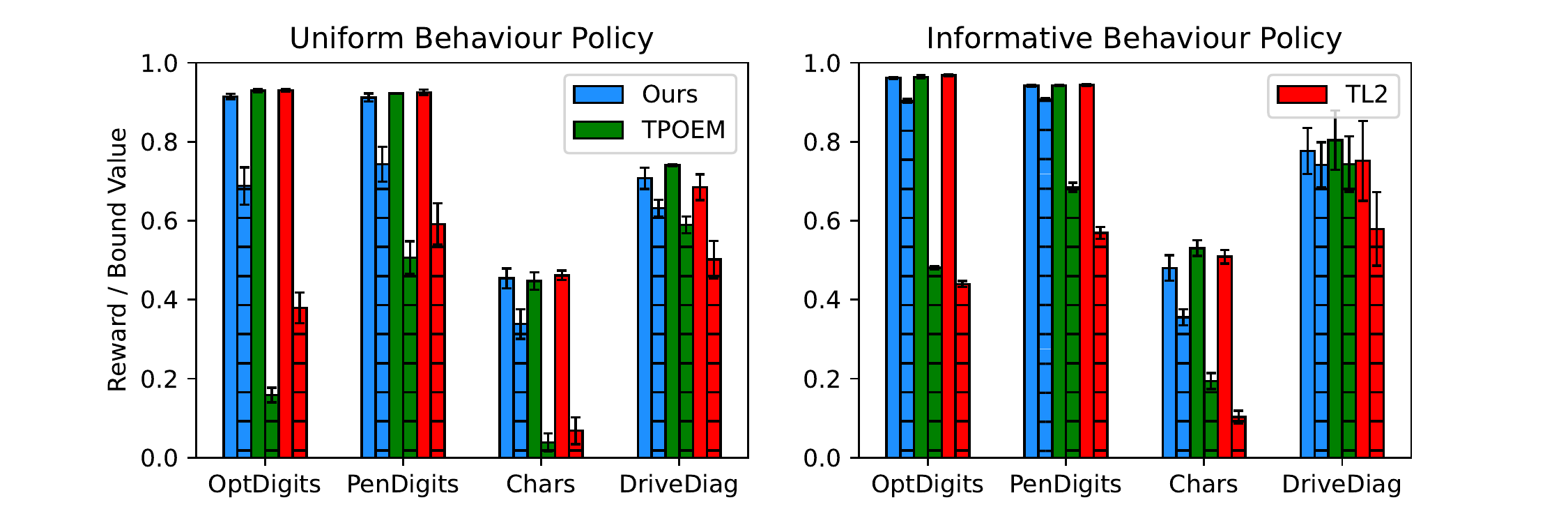}
\caption{(Left) The expected reward (blue) and bound value (red) for each bound in our comparison of methods for choosing the prior. DP is the differentially private prior, DS is the distribution stability bound, LB is the localised Bernstein bound and HAEG is the Hoeffding-Azuma Empirical Gibbs bound. (Right) The expected reward (solid bars) and bound value (striped bars) for our proposed offline bandit algorithm (blue), the TPOEM baseline (green) and the TL2 baseline (red) in the CB classification benchmark.}
\label{fig:priors_and_classification}
\end{figure*}

We compare the expected reward and bound values of our method against two baselines. The first baseline is inspired by the POEM algorithm and PAC bound by Swaminathan and Joachims \cite{swaminathan2015batch}. The POEM PAC bound uses covering numbers to measure the complexity of the policy class. Based on the covering number bounds for neural networks by Anthony and Bartlett \cite{anthony2009neural}, we expect that the original POEM PAC bound is vacuous for the CB Classification benchmark with our neural network policy class. Therefore, for a tougher comparison, we compare our proposed method to a test set bound inspired by the orignal POEM bound. We call this TestPOEM (TPOEM). Like the original POEM algorithm, it uses the sample variance of the CIS estimate to regularise the policy selection. We also compare against a second baseline that is similar to TPOEM, except it uses the L2 norm of the neural network weights to regularise the policy selection.

For TPOEM and TL2, we use $\tau = 1/K$ since in Section \ref{sec:estimator_comparison} we saw that this was the best choice for uniform behaviour policies and a good choice for the non-uniform behaviour policies.

\begin{experimentmessage}
Our proposed PAC-Bayes offline contextual bandit algorithm can learn neural network policies that achieve competitive expected reward and can provide tighter reward bounds than TPOEM and TL2.
\end{experimentmessage}

We test our proposed method, TPOEM and TL2 in the CB Classification benchmark, first with a data set drawn using a uniform behaviour policy and then with a data set drawn using a more informative behaviour policy. For each CB Classification problem, we train a neural network classifier using 10\% of the original classification data set. The $\epsilon$-smoothed class probabilities of these classifiers, with $\epsilon=0.01$, are the action probabilities of the informative behaviour policies. Figure \ref{fig:priors_and_classification} (right) shows the expected reward and bound values for the three methods. When the behaviour policy was uniform, our method (blue) learned policies with competitive expected reward while providing greater bound values than TPOEM (green) and TL2 (red). When the behaviour policy was informative, our method once again learned policies with competitive expected reward while providing greater bound values on all except the drive diagnosis problem, where the bound value for our method and TPOEM were comparable. The bound value for our method on the PenDigits problem was remarkably tight: the expected reward was 0.94 and the bound value was 0.91.

\section{Conclusion}
\label{sec:conclusion}

We have surveyed and empirically evaluated the available PAC-Bayes reward and regret bounds for bandit problems. In this section, we discuss our findings and highlight some open problems.

\subsection{Findings}

The results of our offline bandit experiments suggest that PAC-Bayes bounds are a useful tool for designing offline bandit algorithms with performance guarantees. In Fig. \ref{fig:bound_comparison}, Fig. \ref{fig:clipping} and Fig. \ref{fig:priors_and_classification} (left), we see that the choice of bound, the choice of estimator, and the choice of the prior can each have a large impact on both the performance of the learned policy and the tightness of the performance guarantee. In Fig. \ref{fig:priors_and_classification} (right), we see that a well-chosen bound, estimator and prior yields an offline bandit algorithm with competitive performance and very tight performance guarantees - even when the policy class is a set of neural networks. Similarly good performance guarantees with neural network-based policies would certainly not be possible with algorithms such as POEM \cite{swaminathan2015batch}, which measure the complexity of the policy class with covering numbers.

Our survey yields a less positive picture for existing online bandit algorithms. The cumulative regret bounds presented in Sec. \ref{sec:reg_bounds} had sub-optimal growth rates in $n$ and the algorithms motivated by these bounds performed poorly compared to EXP3 and UCB1. However, we believe that it would be premature to dismiss PAC-Bayes as a tool for designing online bandit algorithms with cumulative regret bounds. Rather, we believe that these less encouraging findings are indicative of PAC-Bayesian bandit algorithms being a topic that deserves further investigation. In Sec. \ref{sec:improved_regret_bounds} and Sec. \ref{sec:beyond_policy_search}, we describe several topics for future work that may lead to PAC-Bayesian online bandit algorithms with improved cumulative regret bounds.

\subsection{Future Research Questions}

\subsubsection{Tighter PAC-Bayes bounds for "better" estimators}

It is known that the WIS estimate often achieves lower mean squared error than the IS estimate \cite{hesterberg1995weighted}. However, the Efron-Stein PAC-Bayes reward bound for the WIS estimate was looser than some of the reward bounds that used the IS estimate (see Fig. \ref{fig:efron_stein}). Whether improved PAC-Bayes bounds can be derived for the WIS estimate may be a key question to answer. In addition, it may be worthwhile to investigate PAC-Bayes bounds for other improved reward estimates, such as the doubly robust estimate \cite{dudik2014doubly}.

\subsubsection{Improved cumulative regret bounds}
\label{sec:improved_regret_bounds}

The PAC-Bayes Bernstein cumulative regret bound from Thm. \ref{thm:is_bernstein_regret} has a sub-optimal growth rate of $\mathcal{O}(n^{2/3}K^{1/3})$ (ignoring log terms) because it uses a loose upper bound on the variance of the IS estimate. In a follow-up paper, Seldin et al. \cite{seldin2013evaluation} used a more sophisticated bound on the variance of the IS estimate to prove a high probability regret bound of order $\mathcal{O}(\sqrt{nK})$ (ignoring log terms) for EXP3. Investigating whether this more-sophisticated variance bound is compatible with PAC-Bayes analysis is one path towards PAC-Bayesian bandit algorithms with improved cumulative regret bounds.

\subsubsection{Beyond policy search}
\label{sec:beyond_policy_search}

Following the literature on PAC-Bayesian bandits, we have focused exclusively on policy search methods, which directly learn a policy from data. However, PAC-Bayes bounds are compatible with other approaches to bandits. We briefly describe two different kinds of bandit algorithms and how PAC-Bayes bounds might be incorporated.

Broadly speaking, oracle-based bandit algorithms, such as Epoch-Greedy \cite{langford2007epoch}, ILOVETOCONBANDIT \cite{agarwal2014taming} and SquareCB \cite{foster2020beyond}, reduce bandit problems to supervised learning problems, such as predicting the expected reward of each action. For example, SquareCB is a meta-algorithm that turns any online regression algorithm into an online contextual bandit algorithm. In addition, if the online regression algorithm has a regret bound for online regression with an optimal growth rate, then the resulting online contextual bandit algorithm enjoys a cumulative regret bound with an optimal growth rate. This is an appealing approach for designing PAC-Bayesian bandit algorithms because it allows us to utilise PAC-Bayesian supervised learning algorithms, which are plentiful. For instance, there are PAC-Bayesian algorithms for online regression problems (e.g. \cite{gerch2013sparsity, haddouche2022online}) that are compatible with SquareCB.

Confidence bounds are a key ingredient of online bandit algorithms that follow the optimism in the face of uncertainty principle (e.g.\cite{auer2002finite}, \cite{dani2008linear}) and offline bandit algorithms that follow the pessimism in the face of uncertainty principle (e.g. \cite{rashidinejad2021bridging}). Upper/lower confidence bounds are estimates of the expected reward for each action that, with high probability, are guaranteed to be above/below the expected reward. In principle, PAC-Bayes bounds could be used to construct confidence bounds suitable for bandits, though we are not aware of any in the literature. We believe that investigation of PAC-Bayesian confidence bounds, as well as bandit algorithms that use them, is a fruitful direction for future work.

%\ifCLASSOPTIONcompsoc
%  % The Computer Society usually uses the plural form
%\else
%  % regular IEEE prefers the singular form
%\fi
%
%
%The authors will probably like to thank the anonymous reviewers. Their suggestions will most likely help to improve the paper.

\ifCLASSOPTIONcaptionsoff
  \newpage
\fi

\pagebreak

\appendices

\section{Additional Material}

\subsection{General-Purpose PAC-Bayes Bounds for Martingales}
\label{app:general_martingale_bounds}

In this section, we give an overview of general-purpose PAC-Bayes bounds for martingales. Assume we have a collection of martingales indexed by a set $\Pi$, such that for every $\pi \in \Pi$, $(M_n(\pi)| n \in \mathbb{N})$ is a martingale with respect to another sequence of random variables $(X_n|n \in \mathbb{N})$. Let $(Z_n(\pi)| n \in \mathbb{N})$ be the martingale difference sequence associated with $(M_n(\pi)| n \in \mathbb{N})$, i.e. $M_n(\pi) = \sum_{i=1}^{n}Z_i(\pi)$. Define the predictable quadratic variation and the total quadratic variation of $(M_n(\pi)| n \in \mathbb{N})$ respectively as
\begin{align*}
\langle V \rangle_n(\pi) &= \sum_{i=1}^{n}\mathbb{E}\left[Z_i(\pi)^2|X_1, \dots, X_{n-1}\right],\\
[V]_n(\pi) &= \sum_{i=1}^{n}Z_i(\pi)^2.
\end{align*}

For a distribution $\rho$, let $M_n(\rho) = \mathbb{E}_{\pi \sim \rho}[M_n(\pi)]$, $\langle V \rangle_n(\rho) = \mathbb{E}_{\pi \sim \rho}[\langle V \rangle_n(\pi)]$, and $[V]_n(\rho) = \mathbb{E}_{\pi \sim \rho}[[V]_n(\pi)]$. Seldin et al. \cite{seldin2012mart} proved a generic Hoeffding-Azuma inequality for martingales with bounded differences.

\begin{theorem}[PAC-Bayes Hoeffding-Azuma Inequality \cite{seldin2012mart}]
Suppose the martingale difference sequence $(Z_n(\pi)| n \in \mathbb{N})$ satisfies $\mathbb{P}(Z_i(\pi) \in [a_i, b_i]) = 1$ for all $i = 1, \dots, n$ and all $\pi \in \Pi$. For any $\lambda > 0$, any $\delta \in (0, 1)$ and any probability distribution $\mu \in \mathcal{P}(\Pi)$, with probability at least $1 - \delta$, for all distributions $\rho \in \mathcal{P}(\Pi)$ simultaneously:
\begin{equation*}
M(\rho) \leq \frac{\lambda}{8}\sum_{i=1}^{n}(b_i - a_i)^2 + \frac{D_{\mathrm{KL}}(\rho||\mu) + \mathrm{ln}(1/\delta)}{\lambda}.
\end{equation*}
\label{thm:pb_hoeffding}
\end{theorem}
\vspace{-0.4cm}

Seldin et al. \cite{seldin2012mart} also proved a generic PAC-Bayes $kl$ bound for martingale-like sequences with bounded differences.

\begin{theorem}[PAC-Bayes $kl$ Inequality \cite{seldin2012mart}]
Suppose $(Z_n(\pi)| n \in \mathbb{N})$ satisfies $\mathbb{P}(Z_i(\pi) \in [0, 1]) = 1$ and $\mathbb{E}[Z_i(\pi)|X_1, \dots, X_{i-1}] = b(\pi)$ for all $i = 1, \dots, n$ and all $\pi \in \Pi$. Let $b(\rho) = \mathbb{E}_{\pi \sim \rho}[b(\pi)]$. For any $\delta \in (0, 1)$ and any probability distribution $\mu \in \mathcal{P}(\Pi)$, with probability at least $1 - \delta$, for all distributions $\rho \in \mathcal{P}(\Pi)$ simultaneously:
\begin{equation*}
kl\infdivx*{\frac{1}{n}M_n(\rho)}{b(\rho)} \leq \frac{D_{\mathrm{KL}}(\rho||\mu) + \mathrm{ln}(2\sqrt{n}/\delta)}{n}.
\end{equation*}
\label{thm:pac_bayes_kl_general}
\end{theorem}
\vspace{-0.4cm}

Seldin et al. \cite{seldin2012mart} proved a generic PAC-Bayes Bernstein bound for martingales with bounded differences which depends on the predictable quadratic variation ($\langle V \rangle_n(\rho)$) of the martingale mixture $M_n(\rho)$.

\begin{theorem}[PAC-Bayes Bernstein Inequality \cite{seldin2012mart}]
Suppose the martingale difference sequence $(Z_n(\pi)| n \in \mathbb{N})$ satisfies $\mathbb{P}(Z_i(\pi) \in [-b, b]) = 1$ for all $i = 1, \dots, n$ and all $\pi \in \Pi$. For any $\lambda \in [0, 1/b]$, any $\delta \in (0, 1)$ and any probability distribution $\mu \in \mathcal{P}(\Pi)$, with probability at least $1 - \delta$, for all distributions $\rho \in \mathcal{P}(\Pi)$ simultaneously:
\begin{equation*}
M_n(\rho) \leq \lambda(e-2)\langle V \rangle_n(\rho) + \frac{D_{\mathrm{KL}}(\rho||\mu) + \mathrm{ln}(1/\delta)}{\lambda}.
\end{equation*}
\label{thm:pb_bernstein}
\end{theorem}
\vspace{-0.4cm}

Balsubramani \cite{balsubramani2015pac} proved a PAC-Bayes inequality for martingales with bounded differences that is similar to the PAC-Bayes Bernstein Inequality. Balsubramani's bound is tigher when $\langle V \rangle_n(\rho)$ is a lot smaller than $n$. However, it only holds with high probability for sufficiently large $n$.

\begin{theorem}[PAC-Bayes Bernstein Law of Iterated Logarithm (LIL) Inequality \cite{balsubramani2015pac}]
Suppose the martingale difference sequence $(Z_n(\pi)| n \in \mathbb{N})$ satisfies $\mathbb{P}(Z_i(\pi) \in [-e^2, e^2]) = 1$ for all $i = 1, \dots, n$ and all $\pi \in \Pi$. For any $\delta \in (0, 1)$ and any probability distribution $\mu \in \mathcal{P}(\Pi)$, with probability at least $1 - \delta$, the following is true for all $\rho \in \mathcal{P}(\Pi)$ simultaneously. Let $n_0(\rho) = \min\left\{k: 2(e-2)\langle V \rangle_k(\rho) \geq \frac{2}{\lambda_0^2}(\mathrm{ln}(4/\delta) + D_{\mathrm{KL}}(\rho||\mu))\right\}$, where $\lambda_0 = \frac{1}{e^2(1+1/\sqrt{3})}$. For all $n \geq n_0(\rho)$ simultaneously
\begin{align*}
|M_n(\rho)| \leq \frac{2(e-2)}{e^2(1+1/\sqrt{3})}\langle V \rangle_n(\rho),
\end{align*}
and
\begin{align*}
&|M_n(\rho)| \leq \sqrt{6(e-2)\langle V \rangle_n(\rho)}\\
&\times\sqrt{\mathrm{ln}\mathrm{ln}\left(\frac{3(e-2)\langle V \rangle_n(\rho)}{|M_n(\rho)|}\right) + \mathrm{ln}\left(\frac{2}{\delta}\right) + D_{\mathrm{KL}}(\rho||\mu)}.
\end{align*}
\label{thm:pb_bernstein_lil}
\end{theorem}
\vspace{-0.4cm}

The right-hand-side of this bound contains $M_n(\rho)$. However, either $|M_n(\rho)| \leq 1$ or the iterated logarithm term is upper bounded by $\mathrm{ln}\mathrm{ln}\left(3(e-2)V_n(\rho)\right)$. Finally, Wang et al. \cite{wang2015pac} proved a PAC-Bayes bound for locally square-integrable martingales, which can have unbounded differences.

\begin{theorem}[PAC-Bayes Inequality for Locally Square Integrable Martingales \cite{wang2015pac}]
For any $\lambda \geq 0$, any $\delta \in (0, 1)$ and any probability distribution $\mu \in \mathcal{P}(\Pi)$, with probability at least $1 - \delta$, for all distributions $\rho \in \mathcal{P}(\Pi)$ simultaneously:
\begin{equation*}
M_n(\rho) \leq \frac{\lambda}{2}\left(\frac{[V]_n(\rho)}{3} + \frac{2\langle V \rangle_n(\rho)}{3}\right) + \frac{D_{\mathrm{KL}}(\rho||\mu) + \mathrm{ln}(1/\delta)}{\lambda}.
\end{equation*}
\label{thm:pb_unbound}
\end{theorem}
\vspace{-0.4cm}

Haddouche and Guedj \cite{haddouche2022supermartingales} proved a time-uniform version of the PAC-Bayes bound in Thm. \ref{thm:pb_unbound}, i.e. with probability at least $1 - \delta$, the inequality holds for \emph{all} $n \geq 1$ simultaneously. Chugg et al. \cite{chugg2023unified} proposed a unified framework for deriving time-uniform PAC-Bayes bounds and used it to prove time-uniform versions of Thm. \ref{thm:pb_hoeffding}, Thm. \ref{thm:pac_bayes_kl_general} and Thm. \ref{thm:pb_bernstein}.

\subsection{PAC-Bayes Reward Bounds With the Weighted Importance Sampling Estimate}
\label{sec:app_wis_estimate}

We present PAC-Bayes bounds for the weighted (or self-normalised) importance sampling (WIS) estimator. The PAC-Bayes bounds presented in this section are only valid when the data are i.i.d.; for example, when the data set is drawn from a fixed behaviour policy. For the rest of Section \ref{sec:app_wis_estimate}, we assume that this is the case. On the bright side, the bounds in this section do not require the importance weights to be bounded or clipped. For MAB problems, the WIS estimator can be defined as:
\begin{equation*}
r^{\mathrm{WIS}}(\pi, D_n) = \frac{\sum_{i=1}^{n}\frac{\pi(a_i)}{b(a_i)}r_i}{\sum_{i=1}^{n}\frac{\pi(a_i)}{b(a_i)}}.
\end{equation*}

The WIS estimate has some pleasing properties. Firstly, when the rewards are bounded between 0 and 1, it always takes values in the range $[0, 1]$, even when the importance weights $\pi(a)/b(a)$ are unbounded. Secondly, it is invariant to constant shifts in the importance weights. Therefore, we only need to know the unnormalised probability mass/density functions of the policies $\pi$ and $b$.

The WIS estimator is biased but consistent, meaning its bias decays to 0 as $n$ tends to infinity. Liu \cite{liu2001monte} shows that the bias decays to 0 with rate $\mathcal{O}(1/n)$, so we can expect it to be close to 0 as long as $n$ is reasonably large. One can obtain PAC-Bayes bounds on the difference between $r^{\mathrm{WIS}}(\rho, D_n)$ and $R(\rho)$ by upper bounding both terms in the following bias-concentration decomposition:
\begin{align}
R(\rho) - r^{\mathrm{WIS}}(\rho, D_n) &= \underbrace{R(\rho) - R^{\mathrm{WIS}}(\rho)}_\text{bias}\label{eqn:wis_bias_conc}\\
&+ \underbrace{R^{\mathrm{WIS}}(\rho) - r^{\mathrm{WIS}}(\rho, D_n)}_\text{concentration}.\nonumber
\end{align}

We are not aware of any empirical upper bounds on this bias term that don't require additional assumptions on the reward distribution $P_R$. Kuzborskij et al. \cite{kuzborskij2021confident} proved a bound on the bias term, although it only holds when the rewards are one-hot; there is always one action with reward 1 and all remaining actions have reward 0.

The concentration term can be bounded using PAC-Bayes bounds. Since the WIS estimate is not a sum of i.i.d. random variables or even the sum of a martingale difference sequence, we cannot use any of the previously seen PAC-Bayes bounds to bound the concentration term. Kuzborskij and Szepesv{\'a}ri \cite{kuzborskij2019efron} derived a very general Efron-Stein (ES) PAC-Bayes bound and showed that it can be used to upper bound the concentration term in Equation \ref{eqn:wis_bias_conc}. This bound contains the semi-empirical ES variance proxy of the WIS estimate. For any real-valued function $f(\pi, D_n)$, the corresponding semi-empirical ES variance proxy is defined as:
\begin{equation*}
V^{\mathrm{ES}}(\pi, D_n) = \sum_{i=1}^{n}\mathop{\mathbb{E}}_{D_n, D_n^{\prime}}\left[\left(f(\pi, D_n) - f(\pi, D_n^{(i)})\right)^2 \bigg| D_{i}\right].
\end{equation*}

$D_n^{\prime}$ is an independently sampled copy of $D_n$. $D_n^{(i)}$ is the data set $D_n$, except the $i$th element is replaced with the $i$th element of $D_n^{\prime}$. For example, in the MAB setting, $(a_i, r_i)$ is replaced by an independent copy $(a_i^{\prime}, r_i^{\prime})$. This variance proxy is semi-empirical since it depends on both the observed data and the distribution of the data. Kuzborskij and Szepesv{\'a}ri \cite{kuzborskij2019efron} derived a PAC-Bayes bound on the absolute difference between $f(\pi, D_n)$ and its expected value $F(\pi) = \mathbb{E}_{D_n}[f(\pi, D_n)]$. When $f = r^{\mathrm{WIS}}$, we obtain the following result.

\begin{theorem}[Efron-Stein PAC-Bayes Bound for $r^{\mathrm{WIS}}$ \cite{kuzborskij2019efron}]
If the data set $D_n$ is drawn from a single, fixed behaviour policy, then for any $y > 0$, any $\delta \in (0, 1)$ and any probability distribution $\mu \in \mathcal{P}(\Pi)$, with probability at least $1 - \delta$, for all distributions $\rho \in \mathcal{P}(\Pi)$ simultaneously:
\begin{align*}
&\left|r^{\mathrm{WIS}}(\rho, D_n) - R^{\mathrm{WIS}}(\rho)\right| \leq \sqrt{2\left(y + V^{\mathrm{ES}}(\rho, D_n)\right)}\\
&\times \sqrt{D_{\mathrm{KL}}(\rho||\mu) + \frac{1}{2}\mathrm{ln}\left(1 + V^{\mathrm{ES}}(\rho, D_n)/y\right) + \mathrm{ln}(1/\delta)}.
\end{align*}
\label{thm:efron_stein_pac_bayes}
\end{theorem}

Theorem \ref{thm:efron_stein_pac_bayes} is actually a slightly tighter version of the second inequality in Theorem 3 of \cite{kuzborskij2019efron} that holds under weaker assumptions. In the original bound, the factor of $1/2$ in front of $\mathrm{ln}\left(1 + V^{\mathrm{ES}}(\rho, D_n)/y\right)$ is replaced with $\mathrm{ln}(1/\delta)/2$, which is larger than $1/2$ when $\delta \leq e^{-1}$. The original bound of Kuzborskij and Szepesv{\'a}ri is only valid when $\delta \leq e^{-2}$, so the bound in Theorem \ref{thm:efron_stein_pac_bayes} is always slightly tighter. Moreover, this bounds holds simultaneously for all distributions $\rho$, whereas in the original bound of Kuzborskij and Szepesv{\'a}ri, $\rho$ must be given by a fixed probability kernel that maps any data set $D_n$ to a distribution $\rho_{D_n}$. We provide a proof of Theorem \ref{thm:efron_stein_pac_bayes} for general funtions $f(\pi, D_n)$ in Appendix \ref{sec:efron_stein_proof}.

The semi-empirical ES variance proxy for $r^{\mathrm{WIS}}(\pi, D_n)$ depends on the unknown reward distribution $P_{R}$, which means that $V^{\mathrm{ES}}(\rho, D_n)$ cannot be computed. However, Kuzborskij and Szepesv{\'a}ri \cite{kuzborskij2019efron} show that it can be upper bounded by a quantity that can be computed without knowledge of $P_{R}$.

\begin{lemma}[$r^{\mathrm{WIS}}$ ES Variance Proxy Upper Bound \cite{kuzborskij2019efron}]
For $f = r^{\mathrm{WIS}}$ and any $\pi \in \Pi$, we have that:
\begin{equation*}
V^{\mathrm{ES}}(\pi, D_n) \leq 2V^{\mathrm{WIS}}(\pi, D_n) = 2\sum_{i=1}^{n}\mathop{\mathbb{E}}_{D_n, D_n^{\prime}}\left[\tilde{w}_{\pi, i}^2 + \tilde{u}_{\pi, i}^2| D_{i}\right],
\end{equation*}

where
\begin{equation*}
\tilde{w}_{\pi, i} = \frac{\frac{\pi(a_i)}{b(a_i)}}{\sum_{j=1}^{n}\frac{\pi(a_j)}{b(a_j)}}, \qquad \tilde{u}_{\pi, i} = \frac{\frac{\pi(a_i^{\prime})}{b(a_i^{\prime})}}{\frac{\pi(a_i^{\prime})}{b(a_i^{\prime})} + \sum_{j \neq i}\frac{\pi(a_j)}{b(a_j)}}.
\end{equation*}
\label{lem:es_var_bound}
\end{lemma}

Though $V^{\mathrm{WIS}}(\pi, D_n)$ is still semi-empirical, it does not depend on the reward distribution $P_{R}$. Therefore, it can be estimated with arbitrary accuracy if $\pi$ and $b$ are known. We can combine the bias-concentration decomposition in Equation \ref{eqn:wis_bias_conc}, the Efron-Stein PAC-Bayes bound in Theorem \ref{thm:efron_stein_pac_bayes} and the bound on the ES variance proxy in Lemma \ref{lem:es_var_bound} to obtain the following PAC-Bayes bound on the expected reward, which holds with probability greater than $1 - \delta$ and for all $\rho \in \mathcal{P}(\Pi)$ simultaneously:
\begin{align}
R(\rho) &\geq r^{\mathrm{WIS}}(\rho, D_n) - \left|R^{\mathrm{WIS}}(\rho) - R(\rho)\right|\label{eqn:pac_bayes_wis}\\
&-\sqrt{2\left(y + 2V^{\mathrm{WIS}}(\rho, D_n)\right)}\nonumber\\
&\times \sqrt{D_{\mathrm{KL}}(\rho||\mu) + \frac{1}{2}\mathrm{ln}\left(1 + \frac{2V^{\mathrm{WIS}}(\rho, D_n)}{y}\right) + \mathrm{ln}(1/\delta)}.\nonumber
\end{align}

In order to use this bound, we would need to upper bound the bias term $\left|R^{\mathrm{WIS}}(\rho) - R(\rho)\right|$. Ingoring the bias term, the rate of this bound in $n$ depends on the values of $V^{\mathrm{WIS}}(\rho, D_n)$ and $y$. At one extreme, when all policies in the support of $\rho$ result in approximately equal importance weights for every action, we have $V^{\mathrm{WIS}}(\rho, D_n) = \mathcal{O}(1/n)$. At the other extreme, when policies in the support of $\rho$ result in one importance weight dominating all the others, we have $V^{\mathrm{WIS}}(\rho, D_n) = \mathcal{O}(1)$. Therefore, $V^{\mathrm{WIS}}(\rho, D_n) = \mathcal{O}(1/n^{\alpha})$ for some $\alpha \in [0, 1]$. If we choose $y = \mathcal{O}(1/n^{\alpha})$, then the bound in Equation \ref{eqn:pac_bayes_wis} has rate $\mathcal{O}(1/n^{\alpha/2})$.

Next, we discuss the ES PAC-Bayes bound for the WIS estimate in the contextual bandit setting. In the CB setting, the WIS estimate can be defined as:
\begin{equation*}
r^{\mathrm{WIS}}(\pi, D_n) = \frac{\sum_{i=1}^{n}\frac{\pi(a_i|s_i)}{b(a_i|s_i)}r_i}{\sum_{i=1}^{n}\frac{\pi(a_i|s_i)}{b(a_i|s_i)}}.
\end{equation*}

The ES PAC-Bayes bound in Theorem \ref{thm:efron_stein_pac_bayes} can still be used and one can derive an equivalent to the bound in Equation \ref{eqn:pac_bayes_wis}. However, the upper bound on the semi-empirical ES variance proxy $V^{\mathrm{WIS}}(\pi, D_n)$, as defined in Lemma \ref{lem:es_var_bound}, now depends on the unknown state distribution $P_S$. To rectify this, one can use an alternative bias-concentration decomposition suggested by Kuzborskij et al. \cite{kuzborskij2021confident}:
\begin{align*}
R(\rho) - r^{\mathrm{WIS}}(\rho, D_n) &= \underbrace{R(\rho) - R(\rho; s_{1:n})}_\text{concentration of contexts}\\
&+ \underbrace{R(\rho; s_{1:n}) - R^{\mathrm{WIS}}(\rho; s_{1:n})}_\text{bias}\\
&+ \underbrace{R^{\mathrm{WIS}}(\rho; s_{1:n}) - r^{\mathrm{WIS}}(\rho, D_n)}_\text{concentration},
\end{align*}

where
\begin{align*}
R(\rho; s_{1:n}) &= \mathop{\mathbb{E}}_{\pi \sim \rho}\left[\frac{1}{n}\sum_{i=1}^{n}\mathop{\mathbb{E}}_{a \sim \pi(\cdot|s_i), r \sim P_R(\cdot|s_i, a)}\left[r\right]\right],\\
R^{\mathrm{WIS}}(\rho; s_{1:n}) &= \mathop{\mathbb{E}}_{\pi \sim \rho}\left[\mathop{\mathbb{E}}_{D_n}\left[r^{\mathrm{WIS}}(\pi, D_n)|s_1, \dots, s_n\right]\right].
\end{align*}

The concentration of contexts term can be bounded by, for example, the PAC-Bayes Hoeffding-Azuma bound. The concentration term can be bounded using a conditional version of the Efron-Stein PAC-Bayes bound in Theorem \ref{thm:efron_stein_pac_bayes}, which holds with high probability over the sampling of $D_n$ given the observed states $s_1, \dots, s_n$. The upper bound $2V^{\mathrm{WIS}}(\pi, D_n)$ on the ES variance proxy, given the observed states, no longer depends on the state distribution, so it can be estimated with knowledge of only $\pi$ and $b$.

Finally, we note that replacing $r^{\mathrm{WIS}}$ in Theorem \ref{thm:efron_stein_pac_bayes} with $r^{\mathrm{IS}}$ or $r^{\mathrm{CIS}}$ would lead to new ES PAC-Bayes bounds for the IS or CIS estimates. However, this has not yet been explored in the literature.

\subsection{Regret Bounds for Contextual Bandits}
\label{sec:app_cb_regret}

We present a PAC-Bayes cumulative regret bound for contextual bandits by Seldin et al. \cite{seldin2011cb}. We consider the case when the set of actions is finite with $K$ elements ($\mathcal{A} = \{1, \dots, K\}$) and the set of states is finite with $N$ elements ($\mathcal{S} = \{1, \dots, N\}$). The policy class $\Pi$ is the set of all deterministic policies, which for this problem is the set of all functions from $\mathcal{S}$ to $\mathcal{A}$, of which there are $K^N$. For any distribution $\rho$ over $\Pi$, there is a corresponding stochastic policy $\rho(a|s)$, where:
\begin{equation*}
\rho(a|s) = \mathop{\mathbb{E}}_{\pi \sim \rho}\left[\mathbb{I}\{\pi(s) = a\}\right].
\end{equation*}

In this setting, the IS reward estimate for a single policy $\pi \in \Pi$ can be defined as:
\begin{equation}
r^{\mathrm{IS}}(\pi, D_n) = \frac{1}{n}\sum_{i=1}^{n}\frac{\mathbb{I}\{a_i = \pi(s_i)\}}{b(a_i|s_i, D_{i-1})}r_i.
\end{equation}

A lower bound $b(a|s, D_{i-1}) \geq \epsilon_n$ for all $D_{i-1}, s, a$ ensures that the importance weights are bounded by $1/\epsilon_n$. Let $n(s) = \sum_{i=1}^{n}\mathbb{I}\{s_i = s\}$ denote the number of times that state $s$ appears in the data set $D_n$. We define the IS reward estimate for a single state and action as:
\begin{equation}
r^{\mathrm{IS}}(s, a, D_n) = \frac{1}{n(s)}\sum_{i=1, \dots, n: s_i = s}\frac{\mathbb{I}\{a_i = a\}}{b(a_i|s_i, D_{i-1})}r_i.\label{eqn:is_est_sa}
\end{equation}

If $n(s) = 0$, then $r^{\mathrm{IS}}(s, a, D_n) = 0$. The expected regret for a policy $\pi$ can be defined as:
\begin{equation*}
\Delta(\pi) = R(\pi^*) - R(\pi),
\end{equation*}

where $\pi^*$ is a policy in $\Pi$ that maximises the expected reward. The IS regret estimate for a policy $\pi$ is defined as:
\begin{equation}
\Delta^{\mathrm{IS}}(\pi, D_n) = r^{\mathrm{IS}}(\pi^*, D_n) - r^{\mathrm{IS}}(\pi, D_n).
\end{equation}

Seldin et al. \cite{seldin2011cb} show that, as in the MAB setting, a martingale compatible with Bernstein's inequality can be constructed from the CB IS regret estimate. Moreover, the average variance of the CB IS regret estimate can also be bounded by $2/\epsilon_n$. Seldin et al. \cite{seldin2011cb} obtain a PAC-Bayes Bernstein bound on the difference between the expected regret and the IS regret estimate.

\begin{theorem}[CB PAC-Bayes Bernstein bound for $\Delta^{\mathrm{IS}}$ \cite{seldin2011cb}]
For any $\delta \in (0, 1]$ and any $c > 1$, simultaneously for all $\rho \in \mathcal{P}(\Pi)$ that satisfy
\begin{equation}
\frac{NI_{\rho}(S;A) + K(\mathrm{ln}(N) + \mathrm{ln}(K)) + \mathrm{ln}(m_n/\delta)}{2(e-2)n} \leq \frac{\epsilon_n}{c^2},\label{eqn:rho_cond}
\end{equation}

with probability at least $1 - \delta$:
\begin{align}
\Delta(\rho) &\leq \Delta^{\mathrm{IS}}(\rho, D_n) + (1+c)\sqrt{2(e-2)}\label{eqn:reg_bernstein_in}\\
&\times\sqrt{\frac{NI_{\rho}(S;A) + K(\mathrm{ln}(N) + \mathrm{ln}(K)) + \mathrm{ln}(m_n/\delta)}{n \epsilon_n}},\nonumber
\end{align}

where $m_n = \mathrm{ln}(\sqrt{(e-2)n/\mathrm{ln}(1/\delta)})/\mathrm{ln}(c)$, and for all $\rho$ that do not satisfy \ref{eqn:rho_cond}, with the same probability:
\begin{align}
\Delta(\rho) &\leq \Delta^{\mathrm{IS}}(\rho, D_n)\label{eqn:reg_bernstein_out}\\
&+ \frac{2\left(NI_{\rho}(S;A) + K(\mathrm{ln}(N) + \mathrm{ln}(K)) + \mathrm{ln}(m_n/\delta)\right)}{n \epsilon_n}.\nonumber
\end{align}

\label{thm:cb_is_reg_pac_bayes_bernstein}
\end{theorem}

In this PAC-Bayes Bernstein bound, the KL divergence penalty has been replaced with $I_{\rho}(S;A)$, which is the mutual information between states and actions under the policy $\rho(a|s)$. Let $\bar{\rho}(a) = (1/N)\sum_{s}\rho(a|s)$ denote the marginal distribution over $\mathcal{A}$ that corresponds to $\rho(a|s)$ and a uniform distribution over $\mathcal{S}$. Then $I_{\rho}(S;A)$ is defined as:
\begin{equation}
I_{\rho}(S;A) = \frac{1}{N}\sum_{s, a}\rho(a|s)\mathrm{ln}\left(\frac{\rho(a|s)}{\bar{\rho}(a)}\right).
\end{equation}

As shown by Seldin and Tishby \cite{seldin2010pac}, there exists a distribution $\mu$ over $\Pi$ such that for every $\rho$ over $\Pi$:
\begin{equation*}
D_{\mathrm{KL}}(\rho||\mu) \leq NI_{\rho}(S;A) + K\mathrm{ln}(N) + K\mathrm{ln}(K).
\end{equation*}

We could have also choosen $\mu$ to be a uniform prior, in which case $D_{\mathrm{KL}}(\rho||\mu) \leq N\mathrm{ln}(K)$. However, $I_{\rho}(S;A) \leq \mathrm{ln}(K)$, so when the number of states $N$ is much larger than the number of actions $K$, we have $NI_{\rho}(S;A) + K\mathrm{ln}(N) + K\mathrm{ln}(K) \leq N\mathrm{ln}(K)$. Seldin et al. \cite{seldin2011cb} derive a cumulative regret bound for a family of contextual bandit algorithms. Let $\rho_n(a)$ be an arbitrary distribution over $\mathcal{A}$. Let $\tilde{\rho}_n^{\mathrm{exp}}(a|s)$ denote the following smoothed Gibbs policy:
\begin{align}
\rho_n^{\mathrm{exp}}(a|s) &\propto \rho_n(a)e^{\gamma_n r^{\mathrm{IS}}(s, a, D_n)},\label{eqn:exp4}\\
\tilde{\rho}_n^{\mathrm{exp}}(a|s) &= (1 - K\epsilon_{n+1})\rho_n^{\mathrm{exp}}(a|s) + \epsilon_{n+1}.\nonumber
\end{align}

Using the same regret decomposition as in Equation \ref{eqn:regret_decomposition}, one can obtain a per-round regret bound for playing $\tilde{\rho}_n^{\mathrm{exp}}$. Seldin et al. \cite{seldin2011cb} show that $\Delta^{\mathrm{IS}}(\rho_n^{\mathrm{exp}}, D_n) \leq \mathrm{ln}(1/\epsilon_{n+1})/\gamma_n$, and that $R(\rho_n^{\mathrm{exp}}) - R(\tilde{\rho}_n^{\mathrm{exp}}) \leq K\epsilon_{n+1}$ also holds in the CB setting. If the PAC-Bayes Bernstein bound from Theorem \ref{thm:cb_is_reg_pac_bayes_bernstein} is used to bound $\Delta(\rho_n^{\mathrm{exp}}) - \Delta^{\mathrm{IS}}(\rho_n^{\mathrm{exp}}, D_n)$, then we obtain the following per-round regret bound.

\begin{theorem}[CB PAC-Bayes Bernstein per-round regret bound \cite{seldin2011cb}]
For any $\delta \in (0, 1]$ and any $c > 1$, with probability at least $1 - \delta$, for all policies $\rho^{\mathrm{exp}}$ that satisfy Equation \ref{eqn:rho_cond}, the expected per-round regret $\Delta(\tilde{\rho}_n^{\mathrm{exp}})$ is bounded by:
\begin{align*}
\Delta&(\tilde{\rho}_n^{\mathrm{exp}}) \leq (1+c)\sqrt{2(e-2)}\\
&\times\sqrt{\frac{NI_{\rho_n^{\mathrm{exp}}}(S;A) + K(\mathrm{ln}(N) + \mathrm{ln}(K)) + \mathrm{ln}(2m_n/\delta)}{n \epsilon_n}}\\
&+ \frac{\mathrm{ln}(\epsilon_{n+1})}{\gamma_n} + K\epsilon_{n+1},
\end{align*}

and for all $\rho^{\mathrm{exp}}$ that do not satisfy Equation \ref{eqn:rho_cond}, with the same probability:
\begin{align*}
\Delta(\tilde{\rho}_n^{\mathrm{exp}}) &\leq \frac{2NI_{\rho_n^{\mathrm{exp}}}(S;A) + K(\mathrm{ln}(N) + \mathrm{ln}(K)) + \mathrm{ln}(2m_n/\delta)}{n \epsilon_n}\\
&+ \frac{\mathrm{ln}(\epsilon_{n+1})}{\gamma_n} + K\epsilon_{n+1}.
\end{align*}
\label{thm:is_cb_bernstein_regret}
\end{theorem}
\vspace{-0.4cm}

If $\epsilon_{n} = n^{-1/3}K^{-1/3}N^{1/3}$, then this gives a cumulative regret bound of order $\mathcal{O}(n^{2/3}K^{2/3}N^{1/3})$, ignoring log terms. If we were to upper bound $D_{\mathrm{KL}}(\rho||\mu)$ by $N\mathrm{ln}(K)$ instead of the mutual information, then choosing $\epsilon_{n} = n^{-1/3}K^{-2/3}N^{1/3}$ would give a cumulative regret bound of order $\mathcal{O}(n^{2/3}K^{1/3}N^{1/3})$ ignoring log terms. Unfortunately, both of these bounds have sub-optimal scaling with $n$. For example, ignoring log terms, the EXP4.P algorithm of Beygelzimer et al. \cite{beygelzimer2011contextual} has cumulative regret bounded by $\mathcal{O}(\sqrt{nKN})$ in this problem.

\subsection{Optimising Bound Parameters}
\label{sec:app_bound_params}

Many PAC-Bayes bounds contain parameters that must be set before observing the data, such as $\lambda$ in the PAC-Bayes Bernstein bound in Theorem \ref{thm:is_bernstein}. We would like to be able to choose optimal values of these parameters. However, the optimal values are usually data-dependent. For example, the optimal $\lambda$ for the PAC-Bayes Bernstein bound from Theorem \ref{thm:is_bernstein}, using $V^{\mathrm{IS}}(\rho, D_n) \leq 1/\epsilon_n$, is:
\begin{align}
\lambda^{*} &= \argmin_{\lambda}\left\{\frac{\lambda(e-2)}{n\epsilon_n} + \frac{D_{\mathrm{KL}}(\rho||\mu) + \mathrm{ln}(1/\delta)}{\lambda}\right\}\nonumber\\
&= \sqrt{\frac{n\epsilon_n\left(D_{\mathrm{KL}}(\rho||\mu) + \mathrm{ln}(1/\delta)\right)}{e-2}}\label{eqn:optimal_lambda}
\end{align}

Since $\rho$ is data-dependent, $\lambda^*$ is as well. With this choice of $\lambda$, we would obtain the following (invalid) bound:
\begin{equation}
R(\rho) \geq r^{\mathrm{IS}}(\rho, D_n) - 2\sqrt{\frac{(e-2)\left(D_{\mathrm{KL}}(\rho||\mu) + \mathrm{ln}(1/\delta)\right)}{n\epsilon_n}}
\label{eqn:ideal_lamb}
\end{equation}

In this section we present methods for approximately optimising parameters of PAC-Bayes bounds, using the PAC-Bayes Bernstein bound as an example. We compare how close each of them is to the bound in Equation \ref{eqn:ideal_lamb}.

\subsubsection{Sample Splitting}

One approach is to split the data set into subsets of equal size $D_n = D_{1:n/2} \cup D_{n/2:n}$. The first subset is used to find a good value for $\lambda$. For example, we can approximate $\lambda^*$ by $\hat{\lambda}$:
\begin{align*}
\hat{\lambda} = \argmax_{\lambda}\bigg\{\max_{\rho}\bigg\{&r^{\mathrm{IS}}(\rho, D_{1:n/2}) - \frac{\lambda 2(e-2)}{n\epsilon_n}\\
&- \frac{D_{\mathrm{KL}}(\rho||\mu) + \mathrm{ln}(1/\delta)}{\lambda}\bigg\}\bigg\}.
\end{align*}

Since the Bernstein bound holds only for $\lambda \in (0, n\epsilon_n]$, we should take the minimum of $\hat{\lambda}$ and $(n/2)\epsilon_n$. The bound is then evaluated on the second subset with $\lambda = \hat{\lambda}$. Since $\hat{\lambda}$ does not depend on $D_{n/2:n}$, this yields a valid bound.

\begin{theorem}[PAC-Bayes Bernstein Bound with a Subset $\lambda$]
For any $\delta \in (0, 1)$, any prior $\mu \in \mathcal{P}(\Pi)$ and $\tilde{\lambda} = \min(\hat{\lambda}, (n/2)\epsilon_n)$, with probability at least $1 - \delta$ and for all $\rho \in \mathcal{P}(\Pi)$ simultaneously:
\begin{equation*}
R(\rho) \geq r^{\mathrm{IS}}(\rho, D_{n/2:n}) - \frac{\tilde{\lambda}2(e-2)}{n\epsilon_n} - \frac{D_{\mathrm{KL}}(\rho||\mu) + \mathrm{ln}(1/\delta)}{\tilde{\lambda}}.
\end{equation*}
\label{thm:pac_bayes_bern_subset_bound}
\end{theorem}

If $\hat{\lambda}$ is an accurate approximation of $\lambda^*$, and $\hat{\lambda} \leq (n/2)\epsilon_n$, then the PAC-Bayes Bernstein bound evaluated on $D_{n/2:n}$ and with $\lambda = \hat{\lambda}$ is approximately:
\begin{equation*}
R(\rho) \geq r^{\mathrm{IS}}(\rho, D_{n/2:n}) - 2\sqrt{2}\sqrt{\frac{(e-2)\left(D_{\mathrm{KL}}(\rho||\mu) + \mathrm{ln}(1/\delta)\right)}{n\epsilon_n}}.
\end{equation*}

Compared to the bound in Equation \ref{eqn:ideal_lamb}, this bound has a factor of $\sqrt{2}$ in front of the penalty term because it is evaluated using half as many samples.

\subsubsection{Union Bounds and Grids}
\label{sec:union_grid}

Another approach is to define a grid of parameter values, and then use the union bound to obtain a PAC-Bayes bound that holds simultaneously for all values in the grid with high probability. Suppose we choose the following grid $\Lambda = \{\lambda_1, \dots, \lambda_m\}$ and that $\sum_{i=1}^{m}\delta_i = \delta$. We have that for each $i$, with probability at least $1 - \delta_i$:
\begin{equation*}
R(\rho) \geq r^{\mathrm{IS}}(\rho, D_n) - \frac{\lambda_i(e-2)}{n\epsilon_n} - \frac{D_{\mathrm{KL}}(\rho||\mu) + \mathrm{ln}(1/\delta_i)}{\lambda_i}.
\end{equation*}

By a union bound argument, this bound holds for all $\lambda_i \in \Lambda$ simultaneously with probability at least $1 - \delta$. This allows us to choose the best $\lambda \in \Lambda$ after observing the data.

We may also optimise $\lambda$ over a continuous interval. For example, say we want the PAC-Bayes bound to hold with high probability for all $\lambda$ in the interval $[a, b]$ simultaneously, where $0 < a \leq b \leq n\epsilon_n$. We can specify a geometric grid $\Lambda = \{c^ka| k \in \mathbb{N}\} \cap [a, b]$, where $c > 1$. The number of elements in $\Lambda$ is no more than $\mathrm{log}_{c}(b/a) = \mathrm{ln}(b/a)/\mathrm{ln}(c)$. Using the union bound once more, and with $\delta_i = \frac{\mathrm{ln}(b/a)/\mathrm{ln}(c)}{\delta}$, the PAC-Bayes bound holds for all $\lambda \in \Lambda$ with probability at least $1 - \delta$. For any $\lambda \in [a, b]$, there exists a $\lambda^{\prime} \in \Lambda$ with $\lambda^{\prime} \leq \lambda \leq c\lambda^{\prime}$. We can evaluate the bound at this $\lambda^{\prime}$ and then upper bound the terms containing $\lambda^{\prime}$ with terms containing $\lambda$. We then have that with probability at least $1 - \delta$:

\begin{align}
R(\rho) &\geq r^{\mathrm{IS}}(\rho, D_n) - \min_{\lambda \in [a, b]}\bigg\{\frac{\lambda(e-2)}{n\epsilon_n}\label{eqn:lamb_geom}\\
&+ \frac{c\left(D_{\mathrm{KL}}(\rho||\mu) + \mathrm{ln}\left(\frac{\mathrm{ln}(b/a)/\mathrm{ln}(c)}{\delta}\right)\right)}{\lambda}\bigg\}.\nonumber
\end{align}

If the value of $\lambda$ that optimises the bound in Equation \ref{eqn:lamb_geom} is in $[a, b]$, then the bound can be rewritten as:
\begin{align*}
R(\rho) &\geq r^{\mathrm{IS}}(\rho, D_n) \\
&- (1 + c)\sqrt{\frac{(e-2)\left(D_{\mathrm{KL}}(\rho||\mu) + \mathrm{ln}\left(\frac{\mathrm{ln}(b/a)/\mathrm{ln}(c)}{\delta}\right)\right)}{n\epsilon_n}}.
\end{align*}

This bound is the same as the bound in Equation \ref{eqn:ideal_lamb}, except that there is a factor of $1 + c$ instead of 2 in front of the KL divergence penalty and $\mathrm{ln}(1/\delta)$ has been replaced with $\mathrm{ln}(\frac{\mathrm{ln}(b/a)/\mathrm{ln}(c)}{\delta}))$.

For best results we need to choose $a$ and $b$ such that the optimal $\lambda$ is in $[a, b]$ but small enough that $\mathrm{ln}(b/a)$ is not too large. We should choose $c$ to be small enough that $1 + c$ is close to 2, but large enough that $1/\mathrm{ln}(c)$ is small. To choose a suitable $a$ and $b$, we can lower and upper bound any data-dependent terms in the equation for the optimal $\lambda^*$, such as $D_{\mathrm{KL}}(\rho||\mu)$ in Equation \ref{eqn:optimal_lambda}. With $a = \sqrt{n \epsilon_n \mathrm{ln}(1/\delta)/(e-2)}$ and $b = n\epsilon_n$, and following Seldin et al. \cite{seldin2012mart}, one can obtain the following theorem.

\begin{theorem}[PAC-Bayes Bernstein Bound with a Geometric $\lambda$ Grid \cite{seldin2012mart}]
For any $\delta \in (0, 1)$, any prior $\mu \in \mathcal{P}(\Pi)$ and any $c > 1$, with probability at least $1 - \delta$, simultaneously for all $\rho \in \mathcal{P}(\Pi)$ that satisfy:
\begin{equation*}
\sqrt{\frac{D_{\mathrm{KL}}(\rho||\mu) + \mathrm{ln}(\nu/\delta)}{n(e-2)V^{\mathrm{IS}}(\rho, D_n)}} \leq \epsilon_n,
\end{equation*}

we have:
\begin{equation*}
R(\rho) \geq r^{\mathrm{IS}}(\rho, D_n) - (1 + c)\sqrt{\frac{(e-2)\left(D_{\mathrm{KL}}(\rho||\mu) + \mathrm{ln}(\nu/\delta)\right)}{n\epsilon_n}},
\end{equation*}

and for all other $\rho \in \mathcal{P}(\Pi)$ with the same probability, we have:
\begin{equation*}
R(\rho) \geq r^{\mathrm{IS}}(\rho, D_n) - 2\frac{D_{\mathrm{KL}}(\rho||\mu) + \mathrm{ln}(\nu/\delta)}{n\epsilon_n},
\end{equation*}

where $\nu = \mathrm{ln}(\sqrt{n\epsilon_n (e-2)/\mathrm{ln}(1/\delta)})/\mathrm{ln}(c)$.
\label{thm:pac_bayes_bern_grid_bound}
\end{theorem}

We believe that Langford and Caruana \cite{langford2002not} were the first to use a geometric grid. This approach can be extended to infinite (but countable) grids, which allows us to optimise $\lambda$ over an interval $[a, \infty)$. For example, see \cite{catoni2007pac} or \cite{seldin2012mart}. One can use the same techniques to optimise the clipping parameter $\tau$ in any of the PAC-Bayes bounds for the CIS estimate from Section \ref{sec:rew_cis_estimate}. London and Sandler \cite{london2019bayesian} provided a version of their risk bound (Theorem \ref{thm:cis_pac_bayes_risk}) where $\tau$ can be optimised over the interval $(0, 1)$.

\section{Proofs}

\subsection{$M_n^{\mathrm{IS}}(\pi)$ is a martingale}
\label{sec:xis_mart_proof}

\begin{lemma}
The sequence $\{X^{\mathrm{IS}}_i(\pi)\}_{i=1}^{n}$ defined as:
\begin{equation*}
X^{\mathrm{IS}}_i(\pi) = \frac{\pi(a_i)}{b(a_i|D_{i-1})}r_i - R(\pi),
\end{equation*}
is a martingale difference sequence with respect to $\{(a_i, r_i)\}_{i=1}^{n}$. Moreover, if the importance weights $\pi(a)/b(a|D_{i-1})$ are uniformly bounded above by $1/\epsilon_n$, then each $X^{\mathrm{IS}}_i(\pi)$ is uniformly bounded in the range $[-R(\pi), 1/\epsilon_n - R(\pi)]$, and the sum of the sequence is:
\begin{equation*}
\sum_{i=1}^{n}X^{\mathrm{IS}}_i(\pi) = n(r^{\mathrm{IS}}(\pi, D_n) - R(\pi)).
\end{equation*}
\label{lem:xis_mart}
\end{lemma}

Since $M_n^{\mathrm{IS}}(\pi) = \sum_{i=1}^{n}X^{\mathrm{IS}}_i(\pi)$, Lemma \ref{lem:xis_mart} shows that $M_n^{\mathrm{IS}}(\pi)$ is a martingale.

\begin{proof}[Proof of Lemma \ref{lem:xis_mart}]
We first verify that, for any $\pi \in \Pi$, $\{X^{\mathrm{IS}}_i(\pi)\}_{i=1}^{n}$ is a martingale difference sequence with respect to $\{(a_i, r_i)\}_{i=1}^{n}$:
\begin{align*}
\mathop{\mathbb{E}}&\left[X_i^{\mathrm{IS}}(\pi)\bigg|D_{i-1}\right] = \mathop{\mathbb{E}}_{\substack{a_i \sim b(\cdot|D_{i-1}),\\ r_i \sim P_R(\cdot|a_i)}}\left[\frac{\pi(a_i)}{b(a_i|D_{i-1})}r_i - R(\pi)\right],\\
&= \mathop{\mathbb{E}}_{a_i \sim b(\cdot|D_{i-1})}\left[\frac{\pi(a_i)}{b(a_i|D_{i-1})}\mathop{\mathbb{E}}_{r_i \sim P_R(\cdot|a_i)}[r_i]\right] - R(\pi),\\
&= \mathop{\mathbb{E}}_{a_i \sim \pi(\cdot)}\left[\mathop{\mathbb{E}}_{r_i \sim P_R(\cdot|a_i)}[r_i]\right] - R(\pi),\\
&= R(\pi) - R(\pi) = 0.
\end{align*}

Next, we verify that each $X^{\mathrm{IS}}_i(\pi)$ is bounded in the interval $[-R(\pi), 1/\epsilon_n - R(\pi)]$. If the importance weights $\pi(a)/b(a|D_{i-1})$ are uniformly bounded above by $1/\epsilon_n$ and the rewards are bounded in $[0, 1]$, then for any $i$, $(\pi(a_i)/b(a_i|D_{i-1}))r_i \in [0, 1/\epsilon_n]$. Therefore, $X^{\mathrm{IS}}_i(\pi) = (\pi(a_i)/b(a_i|D_{i-1}))r_i - R(\pi) \in [-R(\pi), 1/\epsilon_n - R(\pi)]$.

Finally, we verify that $\{X^{\mathrm{IS}}_i(\pi)\}_{i=1}^{n}$ sums to $n(r^{\mathrm{IS}}(\pi, D_n) - R(\pi))$:
\begin{align*}
\sum_{i=1}^{n}X^{\mathrm{IS}}_i(\pi) &= \sum_{i=1}^{n}\left(\frac{\pi(a_i)}{b(a_i|D_{i-1})}r_i - R(\pi)\right),\\
&= n\left(\frac{1}{n}\sum_{i=1}^{n}\frac{\pi(a_i)}{b(a_i|D_{i-1})}r_i - R(\pi)\right),\\
&= n\left(r^{\mathrm{IS}}(\pi, D_n) - R(\pi)\right).
\end{align*} 
\end{proof}

\subsection{Bias of the CIS estimate}
\label{sec:cis_bias_proof}

\begin{lemma}[Bias of the CIS estimate]
The CIS estimate is biased to underestimate the expected reward:
\begin{equation}
R^{\mathrm{CIS}}(\rho) \leq R(\rho).\label{eqn:mab_cis_bias}
\end{equation}
\label{lem:cis_bias}
\end{lemma}
\vspace{-0.4cm}

\begin{proof}[Proof of Lemma \ref{lem:cis_bias}]
First, we show that $r^{\mathrm{IS}}(\pi, D_n)$ is an unbiased estimate of $R(\pi)$. For any $\pi \in \Pi$, any $i \in 1, \dots, n$, and any history $D_{i-1}$, we have that:
\begin{align*}
\mathop{\mathbb{E}}&_{\substack{a_i \sim b(\cdot|D_{i-1}) \\ r_i \sim P_R(\cdot|a_i)}}\left[\frac{\pi(a_i)}{b(a_i|D_{i-1})}r_i\right]\\
&= \mathop{\mathbb{E}}_{a_i \sim b(\cdot|D_{i-1})}\left[\frac{\pi(a_i)}{b(a_i|D_{i-1})}\mathop{\mathbb{E}}_{r_i \sim P_R(\cdot|a_i)}[r_i]\right],\\
&= \mathop{\mathbb{E}}_{a_i \sim \pi(\cdot)}\left[\mathop{\mathbb{E}}_{r_i \sim P_R(\cdot|a_i)}[r_i]\right] = R(\pi).
\end{align*}

Therefore:
\begin{align*}
\mathop{\mathbb{E}}_{D_n}\left[r^{\mathrm{IS}}(\pi, D_n)\right] &= \mathop{\mathbb{E}}_{D_n}\left[\frac{1}{n}\sum_{i=1}^{n}\frac{\pi(a_i)}{b(a_i|D_{i-1})}r_i\right],\\
&= \frac{1}{n}\sum_{i=1}^{n}R(\pi) = R(\pi).
\end{align*}

Now, we have that:
\begin{align*}
\mathop{\mathbb{E}}_{D_n}\left[r^{\mathrm{CIS}}(\pi, D_n)\right] &= \mathop{\mathbb{E}}_{D_n}\left[\frac{1}{n}\sum_{i=1}^{n}\min\left(\frac{\pi(a_i)}{b(a_i|D_{i-1})}, \frac{1}{\tau}\right)r_i\right],\\
&\leq \mathop{\mathbb{E}}_{D_n}\left[\frac{1}{n}\sum_{i=1}^{n}\frac{\pi(a_i)}{b(a_i|D_{i-1})}r_i\right],\\
&= \mathop{\mathbb{E}}_{D_n}\left[r^{\mathrm{IS}}(\pi, D_n)\right] = R(\pi).
\end{align*}

Taking the expected value $\mathbb{E}_{\pi \sim \rho}[\cdot]$ of both sides yields the statement of the lemma. The proof for the CB case is the same.
\end{proof}

\subsection{Proof of Theorem \ref{thm:cis_pac_bayes_ha}}
\label{sec:cis_ha_proof}

First we state and prove a one-sided version of the Hoeffding-Azuma inequality for supermartingale difference sequences. If, in our basic definition of a martingale (Definition \ref{def:mart}), instead of the martingale property, we have $\mathbb{E}[M_n|X_1, \dots, X_{n-1}] \leq M_{n-1}$ for all $n \in \mathbb{N}$, then we call $(M_n|n \in \mathbb{N})$ a supermartingale.

\begin{lemma}[One-Sided Hoeffding-Azuma inequality]
Let $X_1, \dots, X_n$ be a supermartingale difference sequence (meaning $\mathbb{E}[X_i|X_1, \dots, X_{i-1}] \leq 0$ for $i = 1, \dots, n$) where each $X_i$ is bounded in the interval $[a, b]$. Then for any $\lambda \geq 0$:
\begin{equation*}
\mathop{\mathbb{E}}_{X_1, \dots, X_n}\left[e^{\lambda\sum_{i=1}^{n}X_i}\right] \leq e^{\frac{n\lambda^2(b-a)^2}{8}}.
\end{equation*}
\label{lem:one_side_hoeffding}
\end{lemma}

\begin{proof}[Proof of Lemma \ref{lem:one_side_hoeffding}]

First, by Hoeffding's Lemma (see, for example, Lemma A.1 of \cite{cesa2006prediction}), for any random variable bounded in the interval $[a, b]$ and any $\lambda \in \mathbb{R}$:
\begin{equation*}
\mathbb{E}\left[e^{\lambda X}\right] \leq e^{\lambda \mathbb{E}[X] + \frac{\lambda^2}{8}(b-a)^2}.
\end{equation*}

Now, for any $\lambda \geq 0$, we have that:
\begin{align*}
\mathop{\mathbb{E}}_{X_1, \dots, X_n}\left[e^{\lambda\sum_{i=1}^{n}X_i}\right] &= \mathop{\mathbb{E}}_{X_1, \dots, X_n}\left[\prod_{i=1}^{n}e^{\lambda X_i}\right],\\
&= \mathop{\mathbb{E}}_{X_1, \dots, X_{n-1}}\left[\mathop{\mathbb{E}}_{X_n}\condexp*{\prod_{i=1}^{n}e^{\lambda X_i}}{X_1, \dots, X_{n-1}}\right],\\
&\leq \mathop{\mathbb{E}}_{X_1, \dots, X_{n-1}}\bigg[e^{\lambda\mathop{\mathbb{E}}_{X_n}[X_n|X_1, \dots, X_{n-1}]},\\
&\qquad\qquad\qquad \times e^{\frac{\lambda^2}{8}(b - a)^2}\prod_{i=1}^{n-1}e^{\lambda X_i}\bigg],\\
&\leq e^{\frac{\lambda^2}{8}(b - a)^2}\mathop{\mathbb{E}}_{X_1, \dots, X_{n-1}}\left[e^{\lambda\sum_{i=1}^{n-1}X_i}\right].
\end{align*}

By iterating the above steps, we obtain:
\begin{equation*}
\mathop{\mathbb{E}}_{X_1, \dots, X_n}\left[e^{\lambda\sum_{i=1}^{n}X_i}\right] \leq \prod_{i=1}^{n}e^{\frac{\lambda^2(b - a)^2}{8}} = e^{\frac{n\lambda^2(b-a)^2}{8}}.
\end{equation*}
\end{proof}

Next, we show that the sequence $\{Y_i^{\mathrm{CIS}}(\pi)\}_{i=1}^{n}$, defined as:
\begin{equation*}
Y_i^{\mathrm{CIS}}(\pi) = \min\left(\frac{\pi(a_i)}{b(a_i|D_{i-1})}, \frac{1}{\tau}\right)r_i - R(\pi),
\end{equation*}

is a supermartingale difference sequence with respect to $\{(a_i, r_i)\}_{i=1}^{n}$, and that each term $Y_i^{\mathrm{CIS}}(\pi)$ is bounded in the interval $[-R(\pi), 1/\tau - R(\pi)]$. First, we have:
\begin{align*}
\mathop{\mathbb{E}}&\left[Y_i^{\mathrm{CIS}}(\pi)\bigg|D_{i-1}\right] = \mathop{\mathbb{E}}_{\substack{a_i \sim b(\cdot|D_{i-1}), \\ r_i \sim P_R(\cdot|a_i)}}\left[Y_i^{\mathrm{CIS}}(\pi)\right],\\
&= \mathop{\mathbb{E}}_{\substack{a_i \sim b(\cdot|D_{i-1}), \\ r_i \sim P_R(\cdot|a_i)}}\left[\min\left(\frac{\pi(a_i)}{b(a_i|D_{i-1})}, \frac{1}{\tau}\right)r_i - R(\pi)\right],\\
&\leq \mathop{\mathbb{E}}_{\substack{a_i \sim b(\cdot|D_{i-1}), \\ r_i \sim P_R(\cdot|a_i)}}\left[\frac{\pi(a_i)}{b(a_i|D_{i-1})}r_i - R(\pi)\right],\\
&= \mathop{\mathbb{E}}_{a_i \sim b(\cdot|D_{i-1})}\left[\frac{\pi(a_i)}{b(a_i|D_{i-1})}\mathop{\mathbb{E}}_{r_i \sim P_R(\cdot|a_i)}[r_i]\right] - R(\pi),\\
&= \mathop{\mathbb{E}}_{a_i \sim \pi(\cdot)}\left[\mathop{\mathbb{E}}_{r_i \sim P_R(\cdot|a_i)}[r_i]\right] - R(\pi),\\
&= R(\pi) - R(\pi) = 0.
\end{align*}

Since $\min\left(\pi(a_i)/b(a_i|D_{i-1}), 1/\tau\right)r_i \in [0, 1/\tau]$, we have $Y_i^{\mathrm{CIS}}(\pi) \in [-R(\pi), 1/\tau - R(\pi)]$. Therefore, the sequence $\{Y_i^{\mathrm{CIS}}(\pi)\}_{i=1}^{n}$ is compatible with the one-sided Hoeffding-Azuma inequality in Lemma \ref{lem:one_side_hoeffding}, with $a = -R(\pi)$ and $b = 1/\tau - R(\pi)$.

To prove Theorem \ref{thm:cis_pac_bayes_ha}, we can follow the steps taken in the proof of Theorem \ref{thm:ex_bound}, except with $h(\pi) = (\lambda/n)\sum_{i=1}^{n}Y_i^{\mathrm{CIS}}(\pi) = \lambda(r^{\mathrm{CIS}}(\pi, D_n) - R(\pi))$ and using the one-sided Hoeffding-Azuma inequality in Lemma \ref{lem:one_side_hoeffding} to upper bound $\mathbb{E}_{D_n}[\mathrm{exp}(\lambda(r^{\mathrm{CIS}}(\pi, D_n) - R(\pi)))]$.

\subsection{Variance of the CIS estimate}
\label{sec:cis_var_proof}

\begin{lemma}[Variance of the CIS estimate]
The average variance of the CIS estimate (both the MAB and CB versions) satisfies
\begin{equation*}
V^{\mathrm{CIS}}(\pi, D_n) \leq \frac{1}{\tau}.
\end{equation*}
\label{lem:cis_var_mab}
\end{lemma}
\vspace{-0.4cm}

\begin{proof}[Proof of Lemma \ref{lem:cis_var_mab}] We let
\begin{align}
X_i^{\mathrm{CIS}}(\pi) &= \min\left(\frac{\pi(a_i)}{b(a_i|D_{i-1})}, \frac{1}{\tau}\right)r_i\label{eqn:xcis_mart}\\
&- \mathop{\mathbb{E}}_{\substack{a_i \sim b(\cdot|D_{i-1}),\\ r_i \sim p_R(\cdot|a_i)}}\left[\min\left(\frac{\pi(a_i)}{b(a_i|D_{i-1})}, \frac{1}{\tau}\right)r_i\right].\nonumber
\end{align}

To bound $V^{\mathrm{CIS}}(\pi, D_n)$, we use that fact that the rewards are bounded in the interval $[0, 1]$.

\begin{align*}
V^{\mathrm{CIS}}&(\pi, D_n) = \frac{1}{n}\sum_{i=1}^{n}\mathop{\mathbb{E}}_{\substack{a_i \sim b(\cdot|D_{i-1}) \\ r_i \sim p_R(\cdot|a_i)}}\left[\left(X_i^{\mathrm{CIS}}(\pi)\right)^2\right],\\
&= \frac{1}{n}\sum_{i=1}^{n}\mathop{\mathbb{E}}_{\substack{a_i \sim b(\cdot|D_{i-1}) \\ r_i \sim p_R(\cdot|a_i)}}\left[\min\left(\frac{\pi(a_i)}{b(a_i|D_{i-1})}, \frac{1}{\tau}\right)^2{r_i}^2\right],\\
&- \frac{1}{n}\sum_{i=1}^{n}\mathop{\mathbb{E}}_{\substack{a_i \sim b(\cdot|D_{i-1}) \\ r_i \sim p_R(\cdot|a_i)}}\left[\min\left(\frac{\pi(a_i)}{b(a_i|D_{i-1})}, \frac{1}{\tau}\right){r_i}\right]^2,\\
&\leq \frac{1}{n}\sum_{i=1}^{n}\mathop{\mathbb{E}}_{a_i \sim b(\cdot|D_{i-1})}\left[\min\left(\frac{\pi(a_i)}{b(a_i|D_{i-1})}, \frac{1}{\tau}\right)^2\right],\\
&\leq \frac{1}{n}\sum_{i=1}^{n}\frac{1}{\tau}\mathop{\mathbb{E}}_{a_i \sim b(\cdot|D_{i-1})}\left[\frac{\pi(a_i)}{b(a_i|D_{i-1})}\right],\\
&= \frac{1}{n}\sum_{i=1}^{n}\frac{1}{\tau} = \frac{1}{\tau}.
\end{align*}
\end{proof}

\subsection{Proof of Theorem \ref{thm:cis_pac_bayes_bernstein}}
\label{sec:cis_bern_proof}

First, we state Bernstein's inequality for martingales.

\begin{lemma}[Bernstein's inequality]
Let $X_1, \dots, X_n$ be a martingale difference sequence where each $X_i$ is bounded in the interval $[-b, b]$, for some $b > 0$. Then for all $\lambda \in [0, 1/b]$:

\begin{equation*}
\mathop{\mathbb{E}}_{X_1, \dots, X_n}\left[e^{\lambda\sum_{i=1}^{n}X_i - (e-2)\lambda^2\sum_{i=1}^{n}\mathbb{E}_{X_i}[X_i^2|X_1, \dots, X_{i-1}]}\right] \leq 1.
\end{equation*}
\label{lem:bernstein}
\end{lemma}

For a proof, see Theorem 1 of \cite{beygelzimer2011contextual}. Next, we show that the sequence $\{X_i^{\mathrm{CIS}}(\pi)\}_{i=1}^{n}$, defined in Equation \ref{eqn:xcis_mart}, is a martingale difference sequence with respect to $\{(a_i, r_i)\}_{i=1}^{n}$, and that each term is bounded in the interval $[-1/\tau, 1/\tau]$. For any $\tau \in (0, 1]$, we have:
\begin{align*}
\mathop{\mathbb{E}}&\left[X_i^{\mathrm{CIS}}(\pi)\bigg|D_{i-1}\right] = \mathop{\mathbb{E}}_{\substack{a_i \sim b(\cdot|D_{i-1}), \\ r_i \sim P_R(\cdot|a_i)}}\left[X_i^{\mathrm{CIS}}(\pi)\right]\\
&= \mathop{\mathbb{E}}_{\substack{a_i \sim b(\cdot|D_{i-1}), \\ r_i \sim P_R(\cdot|a_i)}}\bigg[\min\left(\frac{\pi(a_i)}{b(a_i|D_{i-1})}, \frac{1}{\tau}\right)r_i\\
&\qquad \qquad- \mathop{\mathbb{E}}_{\substack{a_i \sim b(\cdot|D_{i-1}),\\ r_i \sim p_R(\cdot|a_i)}}\left[\min\left(\frac{\pi(a_i)}{b(a_i|D_{i-1})}, \frac{1}{\tau}\right)r_i\right]\bigg],\\
&= 0.
\end{align*}

For any $i \in \{1, \dots, n\}$, we have:
\begin{align*}
0 &\leq \mathop{\mathbb{E}}_{\substack{a_i \sim b(\cdot|D_{i-1}),\\ r_i \sim p_R(\cdot|a_i)}}\left[\min\left(\frac{\pi(a_i)}{b(a_i|D_{i-1})}, \frac{1}{\tau}\right)r_i\right]\\
&\leq \mathop{\mathbb{E}}_{\substack{a_i \sim b(\cdot|D_{i-1}),\\ r_i \sim p_R(\cdot|a_i)}}\left[\frac{\pi(a_i)}{b(a_i|D_{i-1})}r_i\right]\\
&= R(\pi) \leq 1.
\end{align*}

Since $\min\left(\pi(a_i)/b(a_i|D_{i-1}), 1/\tau\right)r_i \in [0, 1/\tau]$, we have $X_i^{\mathrm{CIS}}(\pi) \in [-1, 1/\tau] \subseteq [-1/\tau, 1/\tau]$. Also, we have that :
\begin{equation*}
\sum_{i=1}^{n}\mathop{\mathbb{E}}\left[(X_i^{\mathrm{CIS}}(\pi))^2|D_{i-1}\right] = nV^{\mathrm{CIS}}(\pi, D_n)
\end{equation*}

To prove Theorem \ref{thm:cis_pac_bayes_bernstein}, we can follow the steps taken in the proof of Theorem \ref{thm:ex_bound}, except with:
\begin{align*}
h(\pi) &= \lambda\sum_{i=1}^{n}X_i^{\mathrm{CIS}}(\pi) - \lambda^2(e-2)\sum_{i=1}^{n}\mathop{\mathbb{E}}\left[(X_i^{\mathrm{CIS}}(\pi))^2|D_{i-1}\right],\\
&\geq n\lambda\left(r^{\mathrm{CIS}}(\pi, D_n) - R(\pi)\right) - n\lambda^2(e-2)V^{\mathrm{CIS}}(\pi, D_n),
\end{align*}

and using Bernstein's inequality in Lemma \ref{lem:bernstein}.

\subsection{Proof of the Efron-Stein PAC-Bayes Bound (Theorem \ref{thm:efron_stein_pac_bayes})}
\label{sec:efron_stein_proof}

First, we recall the statement of the theorem. Let $f(\pi, D_n)$ be a real-valued function, let $F(\pi) = \mathbb{E}_{D_n}[f(\pi, D_n)]$ denote its expected value and let $V^{\mathrm{ES}}(\pi, D_n)$ be its semi-empirical Efron-Stein variance proxy. If the data set $D_n$ consists of independent random variables, then for any distribution $\mu$ over $\Pi$, any $y > 0$ and any $\delta \in (0, 1)$, with probability at least $1 - \delta$ (over the sampling of $D_n$), we have that for all distributions $\rho$ over $\Pi$:

\begin{align*}
&\left|f(\rho, D_n) - F(\rho)\right| \leq \sqrt{2\left(y + V^{\mathrm{ES}}(\rho, D_n)\right)}\\
&\times \sqrt{D_{\mathrm{KL}}(\rho||\mu) + \frac{1}{2}\mathrm{ln}\left(1 + V^{\mathrm{ES}}(\rho, D_n)/y\right) + \mathrm{ln}(1/\delta)}.
\end{align*}

In the proof, we use some technical lemmas. First, note that the Change of Measure Inequality in Lemma \ref{lem:donsk} can be stated as follows. For any measurable function $h: \Pi \to \mathbb{R}$ and any probability distribution $\mu \in \mathcal{P}(\Pi)$, such that $\mathbb{E}_{\pi \sim \mu}[e^{h(\pi)}] < \infty$, we have:

\begin{equation}
\sup_{\rho \in \mathcal{P}(\Pi)}\left\{\mathop{\mathbb{E}}_{\pi \sim \rho}\left[h(\pi)\right] - D_{\mathrm{KL}}(\rho||\mu)\right\} = \mathrm{ln}\left(\mathop{\mathbb{E}}_{\pi \sim \mu}\left[e^{h(\pi)}\right]\right).\label{eqn:alt_donsk}
\end{equation}

The first technical lemma is an Efron-Stein concentration inequality by Kuzborskij and Szepesv{\'a}ri \cite{kuzborskij2019efron}.

\begin{lemma}[Efron-Stein concentration inequality \cite{kuzborskij2019efron}]
Let $D_n = \{Z_i\}_{i=1}^{n}$ be a collection of independent random variables. Then for any $\pi \in \Pi$ and any $\lambda \in \mathbb{R}$:

\begin{equation*}
\mathop{\mathbb{E}}_{D_n}\left[e^{\lambda(f(\pi, D_n) - \mathbb{E}_{D_n}[f(\pi, D_n)]) - \frac{\lambda^2}{2}V^{\mathrm{ES}}(\pi, D_n)}\right] \leq 1.
\end{equation*}
\label{lem:es_conc}
\end{lemma}

The second technical lemma allows us to swap the order of a supremum and an exponentiation.

\begin{lemma}
For any set $A \subseteq \mathbb{R}$ where $\sup(A)$ exists, we have that:

\begin{equation*}
e^{\sup(A)} = \sup\left(e^A\right),
\end{equation*}

where $e^A = \{e^a|a \in A\}$.\label{lem:exp_sup}
\end{lemma}

\begin{proof}[Proof of Lemma \ref{lem:exp_sup}]
Let $\sup(A) = \alpha$. Suppose that $\sup\left(e^A\right) < e^{\alpha}$. Then there exists an $\epsilon > 0$ such that $\sup\left(e^A\right) \leq e^{\alpha - \epsilon}$. Therefore, for all $a \in A$, $e^a \leq e^{\alpha - \epsilon}$, and so $\alpha - \epsilon$ is an upper bound on $A$. This is a contradiction since $\alpha$ is the least upper bound on $A$.

Now, suppose that $\sup\left(e^A\right) > e^{\alpha}$. This means that $e^{\alpha}$ is not an upper bound for $e^{A}$. Therefore, there must exist an $a \in A$, where $e^a > e^{\alpha}$, and so $a > \alpha$. This is a contradiction, since $\sup(A) = \alpha$. Therefore we must have $e^{\sup(A)} = \sup\left(e^A\right)$.
\end{proof}

The third technical lemma allows us to upper bound the supremum of an integral by the integral of the supremum.

\begin{lemma}
For any function $g:\mathcal{P}(\Pi) \times \mathbb{R} \to \mathbb{R}$, such that $\sup_{\rho \in \mathcal{P}(\Pi)}\{g(\rho, \lambda)\} < \infty$ for all $\lambda \in \mathbb{R}$, we have that:

\begin{equation*}
\sup_{\rho \in \mathcal{P}(\Pi)}\left\{\int_{-\infty}^{\infty}g(\rho, \lambda)\mathrm{d}\lambda\right\} \leq \int_{-\infty}^{\infty}\sup_{\rho \in \mathcal{P}(\Pi)}\{g(\rho, \lambda)\}\mathrm{d}\lambda.
\end{equation*}
\label{lem:sup_int}
\end{lemma}

\begin{proof}[Proof of Lemma \ref{lem:sup_int}]
For every $\rho \in \mathcal{P}(\Pi)$ and $\lambda \in \mathbb{R}$:

\begin{equation*}
g(\rho, \lambda) \leq \sup_{\rho \in \mathcal{P}(\Pi)}\{g(\rho, \lambda)\}.
\end{equation*}

Therefore, for every $\rho \in \mathcal{P}(\Pi)$:

\begin{equation*}
\int_{-\infty}^{\infty}g(\rho, \lambda)\mathrm{d}\lambda \leq \int_{-\infty}^{\infty}\sup_{\rho \in \mathcal{P}(\Pi)}\{g(\rho, \lambda)\}\mathrm{d}\lambda.
\end{equation*}

Therefore, we have that:

\begin{equation*}
\sup_{\rho \in \mathcal{P}(\Pi)}\left\{\int_{-\infty}^{\infty}g(\rho, \lambda)\mathrm{d}\lambda\right\} \leq \int_{-\infty}^{\infty}\sup_{\rho \in \mathcal{P}(\Pi)}\{g(\rho, \lambda)\}\mathrm{d}\lambda.
\end{equation*}
\end{proof}

\begin{proof}[Proof of Theorem \ref{thm:efron_stein_pac_bayes}]
Throughout the proof, let $A(\pi, D_n) = f(\pi, D_n) - F(\pi)$. Also, let $A(\rho, D_n) = \mathbb{E}_{\pi \sim \rho}[A(\pi, D_n)]$ and let $V^{\mathrm{ES}}(\rho, D_n) = \mathbb{E}_{\pi \sim \rho}[V^{\mathrm{ES}}(\pi, D_n)]$.

Using the Change of Measure Inequality (Equation \ref{eqn:alt_donsk}), we have that, for all $\lambda \in \mathbb{R}$:

\begin{align*}
&\sup_{\rho \in \mathcal{P}(\Pi)}\left\{\lambda A(\rho, D_n) - \frac{\lambda^2}{2}V^{\mathrm{ES}}(\rho, D_n) - D_{\mathrm{KL}}(\rho||\mu)\right\}\\
&= \mathrm{ln}\left(\mathop{\mathbb{E}}_{\pi \sim \mu}\left[e^{\lambda A(\pi, D_n) - \frac{\lambda^2}{2}V^{\mathrm{ES}}(\pi, D_n)}\right]\right).
\end{align*}

We exponentiate and then take expected values (over $D_n$) of both sides to obtain:

\begin{align*}
&\mathop{\mathbb{E}}_{D_n}\left[e^{\sup_{\rho \in \mathcal{P}(\Pi)}\left\{\lambda A(\rho, D_n) - \frac{\lambda^2}{2}V^{\mathrm{ES}}(\rho, D_n) - D_{\mathrm{KL}}(\rho||\mu)\right\}}\right]\\
&= \mathop{\mathbb{E}}_{D_n}\left[\mathop{\mathbb{E}}_{\pi \sim \mu}\left[e^{\lambda A(\pi, D_n) - \frac{\lambda^2}{2}V^{\mathrm{ES}}(\pi, D_n)}\right]\right].
\end{align*}

We use Tonelli's theorem to swap the order of the expectations on the right-hand-side. Then we use the Efron-Stein Concentration Inequality in Lemma \ref{lem:es_conc} to obtain:

\begin{equation*}
\mathop{\mathbb{E}}_{D_n}\left[e^{\sup_{\rho \in \mathcal{P}(\Pi)}\left\{\lambda A(\rho, D_n) - \frac{\lambda^2}{2}V^{\mathrm{ES}}(\rho, D_n) - D_{\mathrm{KL}}(\rho||\mu)\right\}}\right] \leq 1.
\end{equation*}

Now, let $h(\rho, \mu, \lambda, D_n) = \lambda A(\rho, D_n) - \frac{\lambda^2}{2}V^{\mathrm{ES}}(\rho, D_n) - D_{\mathrm{KL}}(\rho||\mu)$. We use Lemma \ref{lem:exp_sup}, multiply both sides by $(y/\sqrt{2\pi})e^{-\frac{\lambda^2y^2}{2}}$ for $y > 0$, and then integrate w.r.t. $\lambda$ from $-\infty$ to $\infty$:

\begin{equation*}
\int_{-\infty}^{\infty}\frac{y}{\sqrt{2\pi}}e^{-\frac{\lambda^2y^2}{2}}\mathop{\mathbb{E}}_{D_n}\left[\sup_{\rho \in \mathcal{P}(\Pi)}\left\{e^{h(\rho, \mu, \lambda, D_n)}\right\}\right]\mathrm{d}\lambda \leq 1.
\end{equation*}

We use Fubini's theorem to swap the order of the integral and the expected value.

\begin{equation*}
\mathop{\mathbb{E}}_{D_n}\left[\int_{-\infty}^{\infty}\frac{y}{\sqrt{2\pi}}e^{-\frac{\lambda^2y^2}{2}}\sup_{\rho \in \mathcal{P}(\Pi)}\left\{e^{h(\rho, \mu, \lambda, D_n)}\right\}\mathrm{d}\lambda\right] \leq 1.
\end{equation*}

Using Lemma \ref{lem:sup_int}, we can move the integral inside the supremum.

\begin{equation*}
\mathop{\mathbb{E}}_{D_n}\left[\sup_{\rho \in \mathcal{P}(\Pi)}\left\{\int_{-\infty}^{\infty}\frac{y}{\sqrt{2\pi}}e^{-\frac{\lambda^2y^2}{2}}e^{h(\rho, \mu, \lambda, D_n)}\mathrm{d}\lambda\right\}\right] \leq 1.
\end{equation*}

We can now calculate the integral by rearranging the integrand to get a Gaussian density function.

\begin{align*}
\mathop{\mathbb{E}}_{D_n}\bigg[\sup_{\rho \in \mathcal{P}(\Pi)}\bigg\{&\frac{y}{\sqrt{y^2 + V^{\mathrm{ES}}(\rho, D_n)}}\\
&\times e^{\frac{A(\rho, D_n)^2}{2\left(y^2 + V^{\mathrm{ES}}(\rho, D_n)\right)} - D_{\mathrm{KL}}(\rho||\mu)}\bigg\}\bigg] \leq 1.
\end{align*}

This holds for all $\mu \in \mathcal{P}(\Pi)$ and $y > 0$. Now, we fix $\mu$ and $y$, and then use Markov's inequality. With probability at least $1 - \delta$, and for all $\rho \in \mathcal{P}(\Pi)$ simultaneously:

\begin{equation*}
\frac{y}{\sqrt{y^2 + V^{\mathrm{ES}}(\rho, D_n)}}e^{\frac{A(\rho, D_n)^2}{2\left(y^2 + V^{\mathrm{ES}}(\rho, D_n)\right)} - D_{\mathrm{KL}}(\rho||\mu)} \leq 1/\delta.
\end{equation*}

We can rearrange this inequality to obtain the following inequality, that holds with the same probability:

\begin{align*}
&\left|A(\rho, D_n)\right| \leq \sqrt{2\left(y^2 + V^{\mathrm{ES}}(\rho, D_n)\right)}\\
&\times \sqrt{D_{\mathrm{KL}}(\rho||\mu) + \frac{1}{2}\mathrm{ln}\left(1 + V^{\mathrm{ES}}(\rho, D_n)/y^2\right) + \mathrm{ln}(1/\delta)}.
\end{align*}

Since $y$ was an arbitrary positive number, we can replace $y^2$ with $y$ to recover the statement of the Theorem.
\end{proof}

\subsection{Proof for the Localised PAC-Bayes Bernstein Bound (Theorem \ref{thm:local_pac_bayes_bernstein})}
\label{sec:local_bern_proof}

This proof follows steps in Section 1.3.4. of \cite{catoni2007pac}. First, we recall the statement of the Theorem. For any $\lambda \in [0, n\epsilon_n]$, any $\beta$ satisfying $0 \leq \beta < \lambda$, any $\delta \in (0, 1)$ and any probability distribution $\mu \in \mathcal{P}(\Pi)$, with probability at least $1 - \delta$ and for all distributions $\rho \in \mathcal{P}(\Pi)$ simultaneously:

\begin{align*}
R(\rho) &\geq r^{\mathrm{IS}}(\rho, D_n) - \frac{(\lambda^2 + \beta^2)(e-2)}{(\lambda - \beta)n\epsilon_n}\\
&- \frac{D_{\mathrm{KL}}(\rho||\mu_{\beta r^{\mathrm{IS}}}) + 2\mathrm{ln}(1/\delta)}{\lambda - \beta}.
\end{align*}

Also, recall that $\mu_{\beta R}$ and $\mu_{\beta r^{\mathrm{IS}}}$ are Gibbs distributions defined as:

\begin{equation*}
\mu_{\beta R}(\pi) = \frac{\mu(\pi)e^{\beta R(\pi)}}{\mathop{\mathbb{E}}_{\pi \sim \mu}\left[e^{\beta R(\pi)}\right]}, \quad \mu_{\beta r^{\mathrm{IS}}}(\pi) = \frac{\mu(\pi)e^{\beta r^{\mathrm{IS}}(\pi, D_n)}}{\mathop{\mathbb{E}}_{\pi \sim \mu}\left[e^{\beta r^{\mathrm{IS}}(\pi, D_n)}\right]}.
\end{equation*}

\begin{proof}[Proof of Theorem \ref{thm:local_pac_bayes_bernstein}]
As an intermediate step in proving the original PAC-Bayes Bernstein bound, we have that for all $\lambda \in [-n\epsilon_n, n\epsilon_n]$:

\begin{align}
\mathop{\mathbb{E}}_{D_n}&\left[\mathrm{exp}\left(\sup_{\rho \in \mathcal{P}(\Pi)}\left\{\lambda r^{\mathrm{IS}}(\rho, D_n) - \lambda R(\rho) - D_{\mathrm{KL}}(\rho||\mu_{\beta R})\right\}\right)\right]\nonumber\\
&\leq \mathrm{exp}\left(\frac{\lambda^2(e-2)}{n\epsilon_n}\right).\label{eqn:local_be1}
\end{align}

Next, we attempt to find a relationship between $D_{\mathrm{KL}}(\rho||\mu_{\beta R})$ and $D_{\mathrm{KL}}(\rho||\mu_{\beta r^{\mathrm{IS}}})$. We have that:

\begin{align*}
D_{\mathrm{KL}}(\rho||\mu_{\beta R}) &= \mathop{\mathbb{E}}_{\pi \sim \rho}\left[\mathrm{ln}\left(\frac{\rho(\pi)}{\mu_{\beta R}(\pi)}\right)\right],\\
&= \mathop{\mathbb{E}}_{\pi \sim \rho}\left[\mathrm{ln}\left(\frac{\rho(\pi)}{\mu_{\beta r^{\mathrm{IS}}}(\pi)}\frac{\mu_{\beta r^{\mathrm{IS}}}(\pi)}{\mu_{\beta R}(\pi)}\right)\right],\\
&= D_{\mathrm{KL}}(\rho||\mu_{\beta r^{\mathrm{IS}}}) + \beta\mathop{\mathbb{E}}_{\pi \sim \rho}\left[r^{\mathrm{IS}}(\pi, D_n) - R(\pi)\right]\\
&+ \mathrm{ln}\left(\mathop{\mathbb{E}}_{\pi \sim \mu}\left[e^{\beta R(\pi)}\right]\right) - \mathrm{ln}\left(\mathop{\mathbb{E}}_{\pi \sim \mu}\left[e^{\beta r^{\mathrm{IS}}(\pi, D_n)}\right]\right).
\end{align*}

The last tool we need is a bound on $\mathrm{ln}(\mathbb{E}_{\pi \sim \mu}[e^{\beta R(\pi)}]) - \mathrm{ln}(\mathbb{E}_{\pi \sim \mu}[e^{\beta r^{\mathrm{IS}}(\pi, D_n)}])$. Using two applications of the Change of Measure Inequaity (Equation \ref{eqn:alt_donsk}), we have that:

\begin{align*}
\mathop{\mathbb{E}}_{D_n}&\left[\mathrm{exp}\left(\mathrm{ln}\left(\mathop{\mathbb{E}}_{\pi \sim \mu}\left[e^{\beta R(\pi)}\right]\right) - \mathrm{ln}\left(\mathop{\mathbb{E}}_{\pi \sim \mu}\left[e^{\beta r^{\mathrm{IS}}(\pi, D_n)}\right]\right)\right)\right]\\
&= \mathop{\mathbb{E}}_{D_n}\bigg[\mathrm{exp}\bigg(\mathrm{ln}\left(\mathop{\mathbb{E}}_{\pi \sim \mu}\left[e^{\beta R(\pi)}\right]\right)\\
&\qquad\quad + \inf_{\rho \in \mathcal{P}(\Pi)}\left\{-\beta r^{\mathrm{IS}}(\rho, D_n) + D_{\mathrm{KL}}(\rho||\mu)\right\}\bigg)\bigg],\\
&\leq \mathop{\mathbb{E}}_{D_n}\bigg[\mathrm{exp}\bigg(\mathrm{ln}\left(\mathop{\mathbb{E}}_{\pi \sim \mu}\left[e^{\beta R(\pi)}\right]\right)\\
&\qquad\quad -\beta r^{\mathrm{IS}}(\mu_{\beta R}, D_n) + D_{\mathrm{KL}}(\mu_{\beta R}||\mu)\bigg)\bigg],\\
&= \mathop{\mathbb{E}}_{D_n}\bigg[\mathrm{exp}\bigg(\beta R(\mu_{\beta R}) - D_{\mathrm{KL}}(\mu_{\beta R}||\mu)\\
&\qquad\quad -\beta r^{\mathrm{IS}}(\mu_{\beta R}, D_n) + D_{\mathrm{KL}}(\mu_{\beta R}||\mu)\bigg)\bigg],\\
&= \mathop{\mathbb{E}}_{D_n}\bigg[\mathrm{exp}\bigg(\beta R(\mu_{\beta R}) -\beta r^{\mathrm{IS}}(\mu_{\beta R}, D_n)\bigg)\bigg],\\
&\leq \mathrm{exp}\left(\frac{\beta^2(e-2)}{n \epsilon_n}\right).
\end{align*}

The final inequality is obtained by using Equation \ref{eqn:local_be1} with $\lambda = -\beta$. This inequality and the one in Equation \ref{eqn:local_be1} can be combined using the Cauchy-Schwarz inequality:

\begin{align*}
&\mathop{\mathbb{E}}_{D_n}\bigg[\mathrm{exp}\bigg(\frac{1}{2}\sup_{\rho \in \mathcal{P}(\Pi)}\bigg\{(\lambda - \beta)(r^{\mathrm{IS}}(\rho, D_n) - R(\rho))\\
&\qquad\qquad\qquad\qquad\quad - D_{\mathrm{KL}}(\rho||\mu_{\beta r^{\mathrm{IS}}})\bigg\}\bigg)\bigg]\\
&= \mathop{\mathbb{E}}_{D_n}\bigg[\mathrm{exp}\bigg(\frac{1}{2}\sup_{\rho \in \mathcal{P}(\Pi)}\bigg\{\lambda(r^{\mathrm{IS}}(\rho, D_n) - R(\rho))\\
&\qquad\qquad\qquad\qquad\qquad\;\ - D_{\mathrm{KL}}(\rho||\mu_{\beta R})\bigg\}\bigg)\\
&\times \mathrm{exp}\bigg(\frac{1}{2}\left[\mathrm{ln}\left(\mathop{\mathbb{E}}_{\pi \sim \mu}\left[e^{\beta R(\pi)}\right]\right) - \mathrm{ln}\left(\mathop{\mathbb{E}}_{\pi \sim \mu}\left[e^{\beta r^{\mathrm{IS}}(\pi, D_n)}\right]\right)\right]\bigg)\bigg],\\
&\leq \mathop{\mathbb{E}}_{D_n}\bigg[\mathrm{exp}\bigg(\sup_{\rho \in \mathcal{P}(\Pi)}\bigg\{\lambda(r^{\mathrm{IS}}(\rho, D_n) - R(\rho))\\
&\qquad\qquad\qquad\qquad\quad\;\ - D_{\mathrm{KL}}(\rho||\mu_{\beta R})\bigg\}\bigg)\bigg]^{1/2}\\
&\times \mathop{\mathbb{E}}_{D_n}\bigg[\mathrm{exp}\bigg(\mathrm{ln}\left(\mathop{\mathbb{E}}_{\pi \sim \mu}\left[e^{\beta R(\pi)}\right]\right) - \mathrm{ln}\left(\mathop{\mathbb{E}}_{\pi \sim \mu}\left[e^{\beta r^{\mathrm{IS}}(\pi, D_n)}\right]\right)\bigg)\bigg]^{1/2},\\
&\leq \mathrm{exp}\left(\frac{\lambda^2(e-2)}{n\epsilon_n}\right)^{1/2}\mathrm{exp}\left(\frac{\beta^2(e-2)}{n\epsilon_n}\right)^{1/2},\\
&= \mathrm{exp}\left(\frac{(\lambda^2 + \beta^2)(e-2)}{2n\epsilon_n}\right).
\end{align*}

We use Markov's inequality and then rearrange the result to obtain the statement of the Theorem.
\end{proof}

\section{Further Information About The Experiments}

\subsection{Details About Classification Data Sets}
\label{sec:class_data_sets}

The four data sets (see Table \ref{tab:data}) came from either OpenML (OptDigits, PenDigits and Chars) or the UCI Machine Learning Repository (DriveDiag). In the UCI Repository, the DriveDiag data set can be found under its full name: "Dataset for Sensorless Drive Diagnosis Data Set".

\begin{table}[H]
\centering
\begin{tabular}{lcccc}
\toprule
Name & OptDigits & PenDigits & Chars & DriveDiag\\
\toprule
OpenML ID & $28$ & $32$ & $1459$ & n/a\\
Size ($n$) & $4496$ & $8793$ & $8174$ & $46807$\\
Input dim ($d$) & $64$ & $16$ & $7$ & $48$\\
Classes ($K$) & $10$ & $10$ & $10$ & $11$\\
\bottomrule
\end{tabular}
\caption{The OpenML ID number, size, input dimensionality and number of classes for all the data sets we use in the CB Classification benchmark.}
\label{tab:data}
\end{table}

In Table \ref{tab:data}, the reported size of each data set is approximately 80\% of the size of the original data set. At the start of each repetition of each experiment, we perform a random 80:20 split. We use 80\% of the data to generate the training data for the offline bandit problem, learn a policy and then evaluate the reward bound. The remaining 20\% of the data are used to estimate the expected reward of the learned policy. Therefore, the reported data set size reflects the number of examples used to learn a policy and evaluate a lower bound on the expected reward.

\subsection{Details About Bound Optimisation and Evaluation}
\label{sec:opt_and_eval}

In the MAB benchmark, we allow $\rho$ to be any distribution over $\mathcal{A}$, so each $\rho$ is an element of the standard $K$-simplex. Unless stated otherwise, $\mu$ is a uniform prior. For bounds where the optimal $\rho$ is a Gibbs posterior (Hoeffding-Azuma and Bernstein), we use the optimal Gibbs posterior. For bounds where the optimal $\rho$ is not known in closed-form (Pinsker and $kl^{-1}$), we optimise them by gradient ascent. As shown by Reeb et al. \cite{reeb2018learning}, the derivatives of $kl^{-1}$ can be calculated by differentiating the identity $kl(p||kl^{-1}(p, b)) = b$.

In the CB benchmark, we restrict $\rho$ to be a diagonal Gaussian distribution $\mathcal{N}(\bs{m}, \bs{\sigma}^2I)$, where $\bs{m}$ and $\bs{\sigma}$ are $d \times K$-dimensional vectors, over the weight matrix of the linear softmax policy. Unless stated otherwise, $\mu$ is a standard Gaussian prior over the weights. We optimise each bound with respect to $\rho$ by stochastic gradient ascent, using the local reparameterisation trick \cite{kingma2015local} to calculate stochastic gradients. For the bounds that are linear in the reward estimate (e.g. $\mathbb{E}_{\theta \sim \mathcal{N}(\bs{m}, \bs{\sigma}^2I)}[r^{\mathrm{IS}}(\pi_{\theta}, D_n)]$), this procedure will converge to the mean and variance parameters that maximise the bound. However, when we approximate $\mathbb{E}_{\theta \sim \mathcal{N}(\bs{m}, \bs{\sigma}^2I)}[r^{\mathrm{IS}}(\pi_{\theta}, D_n)]$ with a single sample in the $kl^{-1}$ bound, this procedure will converge to the mean and variance parameters that maximise:
\begin{align*}
\frac{1}{\epsilon_n}\mathop{\mathbb{E}}_{\theta \sim \mathcal{N}(\bs{m}, \bs{\sigma}^2I)}\bigg[&kl^{-1}\bigg(\epsilon_n r^{\mathrm{IS}}(\pi_{\theta}, D_n),\\ &\frac{D_{\mathrm{KL}}(\mathcal{N}(\bs{m}, \bs{\sigma}^2I)||\mu) + \mathrm{ln}(2\sqrt{n}/\delta)}{n}\bigg)\bigg],
\end{align*}

which may not be the mean and variance parameters that maximise the $kl^{-1}$ bound. Nevertheless, the resulting approximately optimal Gaussian posterior still results in a valid bound. In all our experiments, this approximation appeared to work well.

For the PAC-Bayes Hoeffding-Azuma and Bernstein bounds, we set $\lambda$ to the (data-independent) value that would be optimal if $D_{\mathrm{KL}}(\rho||\mu) = 0$. In the MAB benchmark we can calculate $\mathbb{E}_{\pi \sim \rho}[r^{\mathrm{IS}}(\pi, D_n)]$ exactly. In the CB benchmark, when evaluating the bound value, we approximate $\mathbb{E}_{\pi \sim \rho}[r^{\mathrm{IS}}(\pi, D_n)]$ by averaging over 100 samples from $\rho$.

\subsection{Details About Implementation of the Priors}
\label{sec:prior_implementation}

We tested the sample splitting prior with a value of $m = n/2$, meaning half the data were used to learn a prior and the other half were used to optimise and evaluate a PAC-Bayes bound. To learn priors from the subset $D_{1:m}$ of the data, we first split $D_{1:m}$ into training data $D_{\mathrm{tr}}$ and validation data $D_{\mathrm{val}}$.

In the MAB benchmark, we used $D_{\mathrm{tr}}$ to calculate an empirical Gibbs prior $\mu_{\beta r^{\mathrm{IS}}}(\pi) \propto \mu(\pi)\mathrm{exp}(\beta \sqrt{n_{\mathrm{tr}}} r^{\mathrm{IS}}(\pi, D_{\mathrm{tr}}))$ for each $\beta$ in a grid. We selected the value of $\beta$ where $r^{\mathrm{IS}}(\mu_{\beta r^{\mathrm{IS}}}, D_{\mathrm{val}})$ was the greatest and then calculated a final empirical Gibbs prior with this $\beta$, and using all the data in $D_{1:m}$. $\mu$ was a uniform prior and we used the grid $\beta \in \{1, 5, 10\}$.

In the CB Linear benchmark, we followed the same procedure, but with some small modifications. We approximated $\mu_{\beta r^{\mathrm{IS}}}$ with a diagonal Gaussian for each $\beta$ in the grid $\{10, 100, 1000\}$. $\mu$ was a standard Gaussian prior over the weights of the linear softmax policy.

When using differentially private priors in the MAB benchmark, we used priors of the form $\mu_{\bs{w}}(a) = \mathrm{exp}(\bs{w}_a)/\mathrm{exp}(\sum_{a^{\prime} = 1}^{K}\bs{w}_{a^{\prime}})$, parameterised by $\bs{w} \in \mathbb{R}^{K}$. $\bs{w}_a$ is the $a$th element of $\bs{w}$. We used Preconditioned Stochastic Gradient Langevin Dynamics (PSGLD) \cite{li2016pre} to draw $\bs{w}$ from the Gibbs distribution with density proportional to $p(\bs{w})\mathrm{exp}(\lambda \mathbb{E}_{a \sim \mu_{\bs{w}}}[r^{\mathrm{IS}}(a, D_n)])$. $p(\bs{w})$ was a standard Gaussian. Due to Corollary 5.2 of \cite{dziugaite2018data}, $\mu_{\bs{w}}$ is $2\lambda/(n\epsilon_n)$-differentially private with this choice of $\bs{w}$.

In the CB Linear benchmark, we learned Gaussian priors $\mu_{\bs{w}}(\theta) = \mathcal{N}(\bs{w}, I)$ over the weight matrix $\theta$ of the linear softmax policy $\pi_{\theta}$. We used PSGLD to draw $\bs{w}$ from the distribution with density proportional to $p(\bs{w})\mathrm{exp}(\lambda r^{\mathrm{IS}}(\pi_{\bs{w}}, D_n))$. $p(\bs{w})$ was a standard Gaussian, and $\mu_{\bs{w}}$ is $2\lambda/(n\epsilon_n)$-differentially private with this choice of $\bs{w}$.

In both benchmarks we drew $\bs{w}$'s from Gibbs distributions with $\lambda \in \{0.1\sqrt{n \epsilon_n}, 0.5\sqrt{n \epsilon_n}, \sqrt{n \epsilon_n}\}$ and used the one that gave the best bound value, which we justify with the union bound.

To evaluate the $kl^{-1}$ Lever bound and the distribution stability bound we need to calculate or sample from the Gibbs posterior $\rho_{\gamma}$. In the MAB benchmark, we can calculate $\rho_{\gamma}$ in closed-form. In the CB benchmark, we drew samples from $\rho_{\gamma}$ using PSGLD and approximated $\mathbb{E}_{\pi \sim \rho_{\gamma}}[r^{\mathrm{IS}}(\pi, D_n)]$ by averages over 100 samples from $\rho_{\gamma}$. In both the MAB and CB benchmarks, and for both bounds, we evaluated the bounds for several Gibbs posteriors with $\gamma \in \{0.1\sqrt{n \epsilon_n}, 0.5\sqrt{n \epsilon_n}, \sqrt{n \epsilon_n}\}$ and used the ones that gave the best bound values.

Using the Change of Measure Inequality in Lemma \ref{lem:donsk}, the optimal posterior for the localised PAC-Bayes Bernstein bound is the Gibbs distribution $\rho_{\lambda}$, regardless of the value of $\beta$. We evaluated the bound at $\lambda \in \{\tilde{\lambda}, 1.5\tilde{\lambda}, 2.0\tilde{\lambda}\}$, where $\tilde{\lambda} = \sqrt{2n\epsilon_n\mathrm{ln}(1/\delta)/(e-2)}$ is the optimal value of $\lambda$ when $\beta = 0$ and $D_{\mathrm{KL}}(\rho||\mu_{\beta r^{\mathrm{IS}}}) = 0$. For each $\lambda$, we evaluated the bound with $\beta \in \{0, \lambda/4, \lambda/2\}$. We used the $(\lambda, \beta)$ pair that gave the best bound value. The optimal posterior for the PAC-Bayes Hoeffding-Azuma bound with the empirical Gibbs prior is also the Gibbs distribution $\rho_{\lambda}$. We evaluated this bound with $\lambda = \sqrt{2n(\epsilon_n^2\mathrm{ln}(1/\delta) + 2)}$ and $\beta \in \{0, \lambda/4, \lambda/2\}$. This value of $\lambda$ is approximately the optimal value when $D_{\mathrm{KL}}(\rho||\mu_{\beta r^{\mathrm{IS}}}) = 0$.

\subsection{Description of the TPOEM and TL2 Baselines}
\label{sec:tpoem_and_tl2}

Like the original POEM algorithm \cite{swaminathan2015batch}, TPOEM uses the sample variance of the CIS reward estimate to regularise the policy selection. Its objective function is:

\begin{equation}
r^{\mathrm{CIS}}(\pi_{\theta}, D_n) - \beta\sqrt{\frac{v^{\mathrm{CIS}}(\pi_{\theta}, D_n)}{n}},\label{eqn:tpoem_obj}
\end{equation}

where

\begin{align*}
v^{\mathrm{CIS}}(\pi_{\theta}, D_n) = \frac{1}{n-1}\sum_{i=1}^{n}\bigg(&\mathrm{min}\left(\frac{\pi_{\theta}(a_i|s_i)}{b(a_i|s_i)}, \frac{1}{\tau}\right)r_i\\
&- r^{\mathrm{CIS}}(\pi_{\theta}, D_n)\bigg)^2,
\end{align*}

is the sample variance of the CIS estimate. We split the data set into training data $D_{\mathrm{tr}}$ and validation data $D_{\mathrm{val}}$ such that $D_{\mathrm{tr}}$ contains four times the number of samples in $D_{\mathrm{val}}$. We maximise Equation \ref{eqn:tpoem_obj} with respect to the weights $\theta$ using the training data, and for each $\beta \in \{10^{-k}|k \in \{0, \dots, 5\}\}$. This gives us a set of policies $\Pi_{\Theta}^{\beta}$ with 6 elements (one for each $\beta$). Using the validation data, we evaluate the following bound which is essentially a simpler version of the original POEM bound that only holds for finite policy classes.

\begin{equation}
r^{\mathrm{CIS}}(\pi_{\theta}, D_{\mathrm{val}}) - \sqrt{\frac{2v^{\mathrm{CIS}}(\pi_{\theta}, D_n)\mathrm{ln}(2|\Pi_{\Theta}^{\beta}|/\delta)}{n_{\mathrm{val}}}} - \frac{7\mathrm{ln}(2|\Pi_{\Theta}^{\beta}|/\delta)}{\tau(n_{\mathrm{val}}-1)},\label{eqn:tpoem_bound}
\end{equation}

where $n_{\mathrm{val}} = |D_{\mathrm{val}}|$ is the number of examples in the validation data set. We choose the policy $\pi_{\theta} \in \Pi_{\Theta}^{\beta}$ that maximises the bound in Equation \ref{eqn:tpoem_bound}. The TL2 baseline uses the L2 norm of the neural network weights to regularise the policy selection. It uses the objective function:

\begin{equation}
r^{\mathrm{CIS}}(\pi_{\theta}, D_n) - \beta\norm{\theta}_2^2,\label{eqn:tl2_obj}
\end{equation}

As with TPOEM, we split the data set into $D_{\mathrm{tr}}$ and $D_{\mathrm{val}}$ with the same relative sizes. We maximise Equation \ref{eqn:tl2_obj} with respect to $\theta$ using $D_n = D_{\mathrm{tr}}$ and for each $\beta \in \{10^{-k}|k \in \{1, \dots, 6\}\}$. This gives us a set of policies $\Pi_{\Theta}^{\beta}$ with 6 elements. Using $D_{\mathrm{val}}$, we evaluate a PAC bound based on the Hoeffding-Azuma inequality.

\begin{equation}
r^{\mathrm{CIS}}(\pi_{\theta}, D_{\mathrm{val}}) - \frac{1}{\tau}\sqrt{\frac{\mathrm{ln}(|\Pi_{\Theta}^{\beta}|/\delta)}{2n_{\mathrm{val}}}},\label{eqn:tl2_bound}
\end{equation}

We choose the policy $\pi_{\theta} \in \Pi_{\Theta}^{\beta}$ that maximises the bound in Equation \ref{eqn:tl2_bound}.

\section{Additional Experiments}

\subsection{Experiments With The Efron-Stein PAC-Bayes Bound for Weighted Importance Sampling}
\label{sec:efron_stein_experiment}

We compare the Efron-Stein (ES) PAC-Bayes bound for the $r^{\mathrm{WIS}}$ estimate (in Equation \ref{eqn:pac_bayes_wis}) against the Hoeffding-Azuma (Theorem \ref{thm:ex_bound}), $kl^{-1}$ (Equation \ref{eqn:pac_bayes_kl_inv}), Pinsker (Equation \ref{eqn:pac_bayes_pinsker}), and  Bernstein (Equation \ref{eqn:pac_bayes_bern_eps}) PAC-Bayes bounds for the $r^{\mathrm{IS}}$ estimate.

We compare the bounds in the offline MAB Binary benchmark, in which the policy class is the set of all deterministic policies (i.e. the set of actions). As in our experiments in Section \ref{sec:bound_comparison}, we optimise each bound with respect to the posterior $\rho$ and then report the value of the bound and the expected reward for this $\rho$. Details about how we optimise the bounds for the $r^{\mathrm{IS}}$ estimate with respect to $\rho$ and then evaluate them can be found in Appendix \ref{sec:opt_and_eval}.

For convenience, we re-state the RHS of the ES PAC-Bayes bound for the $r^{\mathrm{WIS}}$ estimate from Equation \ref{eqn:pac_bayes_wis}:
\begin{align*}
&r^{\mathrm{WIS}}(\rho, D_n) - \left|R^{\mathrm{WIS}}(\rho) - R(\rho)\right|\\
&-\sqrt{2\left(y + 2V^{\mathrm{WIS}}(\rho, D_n)\right)}\nonumber\\
&\times \sqrt{D_{\mathrm{KL}}(\rho||\mu) + \frac{1}{2}\mathrm{ln}\left(1 + \frac{2V^{\mathrm{WIS}}(\rho, D_n)}{y}\right) + \mathrm{ln}(1/\delta)}.\nonumber
\end{align*}

Also, recall that $V^{\mathrm{WIS}}(\pi, D_n)$ was defined as:
\begin{equation*}
V^{\mathrm{WIS}}(\pi, D_n) = \sum_{i=1}^{n}\mathop{\mathbb{E}}_{D_n, D_n^{\prime}}\left[\tilde{w}_{\pi, i}^2 + \tilde{u}_{\pi, i}^2| D_{i}\right],
\end{equation*}

where
\begin{equation*}
\tilde{w}_{\pi, i} = \frac{\frac{\pi(a_i)}{b(a_i)}}{\sum_{j=1}^{n}\frac{\pi(a_j)}{b(a_j)}}, \qquad \tilde{u}_{\pi, i} = \frac{\frac{\pi(a_i^{\prime})}{b(a_i^{\prime})}}{\frac{\pi(a_i^{\prime})}{b(a_i^{\prime})} + \sum_{j \neq i}\frac{\pi(a_j)}{b(a_j)}}.
\end{equation*}

$a^{\prime}_i$ is an independently sampled copy of $a_i$. To evaluate the ES PAC-Bayes bound, we first assume that the bias $\left|R^{\mathrm{WIS}}(\rho) - R(\rho)\right|$ is always equal to 0. Since the reward distribution $P_R$ in the MAB benchmark is actually known, we could estimate $\left|R^{\mathrm{WIS}}(\rho) - R(\rho)\right|$ to arbitrary accuracy to check whether this is a reasonable assumption. Some rough estimates suggest that in the MAB benchmark with $n=1000$, the bias of the WIS estimate is approximately $10^{-9}$ or smaller.

\begin{figure}[H]
\centering
\includegraphics[width=1.0\columnwidth]{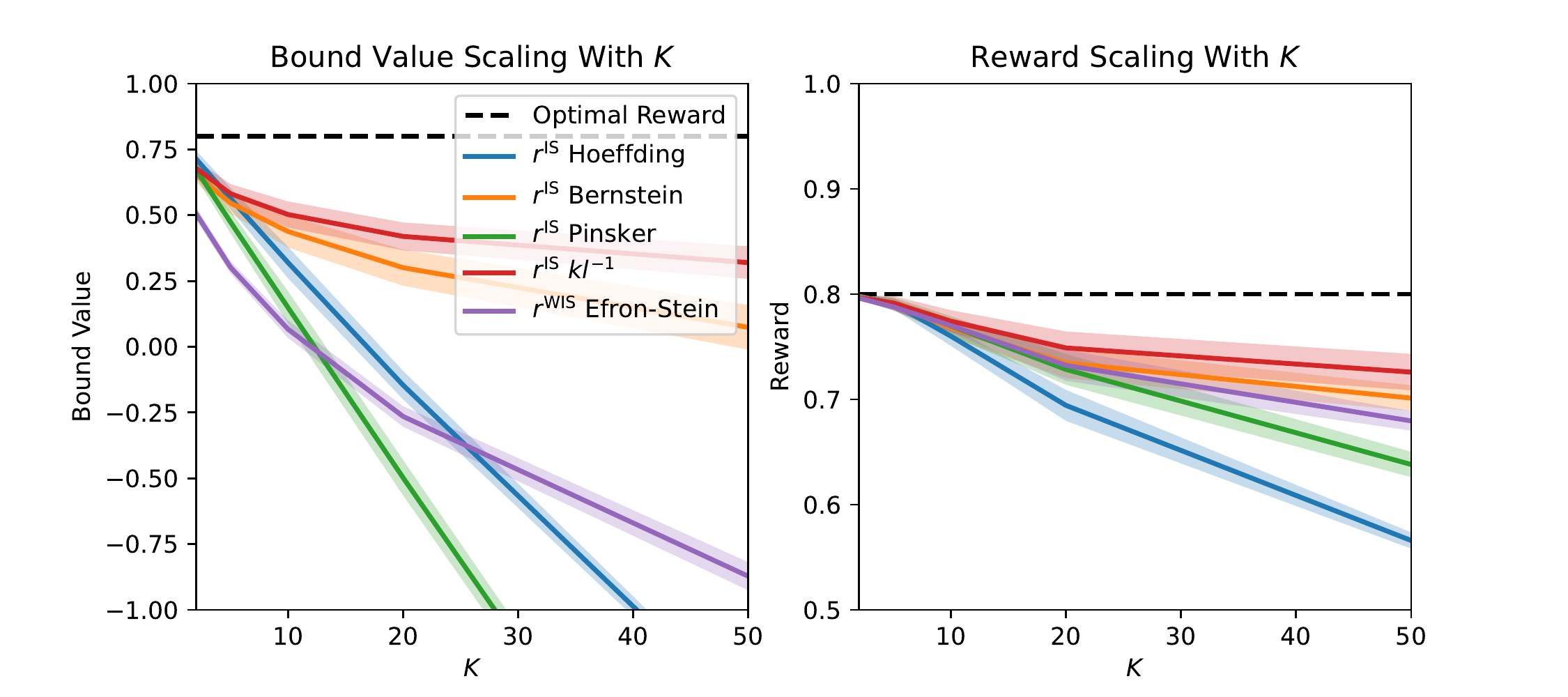}
\caption{The bound value (left) and expected reward (right) for the Efron-Stein WIS bound and each of the IS bounds in the MAB Binary benchmark. The number of actions $K$ varies from 2 to 50 along the $x$ axes.}
\label{fig:efron_stein}
\end{figure}

Next, we replace the semi-empirical ES variance proxy $V^{\mathrm{WIS}}(\pi, D_n)$ with a fully empirical estimate. For $i$ in $\{1, \dots, n\}$, we draw $m=1000$ actions $\{a^{\prime}_{ik}\}_{k=1}^{m}$ and another $m$ actions $\{a^{\prime\prime}_{ik}\}_{k=1}^{m}$ from the behaviour policy $b$. $\{a^{\prime}_{ik}\}_{k=1}^{m}$ for $i = 1, \dots, n$ are 1000 draws of the ghost sample $D_n^{\prime}$ and $\{a^{\prime\prime}_{ik}\}_{k=1}^{m}$ for $i = 1, \dots, n$ are 1000 re-draws of the original sample $D_n$. Then for each policy $\pi \in \Pi$, we calculate:
\begin{align*}
\widehat{w}_{\pi, i, k} &= \frac{\frac{\pi(a_i)}{b(a_i)}}{\sum_{j=1}^{i}\frac{\pi(a_j)}{b(a_j)} + \sum_{j=i+1}^{n}\frac{\pi(a_{jk}^{\prime\prime})}{b(a_{jk}^{\prime\prime})}},\\
\widehat{u}_{\pi, i, k} &= \frac{\frac{\pi(a_{ik}^{\prime})}{b(a_{ik}^{\prime})}}{\frac{\pi(a_{ik}^{\prime})}{b(a_{ik}^{\prime})} + \sum_{j=1}^{i-1}\frac{\pi(a_j)}{b(a_j)} + \sum_{j=i+1}^{n}\frac{\pi(a_{jk}^{\prime\prime})}{b(a_{jk}^{\prime\prime})}}.
\end{align*}

In the ES PAC-Bayes bound, we replace $V^{\mathrm{WIS}}(\pi, D_n)$ with the estimate:
\begin{equation*}
\widehat{V}^{\mathrm{WIS}}(\pi, D_n) = \sum_{i=1}^{n}\frac{1}{m}\sum_{k=1}^{m}\widehat{w}_{\pi, i, k}^2 + \widehat{u}_{\pi, i, k}^2
\end{equation*}

Note that we only have to calculate $\widehat{V}^{\mathrm{WIS}}(\pi, D_n)$ once before we optimise the ES PAC-Bayes bound with respect to $\rho$. Since, in this problem, the policy class $\Pi$ is finite with $K$ elements, calculating $\widehat{V}^{\mathrm{WIS}}(\pi, D_n)$ for every $\pi \in \Pi$ is possible. However, this would obviously not be possible for infinite policy classes. Strictly speaking, we should replace $V^{\mathrm{WIS}}(\pi, D_n)$ with an upper bound rather than an estimate to obtain a valid bound, as is done in \cite{kuzborskij2021confident}. Using an estimate rather than an upper bound results in a favourable evaluation of the ES PAC-Bayes bound.

We always use a data set of size $n=1000$, and the data set is generated using a uniform behaviour policy. We varied the number of actions $K$ from $2$ to $50$ to investigate how the bounds compare in MAB problems with different numbers of actions.

Figure \ref{fig:efron_stein} shows the bound value (left) and the expected reward (right) for each of the bounds we compared at each $K$. In the left plot in Figure \ref{fig:efron_stein}, we observe that the value of the Efron-Stein WIS bound is the lowest for $K \leq 10$. As $K$ increases above 10, the Efron-Stein WIS bound overtakes both the IS Pinsker and IS Hoeffding-Azuma bounds. However, for $K \geq 10$, the Efron-Stein WIS bound is vacuous (i.e. less than 0). On the bright side, the Efron-Stein WIS bound appears to work well as a learning objective. In the right plot of Figure \ref{fig:efron_stein}, we see that the policy learned by maximising the Efron-Stein WIS bound achieves close to the highest expected reward.

\subsection{Insights About Choosing Bound Parameters}
\label{sec:app_param_exp}

We compare the methods presented in Appendix \ref{sec:app_bound_params} for approximately optimising PAC-Bayes bounds with respect to their parameters. We use each method to set the $\lambda$ parameter of the $r^{\mathrm{IS}}$ PAC-Bayes Bernstein bound.

In both the MAB and CB benchmarks, we compare the Bernstein bound with $\lambda$ learned using a subset of the data (Theorem \ref{thm:pac_bayes_bern_subset_bound}) and the Bernstein bound with $\lambda$ optimised over a geometric grid (Theorem \ref{thm:pac_bayes_bern_grid_bound}). We compare the grid bound with several choices of the grid parameter $c \in \{1.1, 1.2, 1.5\}$. We also compare against some baselines: the $kl^{-1}$ bound (Theorem \ref{thm:pac_bayes_kl}), the Bernstein bound with a fixed value of $\lambda$ (Theorem \ref{thm:is_bernstein}) and the idealised Bernstein bound with the optimal choice of $\lambda$ (Equation \ref{eqn:ideal_lamb}). For the fixed value of $\lambda$, we use $\lambda = \sqrt{n\epsilon_n\mathrm{ln}(1/\delta)/(e-2)}$, which is equal to the optimal value when $D_{\mathrm{KL}}(\rho||\mu) = 0$. The $kl^{-1}$ bound represents the best parameter-free bound, the Bernstein bound with a fixed $\lambda$ represents a naive choice of $\lambda$, and the idealised Bernstein bound is the best bound we could hope to achieve by optimising $\lambda$.

\begin{figure}[H]
\centering
\includegraphics[width=1.0\columnwidth]{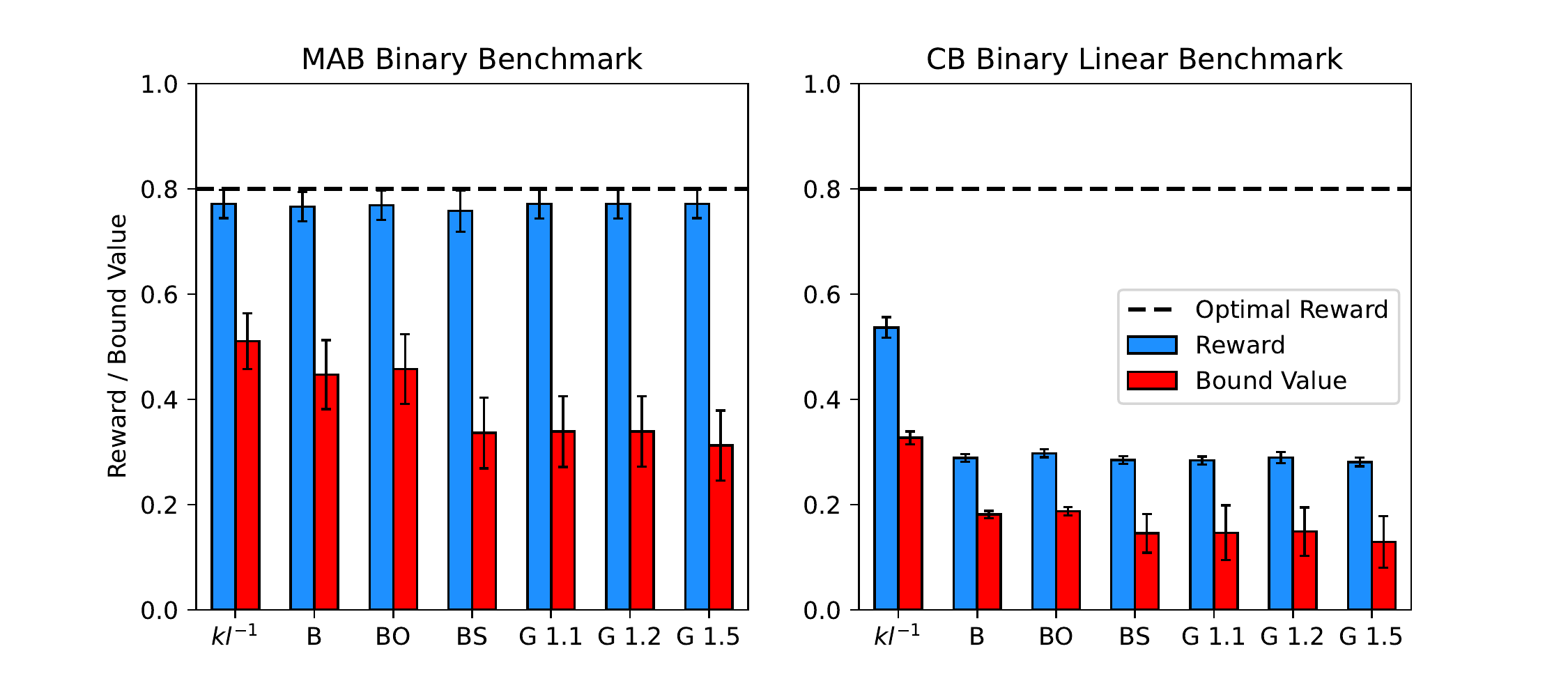}
\caption{The expected reward (blue) and bound value (red) in our comparison of the methods for choosing the $\lambda$ parameter of the PAC-Bayes Bernstein bound. B is the Bernstein bound with a fixed choice of $\lambda$, BO is the (invalid) Bernstein bound with the optimal $\lambda$, BS is the Bernstein bound with $\lambda$ learned using a subset of the data and G 1.1, G 1.2 and G 1.5 are the Bernstein bound with $\lambda$ optimised over a geometric grid with $c=1.1$, $c=1.2$ and $c=1.5$.}
\label{fig:params}
\end{figure}

In Figure \ref{fig:params}, we can see that in both the MAB and CB benchmarks, the sample splitting $\lambda$ and the grid $\lambda$'s all yield almost identical expected reward and bound values. We find that both methods of approximately optimising the Bernstein bound w.r.t. $\lambda$ give worse bound values than the fixed data-independent choice of $\lambda$. Surprisingly, the fixed $\lambda$ is almost as good as the optimal $\lambda$.

Now, we briefly explore why the fixed value of $\lambda$ was almost as good as the optimal value. Both the PAC-Bayes Hoeffding-Azuma and PAC-Bayes Bernstein bounds can be written in the form:

\begin{equation*}
R(\rho) \geq r(\rho, D_n) - a\lambda - \frac{D_{\mathrm{KL}}(\rho||\mu) + \mathrm{ln}(1/\delta)}{\lambda}
\end{equation*}

For bounds of this form, the optimal $\lambda$ is the one that minimises:

\begin{equation*}
f(\lambda) = a\lambda + (D_{\mathrm{KL}}(\rho||\mu) + \mathrm{ln}(1/\delta))/\lambda
\end{equation*}

One can verify that $\lambda^* = \sqrt{(D_{\mathrm{KL}}(\rho||\mu) + \mathrm{ln}(1/\delta))/a}$. We found that the value of the PAC-Bayes Bernstein bound at $\hat{\lambda} = \sqrt{\mathrm{ln}(1/\delta)/a}$ was almost the same as the bound value at $\lambda^*$. It can be shown that the second derivative of $f$ evaluated at $\lambda^*$ is:

\begin{equation*}
f^{\prime\prime}(\lambda^*) = \frac{2a^{3/2}}{\sqrt{D_{\mathrm{KL}}(\rho||\mu) + \mathrm{ln}(1/\delta)}}
\end{equation*}

When $a$ is close to 0, $f^{\prime\prime}(\lambda^*)$ will also be close to 0. Therefore, we can expect $f(\lambda)$ to be almost constant in the neighbourhood of $\lambda^*$ when $a$ is near 0. For the PAC-Bayes Bernstein $r^{\mathrm{IS}}$ bound, $a = (e-2)/n \epsilon_n$. In the MAB Binary benchmark considered in Section \ref{sec:parameter_comparison}, we had $n = 1000$ and $\epsilon_n = 0.1$, so $a \approx 0.000718$. This may explain why the Bernstein bound value at $\hat{\lambda}$ was close to the Bernstein bound value at $\lambda^*$.

In Figure \ref{fig:lambda}, we plot $f(\lambda)$ for the PAC-Bayes Bernstein bound. We set $n \in \{100, 1000, 10000\}$, and to match our earlier experiment in the MAB Binary benchmark, we set $\epsilon_n=0.1$, $\delta = 0.05$ and $D_{\mathrm{KL}}(\rho||\mu) = \mathrm{ln}(K)$, with $K = 10$. This is the maximum value of the KL divergence, which means the difference between $\hat{\lambda}$ and $\lambda^*$ is maximised.

\begin{figure}[H]
\centering
\includegraphics[width=1.0\columnwidth]{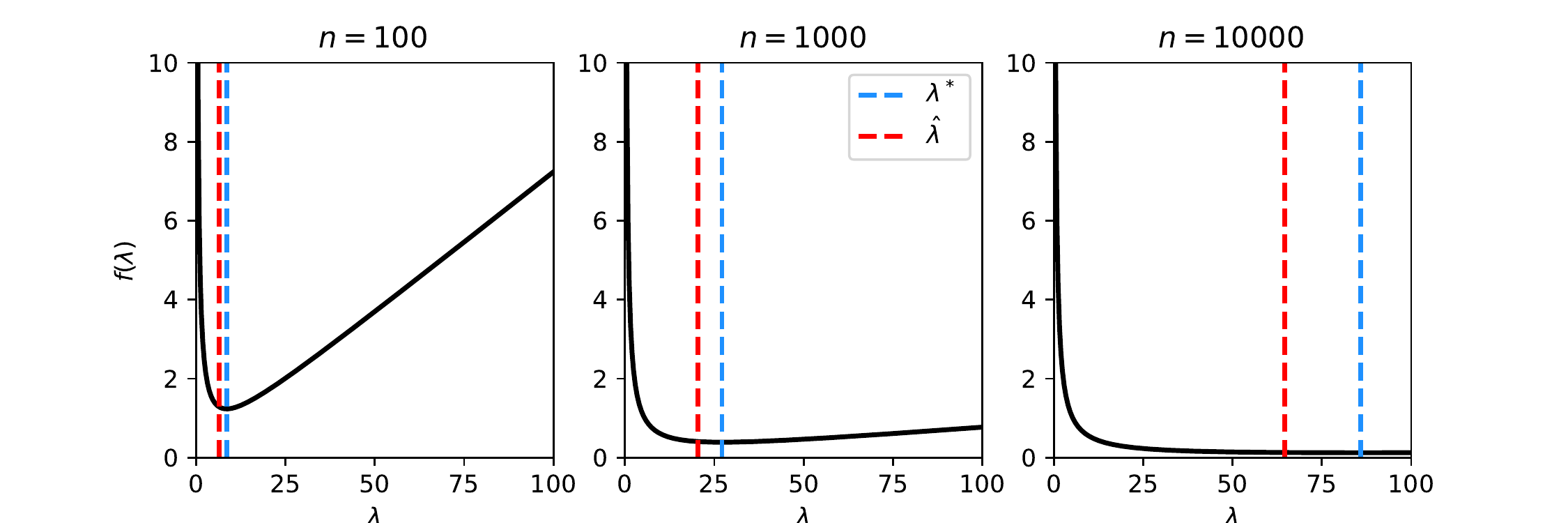}
\caption{$f(\lambda)$ for $a = (e-2)/(n \epsilon_n)$ with $\epsilon_n = 0.1$, $\delta = 0.05$ and $D_{\mathrm{KL}}(\rho||\mu) = \mathrm{ln}(K)$. $n$ is equal to 100 (left), 1000 (middle) and 10000 (right).}
\label{fig:lambda}
\end{figure}

In Figure \ref{fig:lambda}, we see that as $n$ increases (and $a$ decreases) $f(\lambda)$ becomes almost constant in the neighbourhood of $\lambda^*$. The value of $f(\lambda)$ at $\hat{\lambda}$ and $\lambda^*$ is almost the same even for $n = 100$.

\end{document}